\documentclass{article}

     \usepackage[preprint,nonatbib]{neurips_2023}

\usepackage[utf8]{inputenc} % allow utf-8 input
\usepackage[T1]{fontenc}    % use 8-bit T1 fonts
\usepackage{hyperref}       % hyperlinks
\usepackage{url}            % simple URL typesetting
\usepackage{booktabs}       % professional-quality tables
\usepackage{amsfonts}       % blackboard math symbols
\usepackage{nicefrac}       % compact symbols for 1/2, etc.
\usepackage{microtype}      % microtypography
\usepackage{xcolor}         % colors
\usepackage{microtype}
\usepackage{graphicx}
\usepackage{subfigure}
\usepackage{booktabs} % for professional tables
\usepackage{multicol}
\usepackage{wrapfig}

\usepackage{times}
\usepackage{soul}
\usepackage{url}
\usepackage{graphicx}
\usepackage{amsmath}
\usepackage{amsthm}
\usepackage{algorithm}
\usepackage[noend]{algorithmic}
\usepackage{amsfonts,amssymb} 

\usepackage[capitalize,noabbrev]{cleveref}

\usepackage{bm,array}
\usepackage{comment}
 \usepackage{makecell}
 \usepackage{color}
 \usepackage{multirow}
 \usepackage{float}
\usepackage[normalem]{ulem}

\newcommand{\cS}{{\mathcal{S}}}

\newcommand{\bc}{\textbf{c}}

\newcommand{\bDelta}{\pmb{\lambda}}

\newcommand{\bbR}{\mathbb{R}}
\newcommand{\bbE}{\mathbb{E}}
\newcommand{\tbbE}{\widetilde{\mathbb{E}}}

\theoremstyle{plain}
\newtheorem{theorem}{Theorem}[section]
\newtheorem{proposition}[theorem]{Proposition}
\newtheorem{lemma}[theorem]{Lemma}

\theoremstyle{definition}

\theoremstyle{remark}

\newcommand{\squishlist}{
 \begin{list}{$\bullet$}
  { \setlength{\itemsep}{0pt}
     \setlength{\parsep}{2pt}
     \setlength{\topsep}{2pt}
     \setlength{\partopsep}{0pt}
     \setlength{\leftmargin}{1.5em}
     \setlength{\labelwidth}{1em}
     \setlength{\labelsep}{0.5em} } }

\newcommand{\squishend}{
  \end{list}  }

\newif\ifnotes\notestrue
\def\htien#1{}

\usepackage[textsize=tiny]{todonotes}

\author{%
  Hao Jiang,~Tien Mai,~Pradeep Varakantham and Huy Minh Hoang
  \\
  School of Computing and Information Systems\\
  Singapore Management Univerisity, Singapore\\
  \texttt{\{haojiang.2021@phdcs.,atmai@, pradeepv@, mhhoang@\}smu.edu.sg} \\
}

\title{Solving Richly Constrained Reinforcement Learning through State Augmentation and Reward Penalties}

\begin{document}

\maketitle

\begin{abstract}
Constrained Reinforcement Learning has been employed to compute safe policies through the use of expected cost constraints. The key challenge is in handling constraints on expected cost accumulated across time steps. Existing methods have developed innovative ways of converting this cost constraint over entire policy to constraints over local decisions (at each time step). While such approaches have provided good solutions with regards to objective, they can either be overly aggressive or conservative with respect to costs. This is owing to use of estimates for "future" or "backward" costs in local cost constraints.

To that end, we provide an equivalent unconstrained formulation to constrained RL that has an augmented state space and reward penalties. This intuitive formulation is general and has interesting theoretical properties. More importantly, this provides a new paradigm for solving richly constrained (e.g., constraints on expected cost, Value at Risk, Conditional Value at Risk) Reinforcement Learning problems effectively. As we show in our experimental results, we are able to outperform leading approaches for different constraint types on multiple benchmark problems.
\end{abstract}

\section{Introduction}

There are multiple objectives of interest when handling safety depending on the type of domain: (a) ensuring safety constraint is never violated; (b) ensuring safety constraint is not violated in expectation; (c) ensuring the chance of safety constraint violation is small (Value at Risk, VaR)~\cite{lucas1998extreme}; (d) ensuring the expected cost of violation is bounded (Conditional Value at Risk, CVaR)~\cite{rockafellar2000optimization, yang2021wcsac}; and others. One of the main models in Reinforcement Learning to ensure safety is Constrained RL, which employs objective (b) above. Our focus in this paper is also on Constrained RL but considering the four types of constraints mentioned above. 

Constrained RL problems are of relevance in domains that can be represented using an underlying Constrained Markov Decision Problem (CMDP)~\cite{altman1999constrained}. The main challenge in solving Constrained RL problems is the expected cost constraint, which requires averaging over multiple trajectories from the policy. Such problems have many applications including but not limited to: (a) electric self driving cars reaching destination at the earliest while minimizing the risk of getting stranded on the road with no charge; (b) robots moving through unknown terrains to reach a destination, while having a threshold on the average risk of passing through unsafe areas (e.g., a ditch). Broadly, they are also applicable to problems such as robot motion planning~\cite{1Robotics1, Robotics2, Robotics3}, resource allocation~\cite{Resource_allocation1, Resource_allocation2}, and financial engineering~\cite{finance, finance_2}. 

\noindent \textbf{\textit{Related Work:}} Many model free approaches have been proposed to solve Constrained RL problems. One of the initial approaches to be developed for addressing such constraints is the Lagrangian method~\cite{Cons_2}. However, such an approach does not provide either theoretical or empirical guarantees in ensuring the constraints are enforced.  To counter the issue of safety guarantees, next set of approaches focused on transforming the cost constraint over trajectories into cost constraint over individual decisions in many different ways. One such approach imposed surrogate constraints~\cite{surroogate_1, gabor1998multi} on individual state and action pairs. Since the surrogate constraints are typically stricter than the original constraint on the entire trajectory, they were able to provide theoretical guarantees on safety. However, the issue with such type of approaches is their conservative nature, which can potentially hamper the expected reward objective.  More recent approaches such as CPO (Constrained Policy Optimization)~\cite{cons_5}, Lyapunov~\cite{Cons_4}, BVF~\cite{satija2020constrained} have since provided more tighter local constraints (over individual decisions) and thereby have improved the state of art in guaranteeing safety while providing high quality solutions (with regards to expected reward).  In converting a trajectory based constraint to a local constraint, there is an estimation of cost involved in the trajectory. Due to such estimation, transformed cost constraints over individual decisions are error prone. In problems where the estimation is not close to the actual, results with such approaches with regards to cost constraint enforcement are poor (as we demonstrate in our experimental results). 

\textbf{Contributions:}

To that end, we focus on an approach that relies on exact accumulated costs (and not on estimated costs). In this paper, we make four key contributions:
\squishlist
\item We provide a re-formulation of the constrained RL problem through augmenting the state space with cost accumulated so far and also considering reward penalties when cost constraint is violated. This builds on the idea of augmented MDPs~\cite{hou2014revisiting} employed to solve Risk Sensitive MDPs.  {\em The key advantage of this reformulation is that by penalizing rewards (as opposed to the entire expected value that is done typically using Lagrangian methods), we get more fine grained control on how to handle the constraints.} 
\item We show theoretically that the reward penalties employed in the new formulation are not adhoc and 
%\remove{have a direct mapping to the parameters used by the corresponding \mtien{risk-neural, chance constrained (or VAR)  and CVaR objectives}.  We also demonstrate that our re-formulation} 
can equivalently represent different {constraints} mentioned in the first paragraph of introduction, i.e. {risk-neural, chance constrained (or VAR)  and CVaR constraints}. 
\item We  modify existing RL methods (DQN and SAC) to solve the re-formulated RL problem with augmented state space and reward penalties. A key advantage is the knowledge of exact costs incurred so far (available within the state space) and this allows for assigning credit for cost constraint violations more precisely during learning compared to existing approaches. 
\item Finally, we demonstrate the utility of our approach by comparing against leading approaches for constrained RL on multiple benchmark problems for different types of constraints. We show that our approaches are able to  outperform leading Constrained RL approaches from the literature either with respect to expected value or in enforcing different types of cost constraints or both. 

\squishend

\section{Constrained Markov Decision Process}
A Constrained Markov Decision Process (CMDP) \cite{altman1999constrained} is defined using tuple $\left \langle  S, A, r, p, d, s_0, c_{max}\right \rangle$, where $S$ is set of states with initial state as $s_0$, $A$ is set of actions, $r: S\times A \rightarrow \mathbb{R}$ is reward with respect to each state-action pair, $p: S\times A \rightarrow P$ is transition probability of each state. $d: S \rightarrow d(S)$ is the cost function and $c_{max}$ is the maximum allowed cumulative
cost. {Here, we assume that $d(s)\geq 0$ for all $s\in S$. This assumption is not restrictive as one can always add positive amounts to $d(s)$ and $c_{max}$ to meet the assumption.}
The objective in a risk-neural CMDP is to compute a policy, $\pi: S\times A \rightarrow [0,1]$, which maximizes reward over a finite horizon $T$ while ensuring the cumulative cost does not exceed the maximum allowed cumulative cost. 
\begin{equation}\tag{\sf\small RN-CMDP}
\begin{aligned}
\label{equ:cmdp}
    &\max_\pi \mathbb{E}\left[\sum_{t=0}^T {\gamma^t}r(s_t,a_t)|s_0,\pi\right]
    \text{  s.t.  } \quad \mathbb{E}\left[\sum_{t=0}^T  d(s_t)|s_0,\pi\right]\leq c_{max}.
\end{aligned}
\end{equation}
The literature has seen other types of constraints, e.g., chance constraints requiring that $P_{\pi}(D(\tau)>c_{max}) \leq \alpha$ for a risk level $\alpha\in [0,1]$, or CVaR ones of the form $\bbE_{\pi}[(D(\tau)-c_{max})^+]\leq \beta$, where $D(\tau)=\sum_{s \in \tau}d(s)$ is the cumulative cost in trajectory $\tau$.  Handling different types of constraints would require different techniques. In the next section, we present our approach based on augmented state and reward penalties that assembles all the aforementioned constraint types into one single framework. 
%Introducing such a constraint makes RL problem more challenging to handle.{ It  is generally difficult to find a policy that satisfies the constraint. The Lagrangian-based method generally cannot guarantee feasible policies, except when the Lagrangian multipliers are set to be very high, which also degrades the rewards and leads to bad policies. Approaches that rely on breaking down the trajectory-based constraints into state-based constrained are generally conservative.\ct}

\section{Cost Augmented Formulation for Safe RL}
We first present our extended MDP reformulation and provide several theoretical findings that connect our extended formula with different variants of CMDP. We focus on the case of single-constrained MDP. Extension to multiple-constrained MDP will be discussed in the appendix.% and show how the results can be extended to the multi-constrained setting. 
\subsection{Extended MDP Reformulation}
We introduce our approach to track the accumulated cost at each time period, which allows us to determine states that potentially lead to high-cost trajectories. To this end, let  us define a new MDP with an extended state space $\left \langle  \widetilde{S}, A, \widetilde{r}, \widetilde{p}, d, s_0, c_{max}\right \rangle$ where $\widetilde{S}= \{(s,c)|~ s \in S, c\in \mathbb{R}_+\}$. That is, each state $s'$ of the extended MDP includes an original state from $S$ and information about the accumulated cost. We then define the transition probabilities between states in the extended space.
\[
\widetilde{p}((s'_{t+1},c'_{t+1})|(s_t,c_t), a_t) = \begin{cases}
p(s'_{t+1}|s_t,a_t) &\text{ if }c'_{t+1} = c_t + d(s_t) \\
0 & \text{ otherwise}
\end{cases}
\]
and new rewards with penalties  
\begin{equation}\label{eq:new-rewards}
\widetilde{r}(a_t|(s_t,c_t)) = \begin{cases}
r(a_t|s_t) ~\text{ if } c_t \leq c_{max} &\text{ and }c_t+d(s_t)\leq c_{max} \\
r(a_t|s_t) - {\lambda (c_t+d(s_t))/\gamma^t} &{\text{ if } c_t \leq c_{max} \text{ and }c_t+d(s_t)> c_{max}} \\
{r(a_t|s_t) - \lambda d(s_t)/\gamma^t} &\text{ if } c_t > c_{max} 
\end{cases}
\end{equation}
%\tien{Say a bit about the formulation.}
where $\lambda$ is a positive scalar and $\lambda d(s_t)$ and $\lambda(c_t+d(s_t))$ are penalties given to the agent if the accumulated cost exceeds the upper bound $c_{max}$. In the second case of \eqref{eq:new-rewards} (for stages right before exceeding the upper bound $c_{max}$), we add a  penalty $\lambda c_t$ to capture the accumulated cost until those stages. 
Under the reward penalties specified in the second and the third cases of \eqref{eq:new-rewards}, the accumulated reward for each trajectory $\tau = \{(s_0,a_0),\ldots,(s_T,a_T)\}$ can be written as  $\widetilde{R}(\tau) = \sum_{t} \gamma^t r(a_t|s_t)$ if $D(\tau)\leq c_{max}$  and $\widetilde{R}(\tau)  = \sum_{t} \gamma^t r(a_t|s_t) - \lambda D(\tau)$ if $D(\tau)>c_{max}$, where  $D(\tau)$ is the total cost of trajectory $\tau$, i.e., $D(\tau) = \sum_{s_t\in \tau} d(s_t)$. 

As can be seen from the reward definition, we penalize \textit{every trajectory} that violates the cost constraint. This allows for the fine grained control that ensures cost constraint is enforced correctly, while also allowing for expected reward maximization. Overall, we have the following unconstrained objective which handles the constraints in a relaxed manner through penalties:
\begin{equation}\tag{\sf\small EMDP}
\begin{aligned}
\label{equ:umdp}
    &\max_\pi \mathbb{E}\left[\sum_{t=0}^T {\gamma^t}\widetilde{r}(a_t|(s_t,c_t))\Big|(s_0,c_0),\pi\right]
\end{aligned}
\end{equation}
where $c_0=0$. 
There are also other ways to penalize the rewards, allowing us to establish equivalences between the extended MDP to other risk-averse CMDP, which we will discuss later in the next section.
%%%%%%%%%%%%%%%%%%%%%%%%%%
%%%%%%%%%%%%%%%%%%%%%%%%%%

\subsection{Theoretical   Properties}

To demonstrate the generality in the representation of the reward penalties along with state augmentation in the unconstrained MDP \eqref{equ:umdp}, we provide theoretical properties that map \eqref{equ:umdp} to CMDP under different types of constraints (expected cost, VaR, CVaR, Worst-case cost):
\squishlist
    \item[(i)]  Proposition \ref{prop:EMDP-UCMDP} states that if the penalty parameter $\lambda = 0$, then \eqref{equ:umdp} becomes the classical unconstrained MDP.
    \item[(ii)] Theorem \ref{th:UMDP-WCMDP}  shows that if $\lambda = \infty$, then \eqref{equ:umdp} is equivalent to a worst-case constrained MDP 
    \item[(iii)] Theorem~\ref{th:RN-CMDP} establishes a lower bound on $\lambda$ from which any solution to \eqref{equ:umdp} will satisfy the risk-neural constraint in \eqref{equ:cmdp}.
    \item[(iv)] Theorem~\ref{th:VAR-CMDP}  connects \eqref{equ:umdp} with chance-constrained MDP by providing a lower bound for $\lambda$ from which any solution to \eqref{equ:umdp} will satisfy a VaR constraint $P(\sum_{t} d(s_t) \leq c_{max}) \leq \alpha$. 
    %Moreover, we further strengthen this result by showing, in , that if we chance the reward penalties, then we can obtain an equivalent mapping between \eqref{equ:umdp}  and a VaR CMDP. 
    \item[(v)] Theorems~\ref{th:VAR-CMDP} and  \ref{th:CVAR} further strengthen the above results by showing that, under some different reward settings, \eqref{equ:umdp} is equivalent to a chance-constrained (or VaR) or equivalent to a CVaR CMDP.
\squishend
We now describe our theoretical results in detail. All the proofs can be found in the appendix.  We also extend the results to Constrained MDPs with multiple constraints (e.g., a combination of expected cost on one cost measure and CVaR on another cost measure) in the appendix.
We first state, in Proposition \ref{prop:EMDP-UCMDP}, a quite obvious result saying that if we set the penalty parameter $\lambda = 0$, then the MDP with augmented state space becomes the original unconstrained MDP. 
\begin{proposition}\label{prop:EMDP-UCMDP}
    If $\lambda = 0$,  then  \eqref{equ:umdp} is equivalent to the unconstrained MDP $\max_\pi \mathbb{E}\left[\sum_{t=0}^T {\gamma^t}r(s_t,a_t)|s_0,\pi\right]$. 
\end{proposition}
It can be seen that increasing $\lambda$ will set more penalties to trajectories whose costs exceed the maximum cost allowed $c_{\max}$, which also implies that \eqref{equ:umdp} would lower the probabilities of taking these trajectories. So, intuitively, if we raise $\lambda$ to infinity, then \eqref{equ:umdp} will give policies  that yield \textit{zero} probabilities to violating trajectories. We state this result in Theorem \ref{th:UMDP-WCMDP} below. 
%The theorem below states that if raise the penalty parameter $\lambda$ to infinity, then the unconstrained extended MDP  \eqref{equ:umdp} is equivalent to a worst-case CMP problem, i.e., any trajectories generated by the optimal policy need not exceed the maximum cost allowed $c_{max}$. 
\begin{theorem}[Connection to worst-case CMDP]
\label{th:UMDP-WCMDP}
If we set $\lambda = \infty$, then if $\pi^*$ solves  \eqref{equ:umdp}, it also solves the following worst-case constrained MDP problem
\begin{equation}\tag{\sf\small  WC-CMDP}
\begin{aligned}
\label{equ:rmdp}
    &\max_\pi \mathbb{E}\left[\sum_{t=0}^T {\gamma^t}r(s_t,a_t)|s_0,\pi\right] \text{ s.t. }\quad \sum_{s_t \in \tau } d(s_t)\leq c_{max},~\forall \tau\sim \pi.
\end{aligned}
\end{equation}
As a result, $\pi^*$ is feasible to the risk-neural CMDP \eqref{equ:cmdp}.  
\end{theorem}
The above theorem implies that if we set the penalties to be very large (e.g., $\infty$), then all the trajectories generated by the optimal policy $\pi^*$ will satisfy the constraint, i.e., the accumulated cost will not exceed $c_{max}$. Such a conservative policy would be useful in critical environments where the agent is strictly not allowed to go beyond the maximum allowed cost $c_{max}$. An example would be a routing problem for electrical cars where the remaining energy needs not become empty before reaching a charging station or the destination.  
Note that the worst-case CMDP \eqref{equ:rmdp} would be \textit{non-stationary} and \textit{history-dependent}, i.e., there would be no stationary and history-independent policies being optimal for the worst-case CMDP \eqref{equ:rmdp}. This remark is obviously seen, as at a stage, one needs to consider the current accumulated cost to make feasible actions. Thus, a policy that ignores the historical states and actions would be not optimal (or even not feasible) for the worst-case MDP.
As a result, this worst-case CMDP can not be presented by a standard constrained MDP formulation.  

Theorem \ref{th:UMDP-WCMDP} also tells us that one can get a feasible solution to the risk-neural CMDP
 \eqref{equ:cmdp} by just raising $\lambda$ to infinity. In fact, $\lambda$ does not need to be infinite to achieve feasibility. 
 Below we establish a lower bound for the penalty parameter $\lambda$ such that a solution to \eqref{equ:umdp} is always feasible to the risk-neural CMDP \eqref{equ:cmdp}. Let us define $\Psi^*$ as the optimal value of the unconstrained MDP problem 
 $$\Psi^*  = \max_\pi \mathbb{E}\left[\sum_{t=0}^T {\gamma^t}r(s_t,a_t)|s_0,\pi\right].$$  
 and $\overline{\Psi}$ be the optimal value of the worst-case CMDP \eqref{equ:rmdp}. We define a conditional expectation  $\widetilde{\bbE}_{\pi}\left[D(\tau)|~ D(\tau)\leq c_{\max}\right]$ as the expected cost over trajectories whose costs are less than $c_{max}$
 \[
\widetilde{\bbE}_{\pi}\left[D(\tau)|~ D(\tau)\leq c_{\max}\right] = \sum_{\tau|~ D(\tau)\leq c_{max}} P_{\pi}(\tau) D(\tau)
 \]
 where $P_{\pi}(\tau)$ is the probability of $\tau$ under policy $\pi$. Before presenting the bound, we first need two lemmas. Lemma \ref{lm:lm1}
 establishes a condition under which a policy $\pi$ is feasible to the RN-CMDP. 
\begin{lemma}
\label{lm:lm1}
Let $\phi^* = c_{max} -\max_{\pi}\left\{\tbbE_{\pi}[D(\tau)|~ D(\tau) \leq  c_{max}]\right\}$. 
Given  any policy $\pi$,  if $\widetilde{\bbE}_{\pi}[D(\tau)|~ D(\tau) > c_{max}]\leq \phi^*$, then
$\bbE_{\pi}[D(\tau)]\leq c_{max}$.
\end{lemma}
Lemma \ref{lm:lm2} below further provides an upper bound for the expected cost of violating trajectories under an optimal policy given by the extended MDP reformulation \eqref{equ:umdp}. 
\begin{lemma}
\label{lm:lm2}
Given $\lambda>0$, let $\pi^*$ be an optimal solution to \eqref{equ:umdp}. We have
\[
 \tbbE_{\pi^*}\left[D(\tau)|~ D(\tau)>c_{max}\right] \leq \frac{\Psi^*-\overline{\Psi}}{\lambda}.
\]
\end{lemma}
Using Lemmas \ref{lm:lm1} and \ref{lm:lm2}, we are  ready to state the main result in Theorem \ref{th:RN-CMDP} below. 
\begin{theorem}[Connection to the risk-neural CMDP]
\label{th:RN-CMDP}
For any $\lambda \geq \frac{\Psi^*-\overline{\Psi}}{\phi^*}$, 
a solution to  \eqref{equ:umdp} is always feasible to the RN-CMDP \eqref{equ:cmdp}. 
\end{theorem}
To prove Lemmas \ref{lm:lm1}, \ref{lm:lm2}, we leverage that objective of \eqref{equ:umdp} can be written equivalently as
\begin{equation}
 \label{eq:new-obj}
 \bbE_{\pi}\left[\sum_{t}\gamma^t r(s_t,a_t)\right] - \lambda\widetilde{\bbE}_{\pi}\left[D(\tau)|~ D(\tau)> c_{\max}\right]
\end{equation}
which allows us to establish a relation between $\lambda$ and $\widetilde{\bbE}_{\pi^*}\left[D(\tau)|~ D(\tau)> c_{\max}\right]$, where $\pi^*$ is an optimal policy of \eqref{equ:umdp}. The bounds then come from this relation. We refer the reader to the appendix for detailed proofs. 
 
There is also a lower bound for $\lambda$ from which  any solution to \eqref{equ:umdp} always satisfies a chance constraint (or VaR). To state this result, we define the following VaR CMDP, for risk level $\alpha\in[0,1]$. 
\begin{equation}\tag{\sf\small VaR-CMDP}
\begin{aligned}
\label{equ:var-mdp}
      &\max_\pi \mathbb{E}\left[\sum_{t=0}^T {\gamma^t}r(s_t,a_t)|s_0,\pi\right] \text{  s.t.  } \quad  P_{\pi}\Big[ D(\tau)>c_{\max}\Big] \leq \alpha.
\end{aligned}
\end{equation}
We have the following theorem showing a  connection between \eqref{equ:umdp} and the VaR CMDP above.
\begin{theorem}[Connection to VaR CMDP]
\label{th:VAR-CMDP}
For any $\lambda \geq {(\Psi^*-\overline{\Psi})/}{(\alpha c_{max})
}$,
a solution to  \eqref{equ:umdp} is always feasible to  \eqref{equ:var-mdp}.   
\end{theorem}
We also leverage Eq. \ref{eq:new-obj} to prove the theorem by showing that when $\lambda$ is sufficiently large, the conditional expectation $\widetilde{\bbE}_{\pi^*}\left[D(\tau)|~ D(\tau)> c_{\max}\right]$ can be bounded from  above ($\pi^*$ is an optimal policy of \eqref{equ:umdp}). We then can link this to the chance constraint by noting that  $\widetilde{\bbE}_{\pi^*}\left[D(\tau)|~ D(\tau)> c_{\max}\right] \geq c_{max} P(D(\tau)>c_{\max})$. 

Theorem \ref{th:VAR-CMDP} tells us that one can just raise $\lambda$ to a sufficiently large value to meet a chance constraint of any risk level. 
Here, Theorem \ref{th:VAR-CMDP} only guarantees feasibility to \eqref{equ:var-mdp}. Interestingly, if we modify the reward penalties by making them independent of the costs $d(s)$, than an equivalent mapping to \eqref{equ:var-mdp} can be obtained. Specifically, let us re-define the  following reward for the extended MDP.
That is, we replace the cost $d(s_t)$ by a constant. Theorem \ref{prop:var} below shows that \eqref{equ:umdp} is actually equivalent to a chance-constrained CMDP under the new reward setting. 
\begin{theorem}
[VaR equivalence]\label{prop:var}
If we modify the reward penalties as 
\begin{equation*}\label{eq:new-penalty}   
\widetilde{r}(a_t|(s_t,c_t)) = \begin{cases}
r(a_t|s_t) &\text{ if }c_t+d(s_t)\leq c_{max} \\
r(a_t|s_t) - {\lambda (t+1)/\gamma^t} &{\text{ if } c_t \leq c_{max} \text{ and }c_t+d(s_t)> c_{max}} \\
{r(a_t|s_t) - \lambda/\gamma^t } &\text{ if } c_t > c_{max} 
\end{cases}
\end{equation*}
then if $\pi^*$ is an optimal solution to \eqref{equ:umdp}, then there is $\alpha^\lambda \in [0;\frac{\Psi^*-\overline{\Psi}}{\lambda T}]$ ($\alpha$ is dependent of $\lambda$) such that $\pi^*$ is also optimal to \eqref{equ:var-mdp}. Moreover $\lim_{\lambda\rightarrow \infty}\alpha^\lambda = 0$.  
\end{theorem}
It can be also seen that Theorem \ref{th:UMDP-WCMDP} is a special case of  Theorem \ref{prop:var} when $\lambda = \infty$.

We finally connect \eqref{equ:umdp} with a risk-averse CMDP that has a CVaR intuition.  The theorem below shows that, by slightly changing the reward penalties,  \eqref{equ:umdp} actually solves a risk-averse CMDP problem. 
\begin{theorem}[CVaR CMDP equivalence]
\label{th:CVAR}
If we modify the reward penalties as \[
\widetilde{r}(a_t|(s_t,c_t)) = \begin{cases}
 r(a_t|s_t) & \text{ if } c_t+d(s_t)\leq c_{max} \\
 r(a_t|s_t) - {\lambda (c_t+d(s_t)-c_{max})/\gamma^t} & {\text{ if } c_t \leq c_{max} \text{ and }c_t+d(s_t)> c_{max}} \\
 {r(a_t|s_t) - \lambda d(s_t)/\gamma^t} & \text{ if } c_t > c_{max} 
\end{cases}
\]
then for any $\lambda>0$, there is $\beta^\lambda\in \left[0;\frac{\Psi^*-\overline{\Psi}}{\lambda}\right]$ ($\beta^\lambda$ is dependent of $\lambda$) such that 
any optimal solution to the extended CMDP \eqref{equ:umdp} is also optimal to the following risk-averse CMDP
\begin{equation}\tag{\sf\small CVaR-CMDP}
\begin{aligned}
\label{equ:cvar-mdp}
   &\max_\pi \mathbb{E}\left[\sum_{t=0}^T {\gamma^t}r(s_t,a_t)|s_0,\pi\right] \text{ s.t. }\quad  \bbE_{\tau\sim \pi}\Big[ (D(\tau)-c_{max})^+ \Big] \leq \beta^\lambda.
\end{aligned}
\end{equation}
 Moreover, $\lim_{\lambda\rightarrow \infty} \beta^\lambda = 0$. 
\end{theorem}
In practice, since $\lambda$ is just a scalar, one can just gradually increase it from 0 to get feasible policies or decrease it from a large value if the policy becomes too conservative.
This indicates the generality of the unconstrained extended MDP formulation \eqref{equ:umdp}. In summary, we show that \eqref{equ:umdp}
brings risk-neural, worst-case and VaR and CVaR CMDPs in \eqref{equ:cmdp}, \eqref{equ:rmdp}, \eqref{equ:var-mdp} and \eqref{equ:cvar-mdp} under one umbrella. 
% We illustrate this in Fig. \ref{fig:CMDP}.
% \begin{figure}[htb]
%     \centering
%     \includegraphics[width=1\linewidth]{CMDP.pdf}
%     \caption{The unconstrained  MDP with cost-augmented state space walks through different variants of constrained MDP by raising the penalty parameter $\lambda$. }
%     \label{fig:CMDP}
% \end{figure}

%%%%%%%%%%%%%%%%%%%%%%%%%%%
%%%%%%%%%%%%%%%%%%%%%%%%%%%
\section{Safe RL Algorithms}
In this section, we update existing RL methods to effectively utilize the extended state space and reward penalties, while considering RN-CMDP. Due to the theoretical finding in the previous section, just by tweaking $\lambda$, we can solve  Constrained MDPs with different types of constraints, e.g., expected cost, VaR or CVar CMDPs.  
\subsection{Safe DQN}
Deep Q Network (DQN) \cite{mnih2015human} is an efficient method to learn in primarily discrete action RL  problems. However, the original DQN does not consider safety constraints and cannot be applied to any of the CMDP variants. The main modifications in the updated algorithm, referred to as Safe DQN are with regards to exploiting the extended state space and the reward penalties based on constraint violations. The pseudo code for the Safe DQN algorithm is provided in the appendix.

The impact of extended state space on the algorithm can be observed in almost every line of the algorithm. When selecting an action (Line 4), Safe DQN does not consider the feasibility of the action with respect to cost. Instead, like in the original DQN, it is purely based on the current Q value. The assumption is that the penalties accrued due to violation (in Lines 9-12) will be sufficient to force the agent away from cost-infeasible actions.  Once the new rewards are obtained (based on considering reward penalties), the Q network is updated using the mean square error loss in Line 17. To avoid conservative decisions, we dynamically change the constraint penalty $\lambda$ (in Lines 21-22). We set an initial value for $\lambda$. During the
training, we evaluate the maximum final cost in the recent few episodes, if the maximum final cost exceeds $c_{max}$, we make no change to $\lambda$, otherwise, we set $\lambda$ to be $0.95\times\lambda$ to diminish the conservativeness of the policy. However, the value of $\lambda$ cannot keep decreasing, so we set a lower bound for it. We refer the reader to the appendix for more details.

 \subsection{Safe SAC}
Soft Actor-Critic (SAC) \cite{haarnoja2018soft} is an off-policy algorithm that learns a stochastic policy for discrete and continuous action RL problems.  SAC employs policy entropy in conjunction with value function to ensure a better tradeoff between exploration and exploitation. The Q value function in SAC is defined as follows:
\begin{equation}
\begin{aligned}
\label{equ:SACq}
    Q(s,a)=& \mathbb{E}[\sum_{t=0}^\infty \gamma^t r(s_t,a_t) + \alpha \sum_{t=1}^\infty \gamma^t H(\pi(\cdot|s_t))|s_0=s,a_0=a]
\end{aligned}
\end{equation}
where $H(.)$ denotes the entropy of the action distribution for a given state $s_t$.
% In SAC, we consider entropy defined in Equation\ref{equ:entropy}, where $x$ is a variable with distribution $P(x)$ while $H(x)$ denotes the entropy of $x$.
% \begin{equation}
% \label{equ:entropy}
%     H(P)=\mathbb{E}_{x\sim P}[-\log P(x)]
% \end{equation}
SAC also employs a double Q-trick, i.e., two Q value functions ($Q^i(.),~i \in {1,2}$) are trained and we take the minimum of these two Q value functions as the target to avoid overestimation: 
\begin{equation}
\begin{aligned}
\label{equ:SACtar}
    y=&r(s,a)+\gamma\min_{i=1,2} Q^{i}(s',\tilde{a}')-\alpha \log\pi(\tilde{a}'|s')
\end{aligned}
\end{equation}
where $\tilde{a}'\sim \pi(\cdot|s')$.

Our algorithm that handles safety constraints is referred to as Safe SAC. It builds on SAC by having an extended state space and a new action selection strategy that exploits the extended state space. In Safe DQN,  we primarily rely on the violation of constraints, so as to learn about the bad trajectories and avoid them. While such an approach works well for discrete action spaces and in an off-policy setting, it is sample inefficient and can be slow for on-policy (actor-critic) settings. In Safe SAC, apart from the reward penalty, we also focus on learning feasible actions, which are generated through the use of the cost accumulated so far (available as part of the state space) and an estimate of Q value on the future cost. 

Formally, we define the optimization to select safe actions (at each decision epoch) in (\ref{equ:constraint}). Extending on the double Q trick for reward, we also have double Q for future cost, referred to as $\{Q_D^i\}_{i \in {1,2}}$. At each step, the objective is to pick an action that would maximize the reward Q value for the extended state and action minus the weighted entropy of the action. The constraint ensures that we only pick those actions that will not violate the cost constraint.  Specifically, in the left-hand side of the constraint, we calculate the overall expected cost using :
(a) (estimate) of the future cost, from the current state; 
(b) (actual) cost incurred so far; and (c) subtracting the (actual) cost incurred at the current step, as it is part of both (a) and (b);
{\small \begin{equation}
\begin{aligned}
\label{equ:constraint}
    &\arg\max_a \min_{i=1,2} Q^i((s,c),a)-\alpha \log\pi(a|(s,c)) \textbf{ s.t. } \max_{i=1,2} Q^i_D((s,c),a)+c-d((s,c))\leq c_{max},\forall (s,c)
\end{aligned}
\end{equation}}
The detailed pseudocode for Safe SAC is provided in the appendix.

\section{Experimental Results}
In this section, we intend to experimentally demonstrate different facets of our approaches (Safe SAC and Safe DQN) in
\squishlist
    \item Solving RN-CMDP: We show expected reward performance and expected cost constraint enforcement on multiple benchmark problems. In this case, we compare against the BVF ~\cite{satija2020constrained}, Lyapunov ~\cite{chow2019lyapunov} and an unsafe approach (original DQN)~\cite{mnih2015human}. 
    \item Solving CVaR-CMDP: We show expected reward performance and CVaR cost constraint enforcement on multiple benchmark problems. In this case, there is only one benchmark approach, namely WCSAC~\cite{yang2021wcsac}. 
    \item Ablation analysis: We show results that demonstrate the importance of considering state augmentation and reward penalty simultaneously. 
    \item Impact of reward penalty: We show results that highlight the importance of the reward penalty $\lambda$ in ensuring the right tradeoff between expected reward performance and cost constraint enforcement. 
\squishend
The performance values (expected cost and expected reward) along with the standard deviation in each experiment are averaged over 5 runs. We did not provide any results with Var-CMDP because there are no existing approaches to compare agaisnt that have been developed for it. 
%We empirically compare the performance of our approaches on both discrete and continuous environments with respect to expected reward and expected cost achieved against leading benchmark approaches. For an RL benchmark, we use the original DQN~\cite{mnih2015human} and it is referred to as Unsafe DQN, as it does not account for cost constraints. For leading Constrained RL benchmarks, we use BVF (Backward Value Function)~\cite{satija2020constrained} and Lyapunov~\cite{chow2019lyapunov}. We also provide the comparison of our safe method and CVaR (WCSAC)~\cite{yang2021wcsac} using CVaR constraint in the appendix. We do not use any Lagrangian-based methods as our benchmarks as all the benchmarks we use proves to perform better than Lagrangian. We show results with respect to expected cost constraint, as there are many model free approaches that solve the RN-CMDP problem. The performance values (expected cost and expected reward) along with the standard deviation in each experiment are averaged over 5 runs.
\begin{figure*}[htbp]
    \centering
    \begin{minipage}{0.3\linewidth}
        \centering
        %\caption*{}
        \includegraphics[scale=0.35]{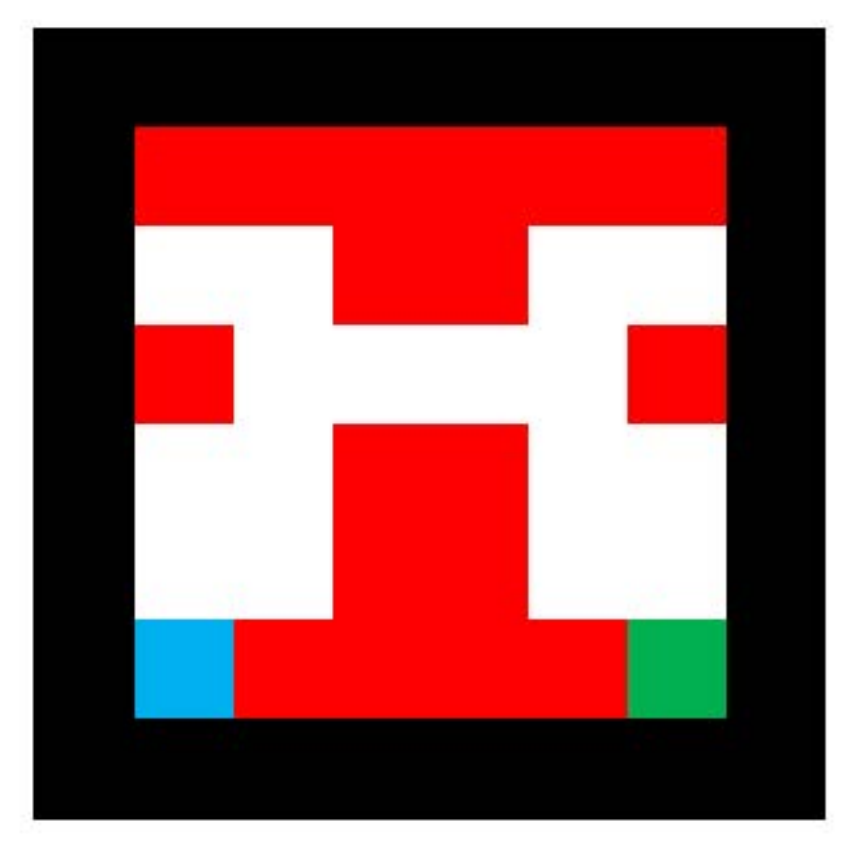}
        \label{grid}
    \end{minipage}\hfill
    \centering
    \begin{minipage}{0.34\linewidth}
        \centering
        %\caption*{}
        \includegraphics[scale=0.33]{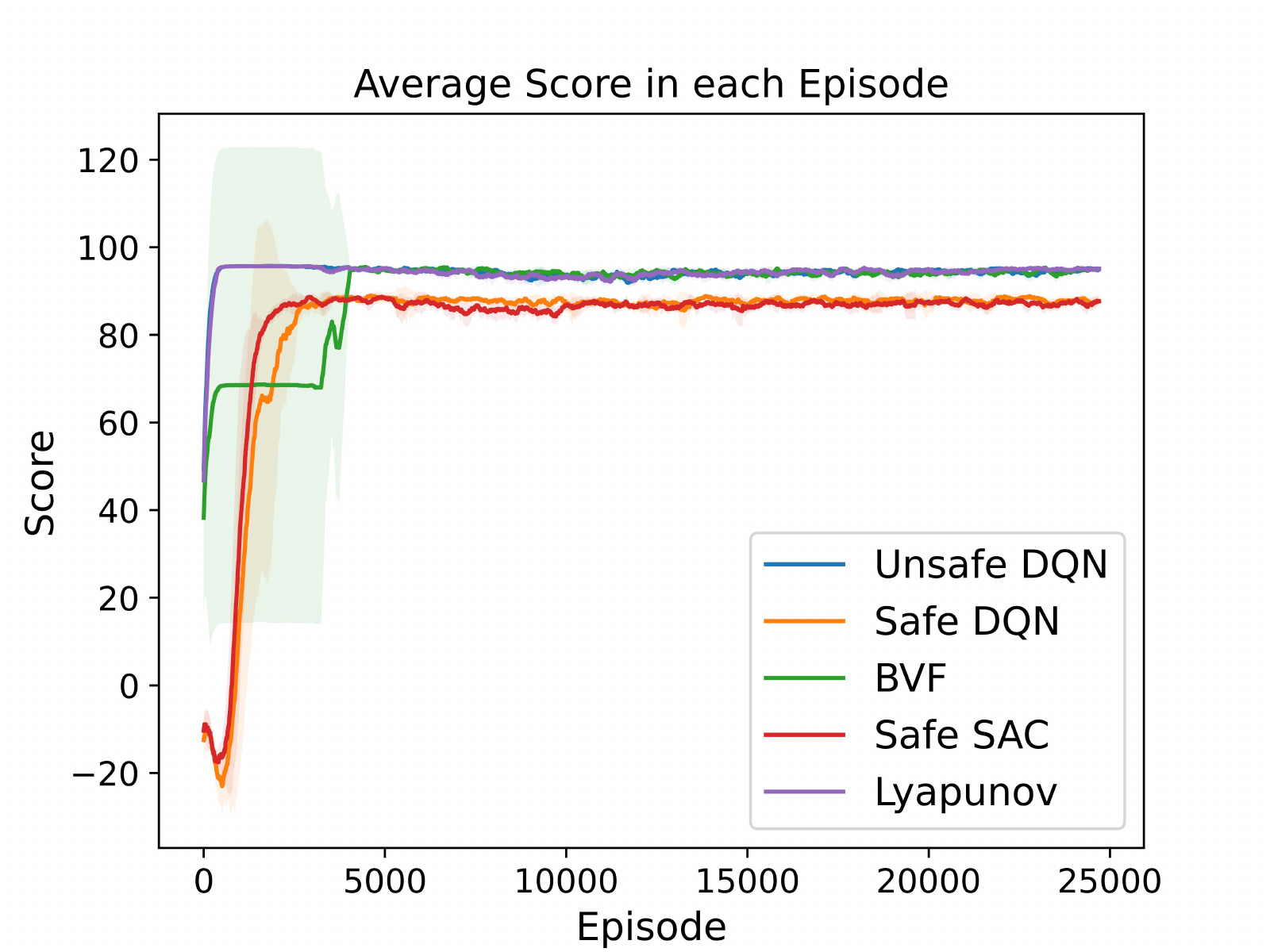}
        \label{gridscore}
    \end{minipage}\hfill
    \begin{minipage}{0.34\linewidth}
        \centering
        %\caption*{}
        \includegraphics[scale=0.33]{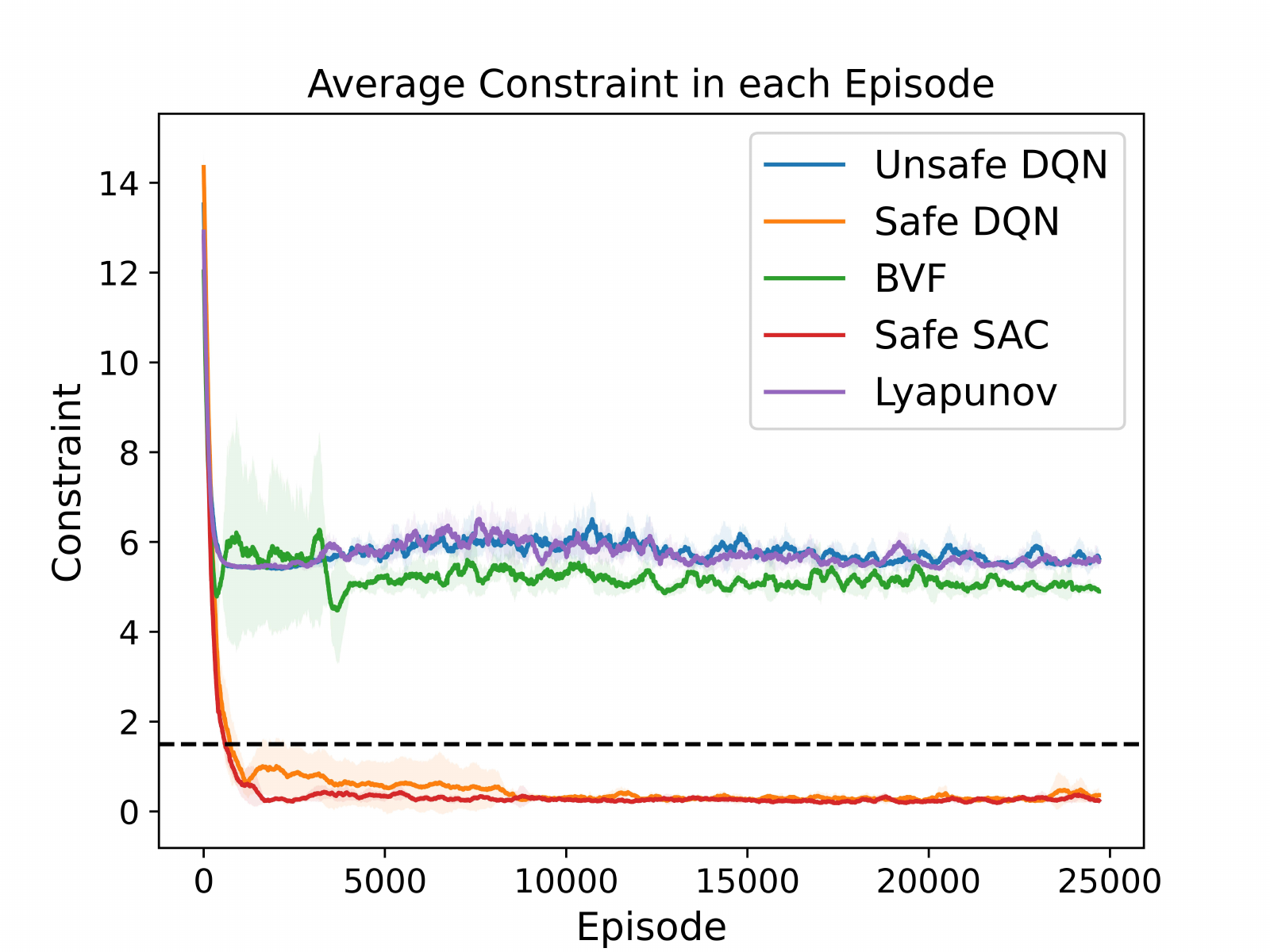}
        \label{gridcons}
    \end{minipage}\hfill
    \caption{Gridworld environment and reward, cost comparison of different approaches.}
    \label{fig:gridperform}
\end{figure*}
\subsection{RN-CMDP}
For a discrete state and discrete action environment, we consider the stochastic 2D grid world problem introduced in previous CMDP works~\cite{leike2017ai,chow2018lyapunov,satija2020constrained,jain2021safe}. The grid on the left of Figure \ref{fig:gridperform} shows the environment. The agent starts at the bottom right corner of the map (green cell) and the objective is to move to the goal at the bottom left corner (blue cell). The agent can only move in the adjoining cells in the cardinal directions. Occasionally agent will execute a random action with probability $p=0.05$ instead of the one selected by the agent. It gets a reward of +100 on reaching the goal, and a penalty of -1 at every time step. There are a number of pits in the map (red cell) and agent gets a random cost ranging from 1 to 1.5 on passing through any pit cell. We consider an 8x8 grid and the maximum time horizon is 200 steps, after which the episode terminates. This modified GridWorld environment is challenging because the agent can travel to destination via a short path with a high cost, but if it wishes to travel safely, it needs to explore enough to find a safe path which is far from the shortest one. We set the expected cost threshold, $c_{max}=2$, meaning agent could pass at most one pit. %For discrete state environments, we use the discrete SAC in \cite{christodoulou2019soft}.

%For objective in Equation \ref{equ:SACo}, we consider the action policy network chooses instead of expected return. We define objective in discrete SAC in Equation \ref{equ:DisSACo}.
% \begin{equation}
% \label{equ:DisSACo}
%     \max_a \pi(a|(s,c))^T[\min_{i=1,2} Q^i((s,c),a)-\alpha \log\pi(a|(s,c))]
% \end{equation}

Figure \ref{fig:gridperform} shows the performance of each method with respect to expected reward (score) and expected cost (constraint): (a) With respect to expected reward, among safe approaches, Lyapunov achieves the highest reward. However, it violates the expected cost constraint by more than twice the cost constraint value. (b) Safe SAC and Safe DQN achieve similar expected reward values, though Safe SAC reaches there faster. This high expected reward value is achieved while satisfying the expected cost constraint after ~1000 episodes. (c) The other constrained RL approach, BVF is the last to converge while not being able to satisfy the expected cost constraint. (d) As expected, Unsafe DQN achieved the highest expected reward but was unable to satisfy the expected cost constraint.

\begin{figure*}[htbp]
    \centering
    \begin{minipage}{0.3\linewidth}
        \centering
        %\caption*{}
        \includegraphics[scale=0.35]{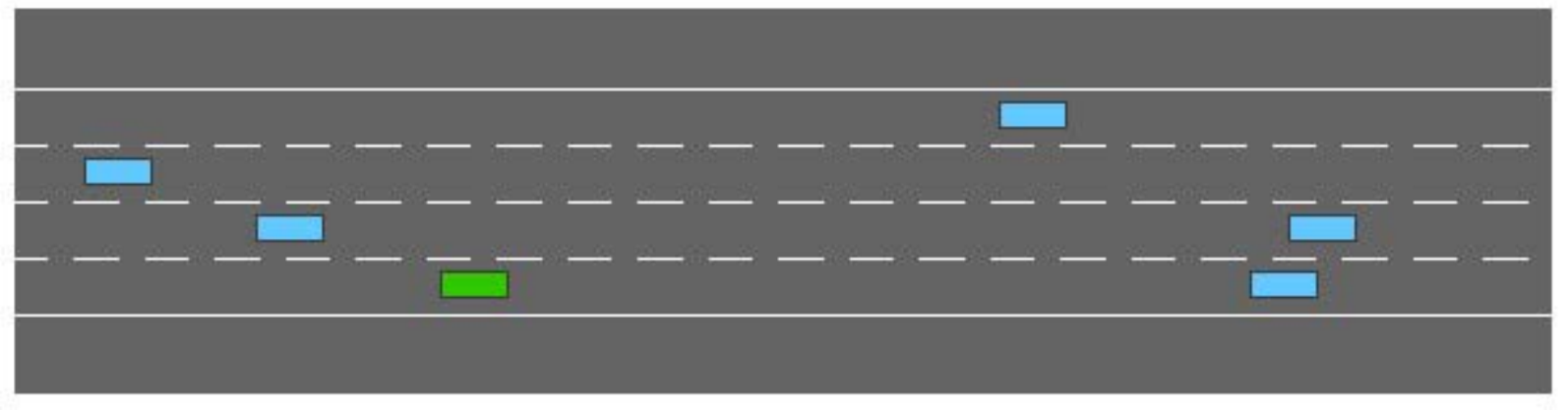}
        \label{highway}
    \end{minipage}\hfill
    \begin{minipage}{0.34\linewidth}
        \centering
        %\caption*{}
        \includegraphics[scale=0.33]{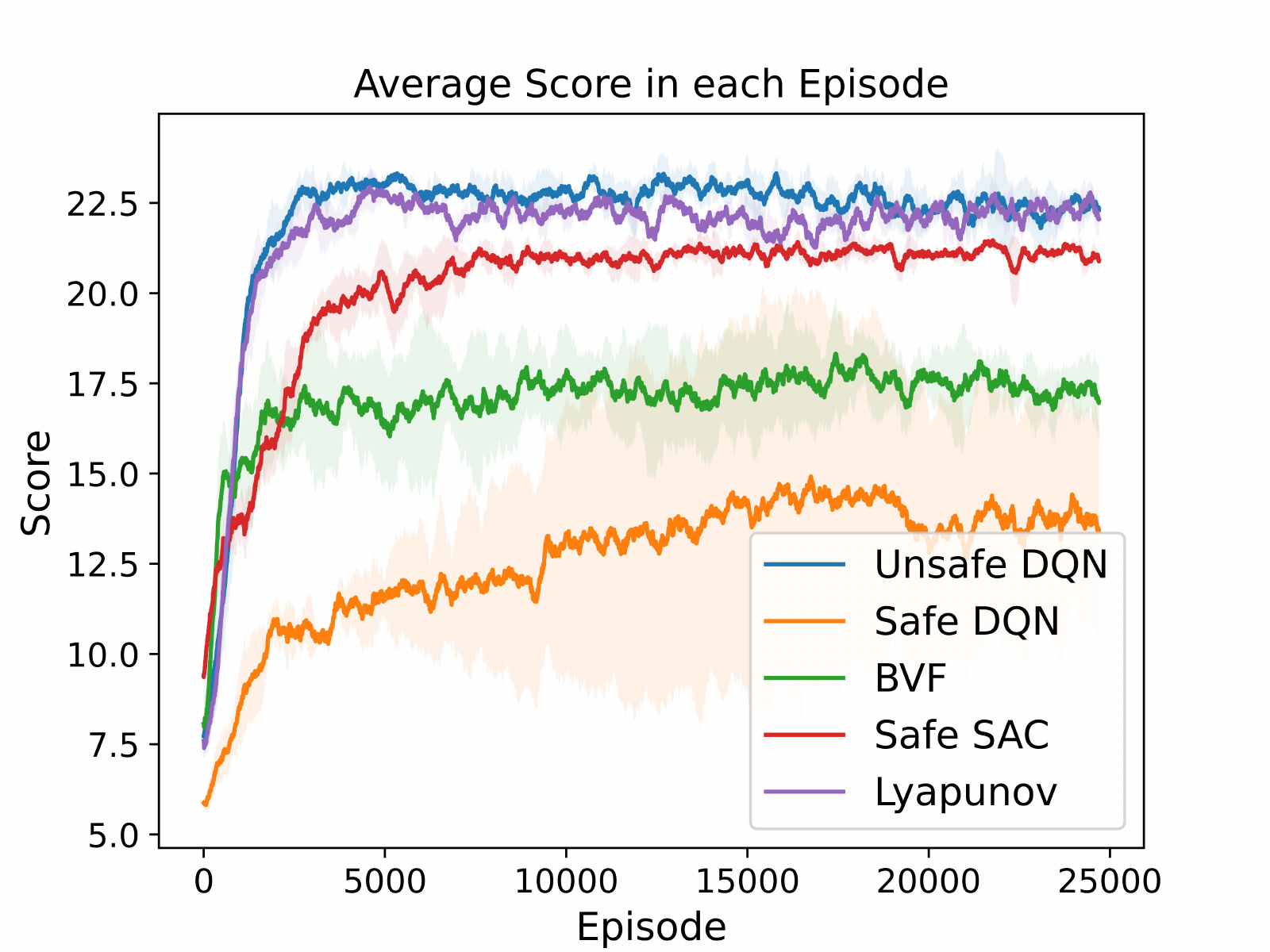}
        \label{highwayscore}
    \end{minipage}\hfill
    \begin{minipage}{0.34\linewidth}
        \centering
        %\caption*{}
        \includegraphics[scale=0.33]{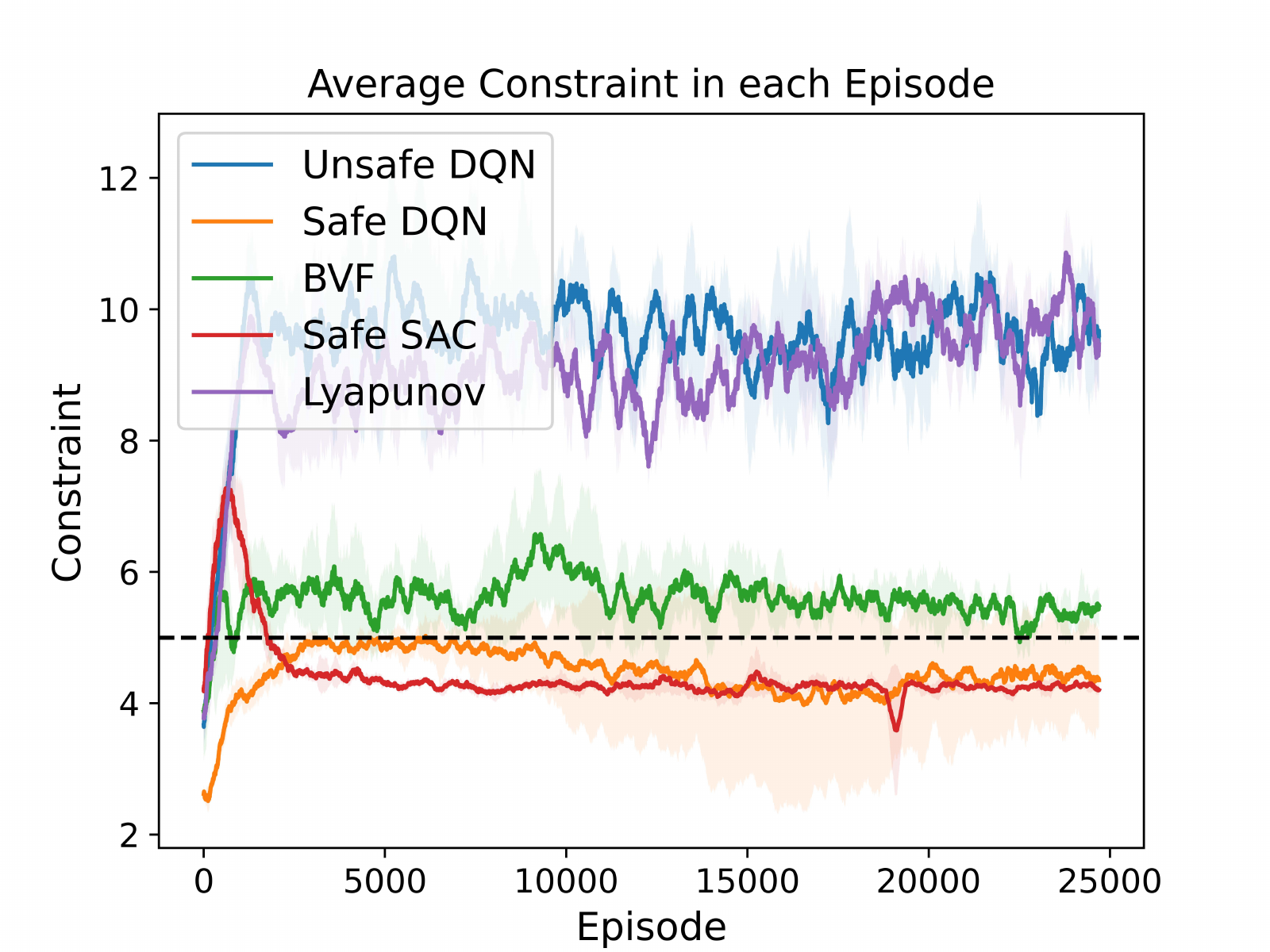}
        \label{highwaycons}
    \end{minipage}\hfill
    \caption{Highway environment and reward, cost comparison of different approaches}
    \label{fig:highwayperform}
\end{figure*}

Next, we consider the highway environment. Inspired by experiment in GPIRL \cite{levine2011nonlinear}, we test our safe methods in the highway environment \cite{highway-env} of Figure~\ref{fig:highwayperform}. The task in highway environment is to navigate a car on a four-lane highway with all other
vehicles acting randomly. The goal for the agent is to maximize its reward while we show the reward settings in the appendix. However, to ensure safety, we set the constraint on the time the agent drives faster than a given speed in the rightmost lane. Figure \ref{fig:highwayperform} shows the expected reward and expected cost performance of our safe methods compared to that of the benchmarks. Safe SAC is able to get high expected rewards while satisfying the expected cost constraint. We also provide more results for RN-CMDP in appendix. 

\begin{figure*}[htbp]
    \centering
    \begin{minipage}{0.3\linewidth}
        \centering
        %\caption*{}
        \includegraphics[scale=0.35]{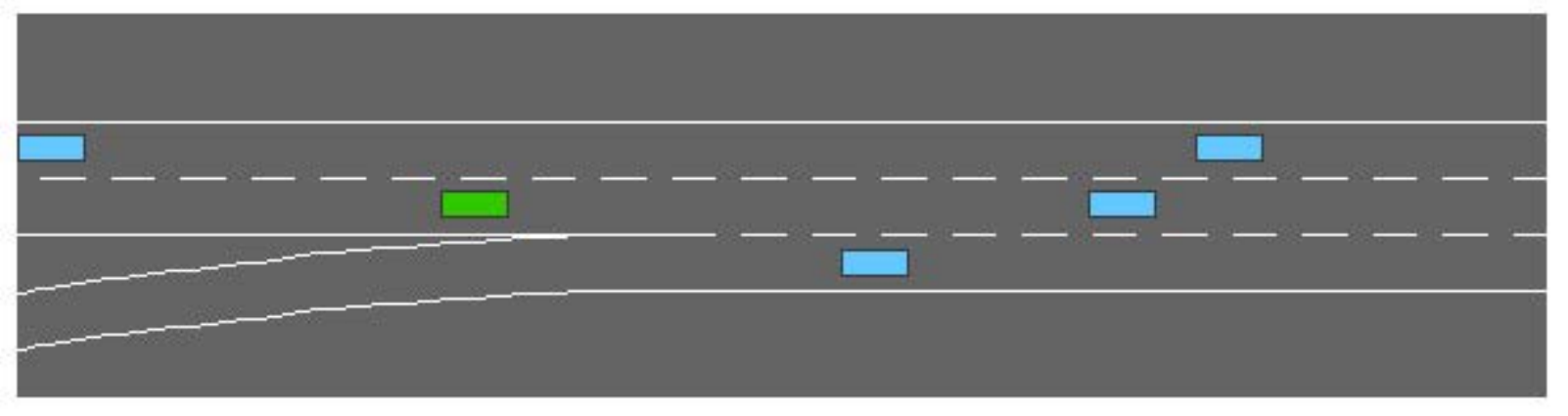}
        \label{mergeenv}
    \end{minipage}\hfill
    \begin{minipage}{0.34\linewidth}
        \centering
        %\caption*{}
        \includegraphics[scale=0.33]{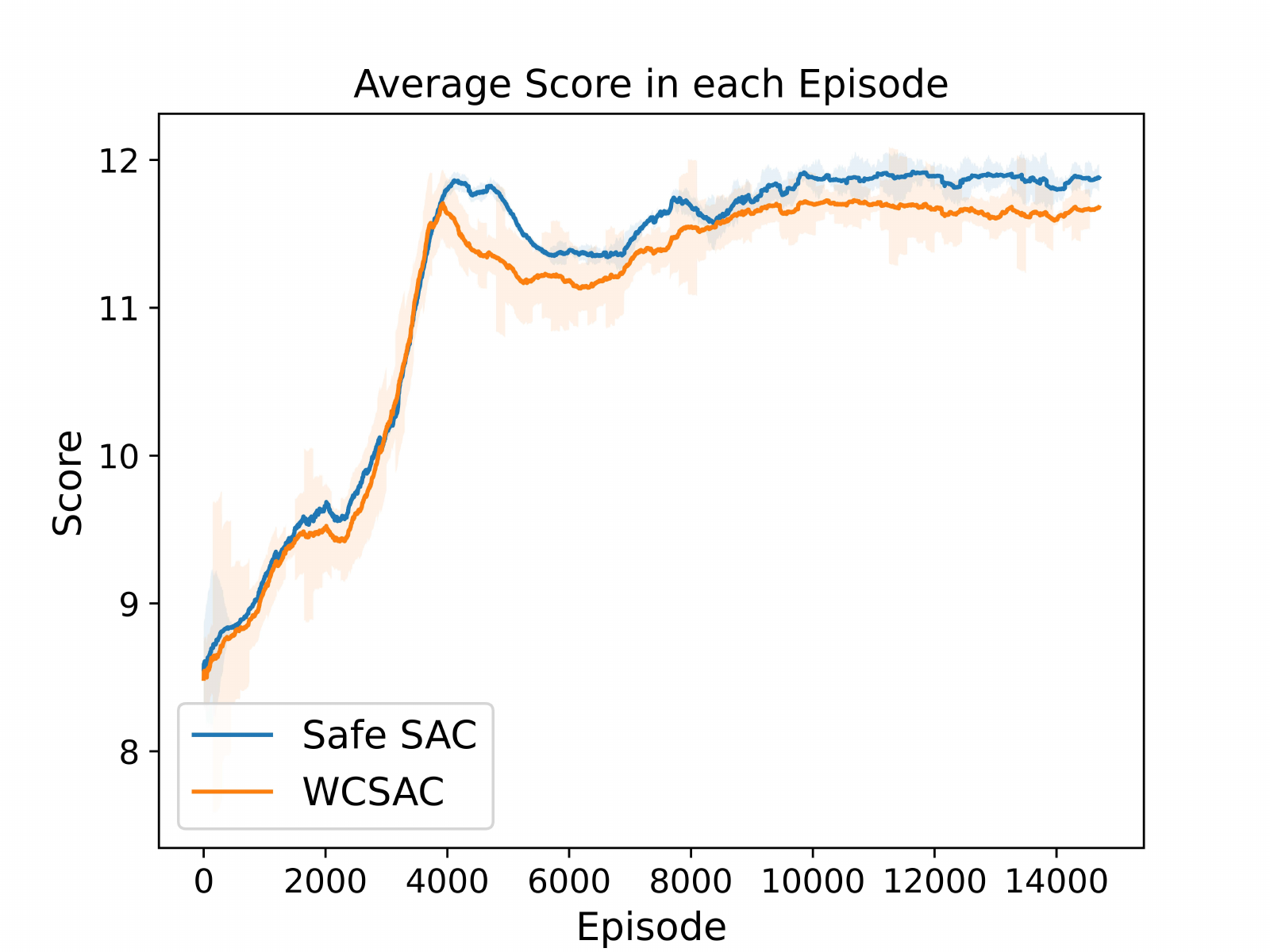}
        \label{mergecvarscore}
    \end{minipage}\hfill
    \begin{minipage}{0.34\linewidth}
        \centering
        %\caption*{}
        \includegraphics[scale=0.33]{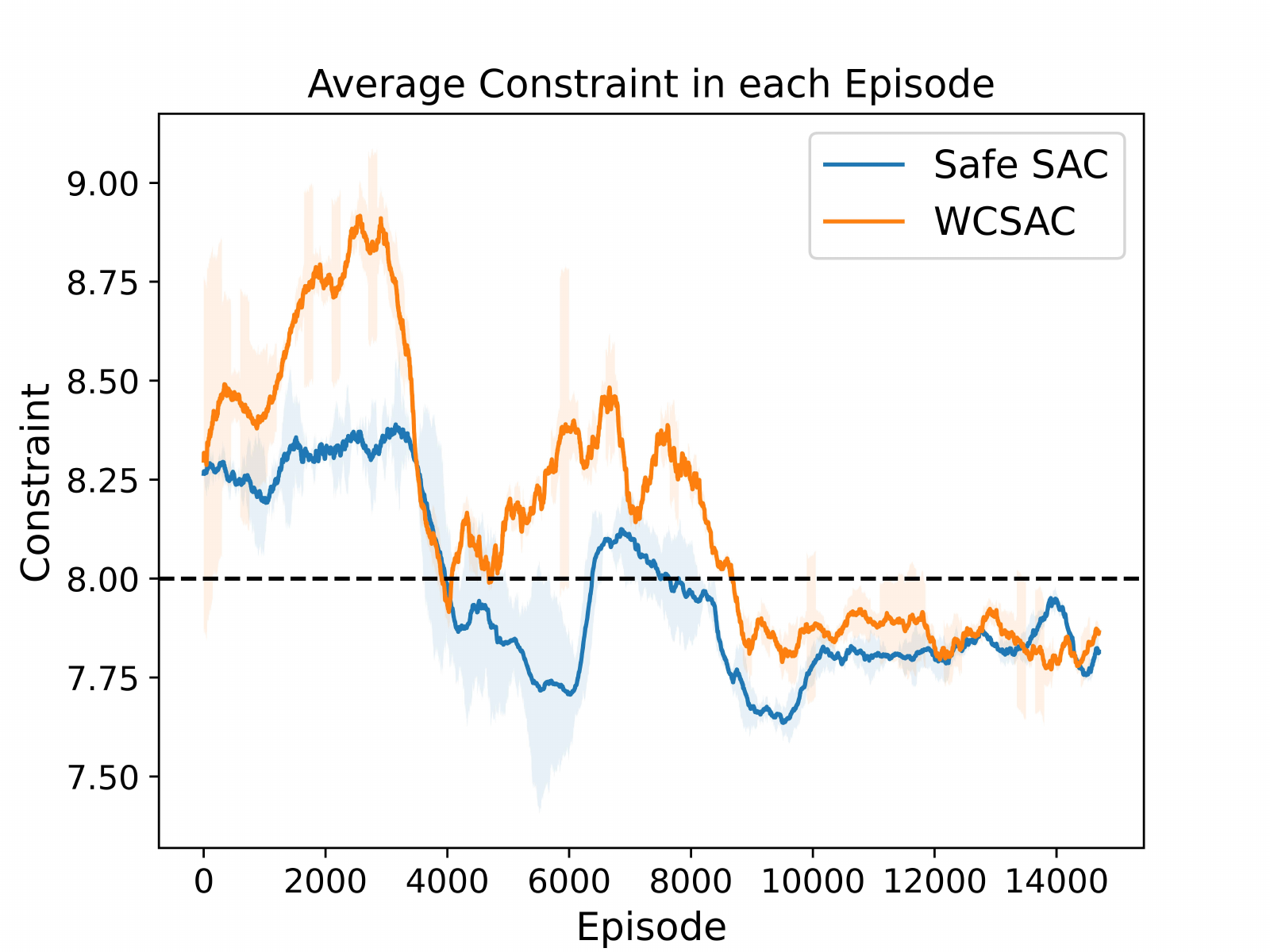}
        \label{mergecvarcons}
    \end{minipage}\hfill
    \caption{Experiment with CVaR Constraint in Merge Environment}
    \label{fig:mergecvar}
\end{figure*}

\subsection{CVaR-CMDP}
We introduce CVaR constraint to our Safe SAC method and compare it against WCSAC \cite{yang2021wcsac}, which is a leading algotrithm for CVaR constraint. As WCSAC is limited to continuous space, we only do the comparison with continuous environments. Figure \ref{fig:mergecvar} shows the comparisons on the merge benchmark problem from GPIRL \cite{levine2011nonlinear}. Safe SAC was able to marginally perform better than WCSAC both with respect to expected reward and expected cost. This performance is noteworthy as we use the same algorithm to handle both RN-CMDP and CVaR-CMDP. We provide more results for CVaR-CMDP in appendix.

\begin{wrapfigure}{r}{0.7\textwidth}
    \centering
    \begin{minipage}{0.33\linewidth}
        \centering
        %\caption*{}
        \includegraphics[scale=0.33]{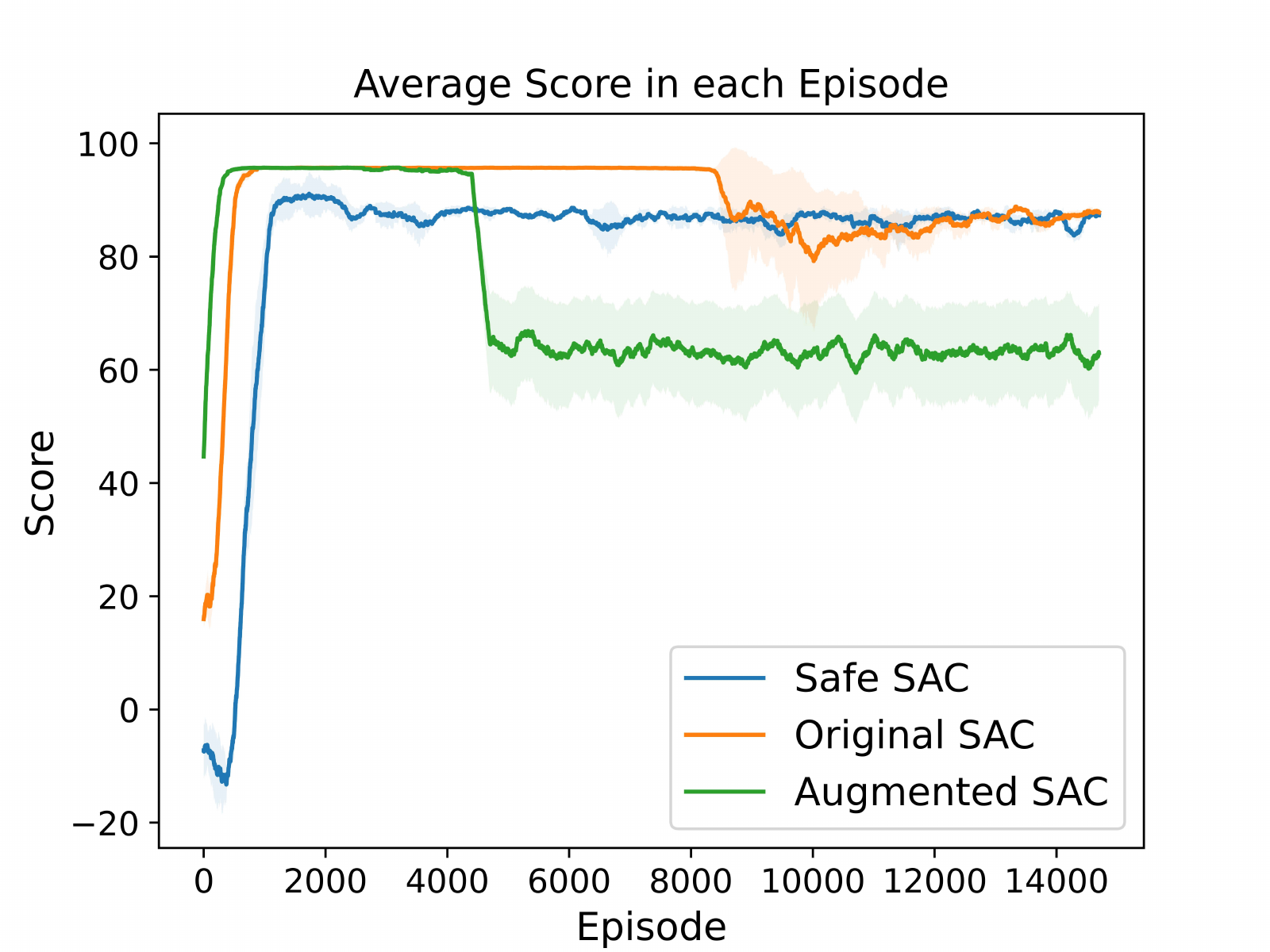}
        \label{comparescore}
    \end{minipage}\hfill
    \begin{minipage}{0.33\linewidth}
        \centering
        %\caption*{}
        \includegraphics[scale=0.33]{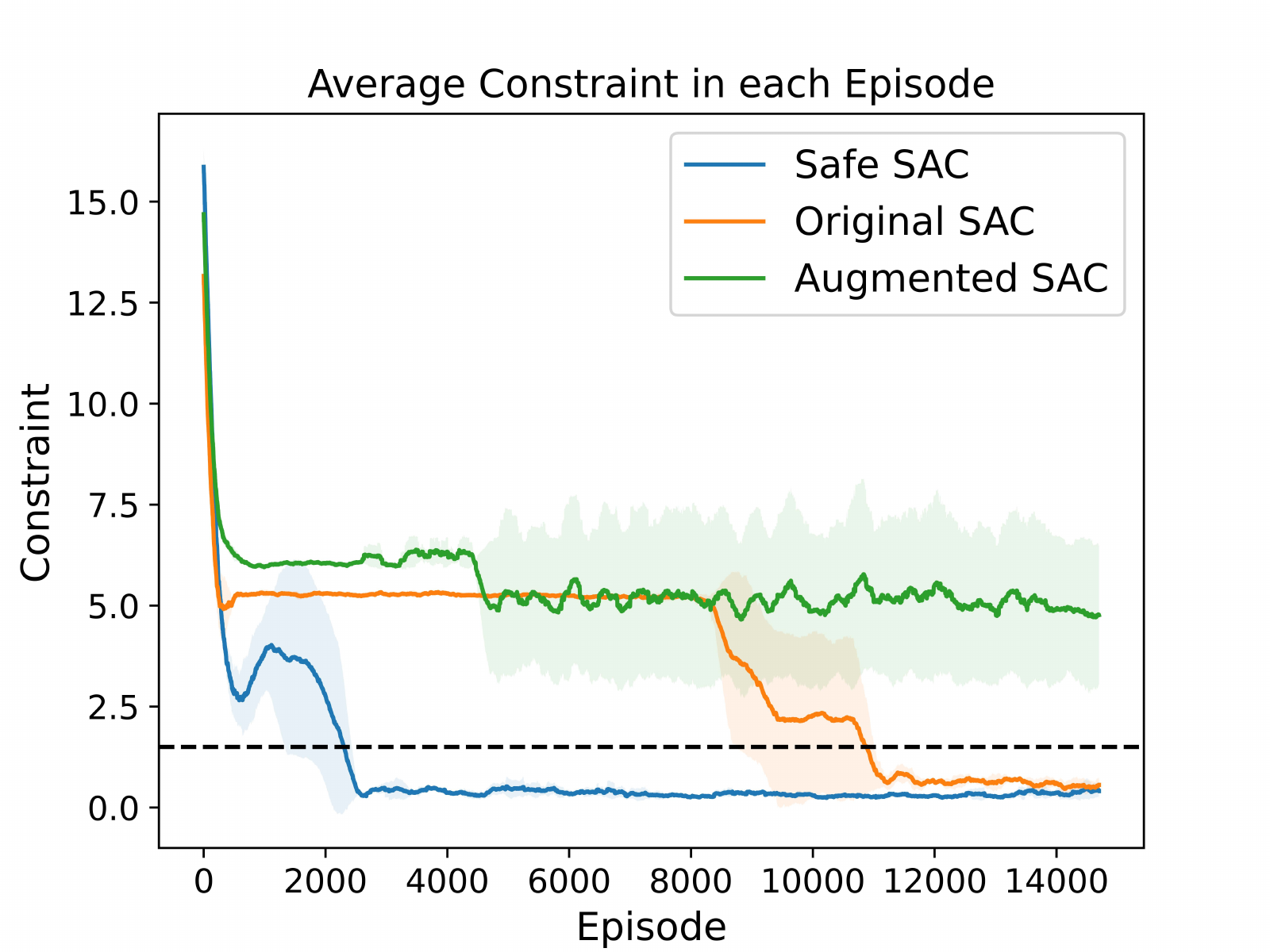}
        \label{comparecons}
    \end{minipage}\hfill
    \caption{Ablation Analysis with GridWorld}
    \label{fig:ablation1}
\end{wrapfigure}

 \subsection{Ablation Analysis}
To investigate the impact of state augmentation and reward penalty, we conduct an ablation analysis using GridWorld and Highway environments. We compare the performance of Original SAC, SAC with only state augmentation (Augmented SAC), and SAC with state augmentation and reward penalty (Safe SAC) on GridWorld in Figure \ref{fig:ablation1}. In Augmented SAC, the agent chooses an action using (\ref{equ:constraint}). If no action could satisfy the constraint, it chooses the action with minimum future cost. More details and results are provided in appendix. %We do not consider SAC with reward penalty for the reason: 1. SAC without state augmentation cannot store the local cost; 2. One state shares different constraints by choosing different trajectories, if we set constraint penalty for original state space, agent is unable to learn. Safe SAC is able to perform the best across the three approaches. More results in appendix %In concrete environment (GridWorld), agent can reach good performance and satisfy the constraint only with state augmentation and reward penalty together. In continuous environment (Highway), Augmented SAC is able to learn the safe policy but suffers conservative in costs while dynamic reward penalty in Safe SAC manages to solve it.

% \begin{equation}
% \begin{aligned}
% \label{equ:constraintaa}
%     &\arg\max_a \min_{i=1,2} Q^i((s,c),a)-\alpha \log\pi(a|(s,c))\\
%     s.t.&\max_{i=1,2} Q^i_D((s,c),a)+c-d((s,c))\leq c_{max},\forall (s,c)
% \end{aligned}
% \end{equation}

% \begin{figure}[htbp]
%     \centering
%     \begin{minipage}{0.5\linewidth}
%         \centering
%         %\caption*{}
%         \includegraphics[scale=0.45]{Highway_Scores_Aba.pdf}
%         \label{comparescore1}
%     \end{minipage}\hfill
%     \begin{minipage}{0.5\linewidth}
%         \centering
%         %\caption*{}
%         \includegraphics[scale=0.45]{Highway_Constraint_Aba.pdf}
%         \label{comparecons1}
%     \end{minipage}\hfill
%     \caption{Ablation Analysis in Highway}
%     \label{fig:ablation2}
% \end{figure}

\begin{wrapfigure}{r}{0.7\textwidth}
 %\begin{figure}[htbp]
    \centering
    \begin{minipage}{0.33\linewidth}
        \centering
        %\caption*{}
        \includegraphics[scale=0.33]{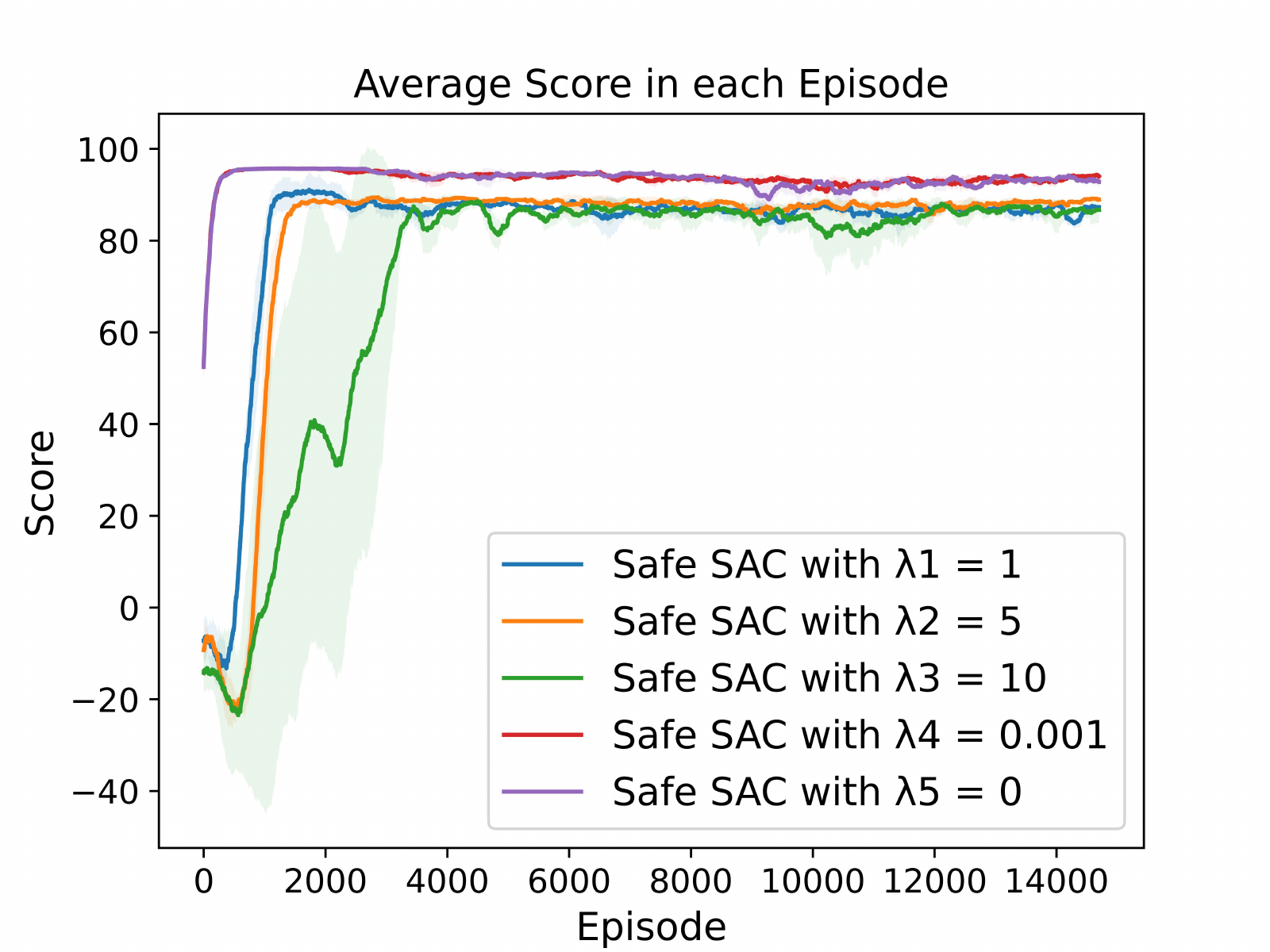}
        \label{consscore}
    \end{minipage}\hfill
    \begin{minipage}{0.33\linewidth}
        \centering
        %\caption*{}
        \includegraphics[scale=0.33]{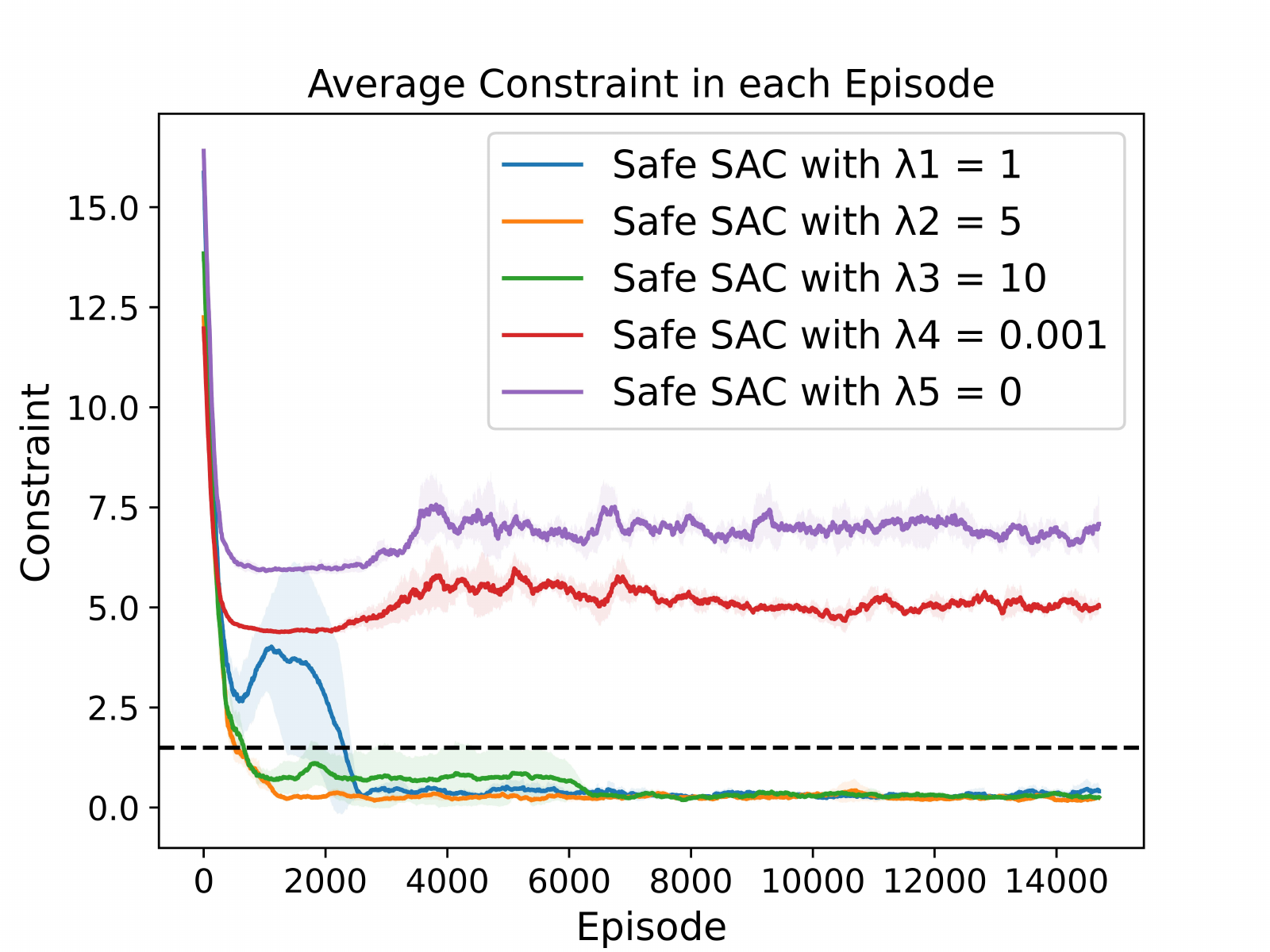}
        \label{conscons}
    \end{minipage}\hfill
    \caption{Experiment in GridWorld with Different Reward Penalties}
    \label{fig:constraintpenalty1}
\end{wrapfigure}

\subsection{Impact of Reward Penalty \texorpdfstring{$\lambda$}{}}
 To investigate the impact of different reward penalty values on the performance, we conduct experiments on GridWorld using Safe SAC, with  $\lambda_1 = 1,\lambda_2=5\lambda_1,\lambda_3=10\lambda_1$, a small $\lambda_4 = 0.001$ and $\lambda_5=0$, we show the results in Figure \ref{fig:constraintpenalty1}. In this experiment,
we fix $\lambda$ so that it would be conservative in some cases. We can see that $\lambda_1,\lambda_2,\lambda_3$ in GridWorld could be good choices to achieve good performance, implying that only with an appropriate $\lambda$ value, the agent can receive high rewards while satisfies the constraint with fast convergence speed. That is consistent with our theory and is why we set a threshold for $\lambda$ and dynamically change it during the safe algorithms.
Moreover, cases with a small $\lambda$ are close to those without $\lambda$. 
%Cases with large $\lambda$ are close to \eqref{equ:rmdp}, which does not allow constraint violation. 
We also provide the results on Highway environment in the appendix.

 \section{Conclusion}

In this paper, we have provided a very generic and scalable mechanism for handling a wide variety of cost constraints. 
%(expected cost, worst-case cost, VaR, CVaR) in CMDP.  
Lagrangian-based approaches, which penalize with respect to expected cost are unable to assign credit appropriately for a cost constraint violation, as expected cost averages over all trajectories. Instead, we propose to penalize with respect to individual reward while maintaining a cost-augmented state, thereby providing precise credit assignment with regard to cost constraint violations. We theoretically demonstrate that this simple cost-augmented state and reward penalized MDP
%(referred to as EMDP) 
can represent all the aforementioned cost constraints. We then provide Safe DQN and Safe SAC which are able to outperform leading expected-cost constrained RL approaches (Lyapunov and BVF)  while, at the same time, providing similar performance to leading approach for CVaR constrained RL (WCSAC). One limitation of our method is that the initial value and threshold of reward penalty might need to be tuned for each environment to achieve good performance, as these values would differ in different settings.

%\section*{Acknowledgements}

\bibliography{reference}
\bibliographystyle{abbrv}

\newpage
\onecolumn
\appendix
\section{Safe DQN Pseudocode}
Algorithm~\ref{alg:DQN} provides the pseudocode for the Safe DQN algorithm.

\begin{algorithm}[htbp]
 \caption{Safe DQN with Extended State Space}
 \label{alg:DQN}
 \textbf{Initialization:} Relay buffer $D$ with capacity $N$ and $D_c$ with capacity $M$, action-value function $Q$ with weight $\theta$, target action-value function $\hat{Q}$ with weight $\theta^-=\theta$, constraint penalty $\lambda$, lower bound for constraint penalty $\Lambda$.
 \begin{algorithmic}[1] %[1] enables line numbers
 \FOR{each episode $k$}
 \STATE Initialize with sequence $(s_0,c_0=0)$.
 \FOR{each time step $t$}
 \STATE Select a random action $a_t$ with probability $\epsilon$, otherwise select $a_t=arg\max_a Q((s_t,c_t),a;\theta)$.
 \STATE Execute action $a_t$, observe $(s_{t+1},c_{t+1}), r_t$.
 \STATE Store $((s_t,c_t),a_t,r_t,(s_{t+1},c_{t+1}))$ in $D$.
 \STATE Update state-cost pair to $(s_{t+1},c_{t+1})$.
 \STATE \textbf{Sample} $((s_j,c_j),a_j,r_j,(s_{j+1},c_{j+1}))$ from $D$.
  \IF{$c_{j}>c_{max}$}
 %\STATE Set objective 
 \STATE $\tilde{r}_j=r(s_j)-\lambda d(s_j)/\gamma^j$
 \ELSIF{$c_{j+1}>c_{max}$}
 \STATE $\tilde{r}_j=r(s_j)-\lambda (c_t+d(s_j))/\gamma^j$
 \ELSE
 \STATE $\tilde{r}_j=r(s_j)$
 %\STATE Set objective
 \ENDIF
 \STATE \textbf{\em \{"mask" indicates if the episode terminates\}}
 \STATE $y_j=\tilde{r}_j+\gamma*max_{a'} \hat{Q}((s_{j+1},c_{j+1}),a';\theta^-)*mask_{j+1}$.
 \STATE Update $\theta$ using $l=(y_j-Q((s_j,c_j),a_j;\theta))^2$.
 \STATE Every $C$ steps reset $\hat{Q}=Q$.
 \ENDFOR
 \STATE Store the final cost $c_{final}$ into $D_c$
 \IF{k\%M==0}
 \IF{$\max_i D_c(i)<c_{max}$ and $0.95\lambda>\Lambda$}
  \STATE $\lambda\leftarrow0.95\lambda$
 \ENDIF
 \STATE Empty replay buffer $D_c$
 \ENDIF
 \ENDFOR
 \end{algorithmic}
 \end{algorithm}

 \section{Safe SAC Pseudocode}
Algorithm~\ref{alg:SAC} provides the pseudocode for the Safe SAC algorithm.

\begin{algorithm}[htbp]
 \caption{Safe SAC with Extended State Space}
 \label{alg:SAC}
 \begin{algorithmic}[1] %[1] enables line numbers
 \STATE \textbf{Initialize:} policy network $\pi$ with weight $\theta$, constraint penalty $\lambda$, lower bound for constraint penalty $\Lambda$, replay buffer $D_c$ with capacity $M$.
 \STATE \textbf{Value Function:} $Q^1,Q^2$ with weights $\phi_1,\phi_2$, target Q value functions $Q^{targ,1},Q^{targ,2}$ with weights $\phi^{targ}_1=\phi_1,\phi^{targ}_2=\phi_2$. 
 \STATE \textbf{Cost Function:} $Q^1_D,Q^2_D$ with weights $\theta_{1,D},\theta_{2,D}$, target cost functions $Q^{targ,1}_D,Q^{targ,2}_D$ with weights $\theta^{targ}_{1,D}=\theta_{1,D},\theta^{targ}_{2,D}=\phi_{2,D}$.  
 \FOR{episode $k=1,2,...,N$}
 \STATE Get initial state-cost pair $(s_0, c_0=0)$; $t\leftarrow 1$
 \WHILE{$t\leq T$}
 \STATE $t_{start}\leftarrow t$
 \WHILE{$t\leq t_{start}+n$ or $t==T$}
 \STATE Select action $a_t$ using Equation \ref{equ:constraint}.
 \STATE Execute $a_t$, observe $(s_{t+1}, c_{t+1})$ and  $r_t$.
 \STATE $t\leftarrow t+1$
 \ENDWHILE
 
 \STATE \textbf{\{Calculate \textbf{targets} for each network:\}}
  
 \STATE $\tilde{r}_t\leftarrow$ \textbf{if} $c_{t}>c_{max}$ \textbf{then} $r_t-\lambda d_t/\gamma^t$ \textbf{elif} $c_{t+1}>c_{max}$ \textbf{then}
 $r_t-\lambda (c_t+d_t)/\gamma^t$ \textbf{else} $r_t$
 \STATE $R\leftarrow$ \textbf{if} $t==T$ \textbf{then} $0$ \textbf{else} {$\tilde{r}_t +\gamma\min_{i=1,2} Q^{targ,i}((s_{t+1},c_{t+1}),\tilde{a} ')-\alpha \log\pi_\theta(\tilde{a}')$, $\tilde{a}'\sim \pi_\theta((s_{t+1},c_{t+1}))$}
 \STATE $R_D\leftarrow$ \textbf{if} $t==T$ \textbf{then} $0$ \textbf{else} $\max_{i=1,2} Q^{targ,i}_D((s_{t+1},c_{t+1}),a_{t+1};\theta_D)$

 \STATE \textbf{\{Update networks\}}
 \FOR{$i\in\{t-1,...,t_{start}\}$}
 \STATE $R\leftarrow r_i+\alpha R$, $R_D\leftarrow d_i+\alpha R_D$
 %\STATE Accumulate the gradient of $\phi^j,\theta_D,\theta$:
%\STATE $\mathrm{d}\phi^j\leftarrow\mathrm{d}\phi^j+\partial(R-Q^j((s_i,c_i), a_i;\phi^j))^2/\partial\phi^j$, for $j=1, 2$
%\STATE $\mathrm{d}\theta_{j,D}\leftarrow\mathrm{d}\theta_{j,D}+\partial(R_D-Q^j_D((s_i,c_i),a_i;\theta_{j,D}))^2/\partial\theta_{j,D}$, for $j=1, 2$
%\IF{the policy is safe}
%\STATE $\mathrm{d}\theta\leftarrow\mathrm{d}\theta+\nabla_\theta \log \pi(a_i)(\min_{j=1,2} Q^{targ,j}((s_i,c_i),a_i)-\alpha \log\pi(a_i))$
%\ELSE
%\STATE $\mathrm{d}\theta\leftarrow\mathrm{d}\theta-\nabla_\theta \log \pi(a_i)R_D$
%\ENDIF
%\ENDFOR
\FOR{$j=1,2$}
\STATE $\mathrm{d}\phi^j\leftarrow\mathrm{d}\phi^j+\partial(R-Q^j)^2/\partial\phi^j$
 \STATE $\mathrm{d}\theta_{j,D}\leftarrow\mathrm{d}\theta_{j,D}+\partial(R_D-Q^j_D)^2/\partial\theta_{j,D}$
 \ENDFOR
 \IF{the policy is safe}
 \STATE $\mathrm{d}\theta\leftarrow\mathrm{d}\theta+w{\nabla_\theta \log \pi(a_i)(\min_{j=1,2} Q^{targ,j}-\alpha \log\pi(a_i))}$
 \ELSE
 \STATE $\mathrm{d}\theta\leftarrow\mathrm{d}\theta-\nabla_\theta \log \pi(a_i)R_D$
 \ENDIF
 \ENDFOR

 \STATE \textbf{\{Update target networks\}}
  \STATE $\phi^{targ}_1\leftarrow\rho\phi^{targ}_1+(1-\rho)\phi_1$
  \STATE $\phi^{targ}_2\leftarrow\rho\phi^{targ}_2+(1-\rho)\phi_2$
  \STATE $\theta^{targ}_{1,D}\leftarrow\rho\theta^{targ}_{1,D}+(1-\rho)\theta_{1,D}$
  \STATE $\theta^{targ}_{2,D}\leftarrow\rho\theta^{targ}_{2,D}+(1-\rho)\theta_{2,D}$
 \ENDWHILE
 \STATE \textbf{\{Dynamically modify reward penalty\}}
 \STATE Store the final cost $c_{final}$ into $D_c$
 \IF{k\%M==0}
 \IF{$\max_i D_c(i)<c_{max}$ and $0.95\lambda>\Lambda$}
  \STATE $\lambda\leftarrow0.95\lambda$
 \ENDIF
 \STATE Empty replay buffer $D_c$
 \ENDIF
 \ENDFOR
 \end{algorithmic}
 \end{algorithm}

\section{Missing Proofs}
\subsection{Proof of Theorem \ref{th:UMDP-WCMDP} }
\textbf{Theorem \ref{th:UMDP-WCMDP}. } \textit{If we set $\lambda = \infty$, then if $\pi^*$ solves  \eqref{equ:umdp}, it also solves the following worst-case constrained MDP problem
\begin{equation}
\begin{aligned}
    &\max_\pi \mathbb{E}\left[\sum_{t=0}^T {\gamma^t}r(s_t,a_t)|s_0,\pi\right]\\
    s.t&.\quad \sum_{s_t \in \tau } d(s_t)\leq c_{max},~\forall \tau\sim \pi.
\end{aligned}\nonumber
\end{equation}
As a result, $\pi^*$ is feasible to the risk-neutral CMDP \eqref{equ:cmdp}.  }
\begin{proof}
We first see that there is a unique mapping between a trajectory $\tau = \{s_0,\ldots,s_T\}$ from the original MDP to a trajectory of the extended MDP $\tau' = \{(s_0,c_0), (s_1,c_1)\ldots,(s_T,c_T)\}$ with 
 $c_0=0$ and $c_t = \sum_{i=0}^{t-1} d(s_t)$. Under the reward penalties,  we can write the objective of the extended MDP as 
 \begin{align}
     \mathbb{E}&\left[\sum_{t=0}^T {\gamma^t}r(a_t|s_t,c_t)|s_0,\pi\right] = \sum_{\tau'  = \{(s_t,c_t)\} \sim \pi } P_\pi (\tau') \left(\sum_{t} \gamma^t\widetilde{r}(a_t|s_t,c_t)\right)\nonumber \\
     &=  \sum_{\substack{\tau  = \{s_0,s_1,...\} \sim \pi \\ D(\tau)\leq c_{max}} } P_\pi(\tau) \left(\sum_t \gamma^t r(s_t,a_t)\right) + \sum_{\substack{\tau  = \{s_0,s_1,...\} \sim \pi \\ D(\tau)> c_{max}} } P_\pi(\tau) \left(\sum_t \gamma^t r(s_t,a_t) -\lambda \sum_t d(s_t)\right) \nonumber \\ 
     &=\bbE_{\pi}\left[\sum_{t}\gamma^t r(s_t,a_t)\right] - \lambda\sum_{\substack{\tau \sim \pi \\ D(\tau) > c_{\max}}} P_\pi (\tau) D(\tau)
 \end{align}
As a result, we can rewrite the MDP problem \eqref{equ:umdp} as 
\begin{equation}
 \label{eq:delta-equiv}   
\max_{\pi} \left\{\bbE_{\pi}\left[\sum_{t}\gamma^t r(s_t,a_t)\right] - \lambda\sum_{\substack{\tau \sim \pi \\ D(\tau) > c_{\max}}} P_\pi (\tau) D(\tau)\right\}
\end{equation}
 So, if we set $\lambda = \infty$, to maximize the expected reward, we need to seek a policy that assigns zero probabilities for all the trajectories $\tau$ such that  $D(\tau)>c_{max}$. Let $\Pi$ be the set of policies satisfying that condition (and assume that $\Pi$ is not empty), i.e.,  for any policy $\pi\in \Pi$ and any trajectory $\tau$ such that $D(\tau)>c_{\max}$, $P_{\pi}(\tau) = 0$. This implies that when $\lambda = \infty$, \eqref{eq:delta-equiv} is equivalent to
 \[
 \bbE_{\pi \in \Pi}\left[\sum_{t}\gamma^t r(s_t,a_t)\right]
 \]
 which is also the worst-case CMDP problem. 
\end{proof}

% \subsection{Proof of Proposition \ref{prop:non-statio}}
% \textit{There would be no stationary and history-independent policies being feasible to the worst-case CMDP \eqref{equ:rmdp}.  }
% \begin{proof}
%  It can be seen that \eqref{equ:rmdp} would require any feasible policy to give zero probabilities for trajectories $\tau$ such that $D(\tau)>c_{max}$. Thus, intuitively, at each state $s_t$,  
% \end{proof}

%%%%%%%%%%%%%%%%%%%%%%%

\subsection{Proof of Lemma \ref{lm:lm1}}
\textbf{Lemma \ref{lm:lm1}. }\textit{Let $\phi^* = c_{max} -\max_{\pi}\left\{\tbbE_{\pi}[D(\tau)|~ D(\tau) \leq  c_{max}]\right\}$. 
Given  any policy $\overline{\pi}$,  if $\widetilde{\bbE}_{\overline{\pi}}[D(\tau)|~ D(\tau) > c_{max}]\leq \phi^*$, then
$\bbE_{\overline{\pi}}[D(\tau)]\leq c_{max}$.}
\begin{proof}
For a policy $\overline{\pi}$ satisfying $\widetilde{\bbE}_{\overline{\pi}}[D(\tau)|~ D(\tau) > c_{max}]\leq \phi^*$, we have
\begin{align}
 \sum_{\tau|~ D(\tau) > c_{max}} P_{\overline{\pi}}(\tau) D(\tau) \leq c_{max} - \max_{\pi}\left\{\tbbE_{\pi}[D(\tau)|~ D(\tau) \leq  c_{max}]\right\} \nonumber,
\end{align}
which is equivalent to
\[
 \sum_{\tau|~ D(\tau) > c_{max}} P_{\overline{\pi}}(\tau) D(\tau) + \max_{\pi}\left\{\tbbE_{\pi}[D(\tau)|~ D(\tau) \leq  c_{max}]\right\} \leq c_{max} 
\]
implying
\[
c_{max} \geq  \sum_{\tau|~ D(\tau) > c_{max}} P_{\overline{\pi}}(\tau) D(\tau)  +  \sum_{\tau|~ D(\tau) \leq  c_{max}} P_{\overline{\pi}}(\tau) D(\tau) = \bbE_{\overline{\pi}}[D(\tau)] 
\]
which is  the  desired inequality.
\end{proof}

%%%%%%%%%%%%%%%%%%%%%%%

\subsection{Proof of Lemma \ref{lm:lm2}}
\textbf{Lemma \ref{lm:lm2}. }\textit{Given $\lambda>0$, let $\pi^*$ be an optimal solution to \eqref{equ:umdp}. We have
\[
 \tbbE_{\pi^*}\left[D(\tau)|~ D(\tau)>c_{max}\right] \leq \frac{\Psi^*-\overline{\Psi}}{\lambda}.
\]}
\begin{proof}
We first note that, from \eqref{eq:delta-equiv}, we can write
\[
\pi^* = \text{argmax}_{\pi} \left\{\bbE_{\pi}\left[\sum_{t}\gamma^t r(s_t,a_t)\right] - \lambda\sum_{\substack{\tau \\ D(\tau) \geq c_{\max}}} P_\pi (\tau) D(\tau)\right\}
\]
Let $\overline{\pi}$ be an optimal policy to the worst-case CMDP \eqref{equ:rmdp}. Since $\overline{\pi}$ is also feasible to the extended MDP \eqref{equ:umdp}, we have  
\begin{align}
    \bbE_{\pi^*}\left[\sum_{t}\gamma^t r(s_t,a_t)\right] - \lambda\sum_{\substack{\tau \\ D(\tau) \geq c_{\max}}} P_{\pi^*} (\tau) D(\tau) \geq \bbE_{\overline{\pi}}\left[\sum_{t}\gamma^t r(s_t,a_t)\right] = \overline{\Psi} \label{eq:proof-lm2-eq1}
\end{align}
Moreover, since $\Psi^*$ is the optimal value of the original unconstrained problem $\Psi^*  = \max_\pi \mathbb{E}\left[\sum_{t=0}^T {\gamma^t}r(s_t,a_t)|s_0,\pi\right]$, we should have
\begin{equation}
\label{eq:proof-lm2-eq2}
\Psi^* \geq  \bbE_{\pi^*}\left[\sum_{t}\gamma^t r(s_t,a_t)\right]     
\end{equation}
Combining \eqref{eq:proof-lm2-eq1} and \eqref{eq:proof-lm2-eq2} gives
\[
\Psi^* -  \lambda\sum_{\substack{\tau \\ D(\tau) \geq c_{\max}}} P_{\pi^*} (\tau) D(\tau) \geq \overline{\Psi},
\]
implying
\[
\sum_{\substack{\tau|~D(\tau) \geq c_{\max}}} P_{\pi^*} (\tau) D(\tau) \leq \frac{\Psi^* - \overline{\Psi}}{\lambda}, 
\]
which is the desired inequality. 
\end{proof}
\subsection{Proof of Theorem \ref{th:RN-CMDP}}
\textbf{Theorem \ref{th:RN-CMDP}. }
\textit{For any $\lambda \geq \frac{\Psi^*-\overline{\Psi}}{\phi^*}$
a solution to  \eqref{equ:umdp} is always feasible to the risk-neutral CMDP \eqref{equ:cmdp}.}
\begin{proof}
    The theorem is a direct result from Lemmas \ref{lm:lm1} and \ref{lm:lm2}. That is, by selecting $\lambda \geq \frac{\Psi^*-\overline{\Psi}}{\phi^*}$, from Lemm \ref{lm:lm2} we can guarantee that  
   \begin{align}
       \sum_{\substack{\tau|~D(\tau) \geq c_{\max}}} P_{\pi^*} (\tau) D(\tau) \leq \frac{\Psi^* - \overline{\Psi}}{\lambda} \leq \phi^*, \label{eq:th2-eq1} 
   \end{align} 
   where $\pi^*$ is an optimal policy to \eqref{equ:umdp}. From Lemma \ref{lm:lm2}, \eqref{eq:th2-eq1} also implies that $\pi^*$ is also feasible to the risk-neutral CMDP \eqref{equ:cmdp}, as desired.  
\end{proof}
%%%%%%%%%%%%%%%%%%%%%%%%%%

\subsection{Proof of Theorem \ref{th:VAR-CMDP}}
\textbf{Theorem \ref{th:VAR-CMDP}. }
\textit{For any $\lambda \geq {(\Psi^*-\overline{\Psi})/}{(\alpha c_{max})
}$,
a solution to  \eqref{equ:umdp} is always feasible to  \eqref{equ:var-mdp}.}
\begin{proof}
We use Lemma \ref{lm:lm2} to see that if $\pi^*$ is a solution to \eqref{equ:umdp}, then it satisfies
\begin{equation}
\label{th4-proof-eq1}
\tbbE_{\tau\sim {\pi^*}}\Big[ D(\tau)|~D(\tau)> c_{\max} \Big] \leq \frac{\Psi^*-\overline{\Psi}}{\lambda}.
\end{equation}
On the other hand, we have
\begin{equation}
\label{th4-proof-eq2}
\begin{aligned}
    \tbbE_{\tau\sim {\pi^*}}\Big[ D(\tau)|~D(\tau)> c_{\max} \Big] &= \sum_{\tau| D(\tau)>c_{max}} P_{\pi^*}(\tau) D(\tau) \\
    &> c_{max} \sum_{\tau| D(\tau)>c_{max}} P_{\pi^*}(\tau)\\ 
    &= c_{\max} P_{\pi^*}(D(\tau)>c_{max}))
\end{aligned}
\end{equation}
Thus, if we select $\lambda \geq (\Psi^*-\overline{\Psi})/(\alpha c_{max})$, we will have the following  chain of inequalities.
\begin{align}   
\alpha &\geq \frac{\Psi^* - \overline{\Psi}}{\lambda c_{max}}\nonumber \\
&\stackrel{(a)}{\geq} \frac{1}{c_{max}} \tbbE_{\tau\sim {\pi^*}}\Big[ D(\tau)|~D(\tau)> c_{\max} \Big]\nonumber \\
&\stackrel{(b)}{\geq} P_{\pi^*}(D(\tau)>c_{max}). \nonumber
\end{align}
where $(a)$ is due to \eqref{th4-proof-eq1} and $(b)$ is due to \eqref{th4-proof-eq2}. This implies that $\pi^*$ is feasible to the chance-constrained MDP \eqref{equ:var-mdp}. We complete the proof. 
\end{proof}

%%%%%%%%%%%%%%%%%%%%%%%
\subsection{Proof of Theorem \ref{prop:var}}
\textbf{Theorem \ref{prop:var}.} \textit{
If we define the reward penalties as 
\begin{equation*}\label{eq:new-penalty1}   
\begin{cases}
\widetilde{r}(a_t|(s_t,c_t)) = r(a_t|s_t) ~\text{ if }c_t+d(s_t)\leq c_{max} \\
\widetilde{r}(a_t|(s_t,c_t)) = r(a_t|s_t) - {\lambda (t+1)/\gamma^t}\\
~~~~~~~{\text{ if } c_t \leq c_{max} \text{ and }c_t+d(s_t)> c_{max}} \\
\widetilde{r}(a_t|(s_t,c_t)) = {r(a_t|s_t) - \lambda/\gamma^t }~\text{ if } c_t > c_{max} 
\end{cases}
\end{equation*}
then if $\pi^*$ is an optimal solution to \eqref{equ:umdp}, then there is $\alpha^\lambda \in [0;\frac{\Psi^*-\overline{\Psi}}{\lambda T}]$ ($\alpha$ is dependent of $\lambda$) such that $\pi^*$ is also optimal to \eqref{equ:var-mdp}. Moreover $\lim_{\lambda\rightarrow \infty}\alpha^\lambda = 0$.}

\begin{proof}
    Under the reward setting, we can write the objective  of \eqref{equ:umdp} as 
 \begin{align}
     \mathbb{E}_{\pi}&\left[\sum_{t=0}^T {\gamma^t}r(a_t|s_t,c_t)|s_0\right] = \sum_{\tau'  = \{(s_t,c_t)\} \sim \pi } P_\pi (\tau') \left(\sum_{t} \gamma^t\widetilde{r}(a_t|s_t,c_t)\right)\nonumber \\
     &=  \sum_{\substack{\tau  = \{s_0,s_1,...\} \sim \pi \\ D(\tau)\leq c_{max}} } P_\pi(\tau) \left(\sum_t \gamma^t r(s_t,a_t)\right) + \sum_{\substack{\tau  = \{s_0,s_1,...\} \sim \pi \\ D(\tau)> c_{max}} } P_\pi(\tau) \left(\sum_t \gamma^t r(s_t,a_t) -\lambda T \right) \nonumber \\ 
     &=\bbE_{\pi}\left[\sum_{t}\gamma^t r(s_t,a_t)\right] - \lambda T P_\pi (D(\tau)>c_{max}). \label{eq:th38-eq1}  
 \end{align}
We now show that if $\pi^*$ is an optimal 
policy to \eqref{equ:umdp}, then it is also optimal for \eqref{equ:var-mdp} with 
where $\alpha^\lambda = P_{\pi^*} (D(\tau)>c_{max})$.  
By contradiction, let us assume that it is not the case. Let $\overline{\pi}$ be optimal for \eqref{equ:var-mdp}. We first see that $\pi^*$ is feasible to \eqref{equ:var-mdp}, thus 
\begin{equation}
\label{eq:th38-eq3}        \mathbb{E}_{\pi^*}\left[\sum_{t=0}^T {\gamma^t}r(s_t,a_t)\right] < \mathbb{E}_{\overline{\pi}}\left[\sum_{t=0}^T {\gamma^t}r(s_t,a_t)\right]. 
\end{equation}
Moreover, since $\overline{\pi}$ is feasible to \eqref{equ:var-mdp}, we have:
\begin{align}
    P_{\overline{\pi}}(D(\tau)>c_{max}) \leq P_{{\pi^*}}(D(\tau)>c_{max}).\label{eq:th38-eq4}
\end{align}
Combine \eqref{eq:th38-eq3} and \eqref{eq:th38-eq4} and  \eqref{eq:th38-eq1}, 
it can be seen that $\pi^*$ is not an optimal policy to \eqref{equ:umdp}, which is contrary to our initial assumption.   So, $\pi^*$ is an optimal policy for the \eqref{equ:var-mdp}.
 We now prove that $\lim_{\lambda\rightarrow \infty}\alpha^\lambda = 0$.
 To this end, we first see that if $\widehat{\pi}$ is an optimal solution to the worst-case CMDP \eqref{equ:rmdp}, then $P_{\widehat{\pi}}(D(\tau)>c_{max}) = 0$. Thus, we have the following chain of inequalities
 \begin{align}
 \Psi^* -  \lambda T \alpha^\lambda&\geq \bbE_{\pi^*}\left[\sum_{t}\gamma^t r(s_t,a_t)\right] - \lambda T P_{\pi^*}(D(\tau)>c_{max}) \nonumber\\   
 &\geq \bbE_{\widehat{\pi}}\left[\sum_{t}\gamma^t r(s_t,a_t)\right] - \lambda T P_{ \widehat{\pi}}(D(\tau)>c_{max}) \nonumber \\
 &= \bbE_{\widehat{\pi}}\left[\sum_{t}\gamma^t r(s_t,a_t)\right] = \overline{\Psi}\nonumber
 \end{align}
Thus 
\[
\alpha^\lambda\leq  \frac{\Psi^* - \overline{\Psi}}{\lambda T}.
\]
implying $\lim_{\lambda \rightarrow \infty } \alpha^\lambda = 0$.     
\end{proof}
%%%%%%%%%%%%%%%%%%%%%%%%
\subsection{Proof of Theorem \ref{th:CVAR}}
\textbf{Theorem \ref{th:CVAR}. }\textit{
If we define the reward penalties as \[
\begin{cases}
\widetilde{r}(a_t|(s_t,c_t)) = r(a_t|s_t) ~\text{ if } c_t+d(s_t)\leq c_{max} \\
\widetilde{r}(a_t|(s_t,c_t)) = r(a_t|s_t) - {\lambda (c_t+d(s_t)-c_{max})/\gamma^t}\\
~~~~~~~{\text{ if } c_t \leq c_{max} \text{ and }c_t+d(s_t)> c_{max}} \\
\widetilde{r}(a_t|(s_t,c_t)) = {r(a_t|s_t) - \lambda d(s_t)/\gamma^t}~\text{ if } c_t > c_{max} 
\end{cases}
\]
then for any $\lambda>0$, there is $\beta^\lambda\in [0;\frac{\Psi^*-\overline{\Psi}}{\lambda}]$ ($\beta^\lambda$ is dependent of $\lambda$) such that 
any optimal solution to the extended CMDP \eqref{equ:umdp} is also optimal to the following risk-averse CMDP
\begin{equation}\tag{\sf\small CVaR-CMDP}
\begin{aligned}
%\label{equ:cvar-mdp}
   &\max_\pi \mathbb{E}\left[\sum_{t=0}^T {\gamma^t}r(s_t,a_t)|s_0,\pi\right]\\
    s.t&.\quad  \bbE_{\tau\sim \pi}\Big[ (D(\tau)-c_{max})^+ \Big] \leq \beta^\lambda.
\end{aligned}
\end{equation}
 Moreover, $\lim_{\lambda\rightarrow \infty} \beta^\lambda = 0$. }

\begin{proof}
We first see that, under the reward penalties defined above, the objective of \eqref{equ:umdp} becomes
 \begin{align}
     \mathbb{E}_{\pi}&\left[\sum_{t=0}^T {\gamma^t}r(a_t|s_t,c_t)|s_0\right] = \sum_{\tau'  = \{(s_t,c_t)\} \sim \pi } P_\pi (\tau') \left(\sum_{t} \gamma^t\widetilde{r}(a_t|s_t,c_t)\right)\nonumber \\
     &=  \sum_{\substack{\tau  = \{s_0,s_1,...\} \sim \pi \\ D(\tau)\leq c_{max}} } P_\pi(\tau) \left(\sum_t \gamma^t r(s_t,a_t)\right)\nonumber\\&\qquad + \sum_{\substack{\tau  = \{s_0,s_1,...\} \sim \pi \\ D(\tau)> c_{max}} } P_\pi(\tau) \left(\sum_t \gamma^t r(s_t,a_t) -\lambda \left(\sum_t d(s_t) - c_{max}\right)\right) \nonumber \\ 
     &=\bbE_{\pi}\left[\sum_{t}\gamma^t r(s_t,a_t)\right] - \lambda\sum_{\substack{\tau \sim \pi \\ D(\tau) > c_{\max}}} P_\pi (\tau) (D(\tau)-c_{max})\nonumber \\ &= \bbE_{\pi}\left[\sum_{t}\gamma^t r(s_t,a_t)\right] - \lambda\bbE_{\tau\sim \pi}\left[(D(\tau)-c_{max})^+\right] \label{eq:th39-eq1}  
 \end{align}
We now show that if $\pi^*$ is an optimal 
policy to \eqref{equ:umdp}, then it is also optimal for \eqref{equ:cvar-mdp} with 
where $\beta^\lambda = \bbE_{\tau\sim \pi^*}\Big[ (D(\tau) - c_{\max})^+ \Big]$.  
By contradiction, let us assume that $\pi^*$ is not optimal for \eqref{equ:cvar-mdp}. We then let $\overline{\pi}$ be optimal for \eqref{equ:cvar-mdp}. We first see that $\pi^*$ is feasible to \eqref{equ:cvar-mdp}, thus 
\begin{equation}
\label{eq:th3-eq3}        \mathbb{E}_{\pi^*}\left[\sum_{t=0}^T {\gamma^t}r(s_t,a_t)\right] < \mathbb{E}_{\overline{\pi}}\left[\sum_{t=0}^T {\gamma^t}r(s_t,a_t)\right] 
\end{equation}
Moreover, since $\overline{\pi}$ is feasible to \eqref{equ:cvar-mdp}, we have:
\begin{align}
    \bbE_{\tau\sim \overline{\pi}}\Big[ (D(\tau) - c_{\max})^+ \Big] \leq \beta^\lambda = \bbE_{\tau\sim {\pi^*}}\Big[ (D(\tau) - c_{\max})^+ \Big] \label{eq:th3-eq4}
\end{align}
Combine \eqref{eq:th3-eq3} and \eqref{eq:th3-eq4} we get 
\begin{equation}
 \label{eq:th39-proof-eq1}   \mathbb{E}_{\pi^*}\left[\sum_{t=0}^T {\gamma^t}r(s_t,a_t)\right] -\lambda \bbE_{\tau\sim \pi^*}\Big[ (D(\tau) - c_{\max})^+ \Big] < \mathbb{E}_{\overline{\pi}}\left[\sum_{t=0}^T {\gamma^t}r(s_t,a_t)\right] -\lambda \tbbE_{\tau\sim \overline{\pi}}\Big[ (D(\tau) - c_{\max})^+ \Big] 
\end{equation}
 Using \eqref{eq:th39-eq1}, \eqref{eq:th39-proof-eq1} implies that $\widehat{\pi}$ yields a strictly better objective value to the extended MDP, as compared to $\pi^*$, which is contrary to the assumption that $\pi^*$ is optimal for \eqref{equ:umdp}.  So, $\pi^*$ should be an optimal policy for the \eqref{equ:cvar-mdp}.
 We now prove that $\lim_{\lambda\rightarrow \infty}\beta^\lambda = 0$.
 To this end, we first see that if $\widehat{\pi}$ is an optimal solution to the worst-case CMDP \eqref{equ:rmdp}, then $\tbbE_{\tau\sim \widehat{\pi}}\Big[ (D(\tau) - c_{\max})^+ \Big] = 0$. Thus, we have the following chain of inequalities:
 \begin{align}
 \Psi^* -  \lambda\beta^\lambda&\geq \bbE_{\pi^*}\left[\sum_{t}\gamma^t r(s_t,a_t)\right] - \lambda\bbE_{\tau\sim \pi^*}\left[(D(\tau)-c_{max})^+\right] \nonumber\\   
 &\geq \bbE_{\widehat{\pi}}\left[\sum_{t}\gamma^t r(s_t,a_t)\right] - \lambda\bbE_{\tau\sim \widehat{\pi}}\left[(D(\tau)-c_{max})^+\right] \nonumber \\
 &= \bbE_{\widehat{\pi}}\left[\sum_{t}\gamma^t r(s_t,a_t)\right]  
 \end{align}
 We recall that $\bbE_{\widehat{\pi}}\left[\sum_{t}\gamma^t r(s_t,a_t)\right]  = \overline{\Psi}$ (i.e., objective 
 value of the worst-case CMDP), thus,  
\[
\beta^\lambda\leq  \frac{\Psi^* - \overline{\Psi}}{\lambda},
\]
implying $\lim_{\lambda \rightarrow \infty } \beta^\lambda = 0$ as desired.
\end{proof}

\section{Multi-constrained MDP}
We now discuss extension to CMDP with multiple cost constraints (e.g., limited fuel and bounded risk) and show how the above theoretical results can be extended to the multi-constrained variants. A multi-constrained risk-neural CMDP can be formulated as 
\begin{equation}\tag{\sf\small MRN-CMDP}
\begin{aligned}
\label{equ:mcmdp}
    &\max_\pi \mathbb{E}\left[\sum_{t=0}^T {\gamma^t}r(s_t,a_t)|s_0,\pi\right]\\
    s.t&.\quad \mathbb{E}\left[\sum_{t=0}^T  d^k(s_t)|s_0,\pi\right]\leq c^k_{max}, \forall k \in [K]
\end{aligned}
\end{equation}
where $[K]$ denotes the set $\{1,\ldots,K\}$. Similar to the single constraint case, to include cost functions in the rewards, we extend the state space to keep track of the accumulated costs as $\widetilde{\cS} = \{(s,c_1,\ldots, c_K)|~ s\in  \cS, c_k\in \bbR, ~\forall k\in [K]\}$ and define new transitions probabilities as 
\[
\begin{cases}
\widetilde{p}(s_{t+1},\bc^K_{t+1}|(s_t,\bc^K_{t}), a_t) = p(s_{t+1}|s_t,a_t)\\
\qquad\qquad\text{ if }c^k_{t+1} = c^k_t + d^k(s_t) \\
\widetilde{p}(s_{t+1},\bc^K_{t+1}|(s_t,\bc^K_t), a_t) = 0\text{ otherwise}
\end{cases}
\]
where $\bc^K_t = (c^1_t,\ldots,c^K)$ for notational simplicity.  The new rewards are also updated in such a way that every trajectory violating the constraints will be penalized. 
\[
\widetilde{r}(a_t|(s_t,\bc^K_t)) = r(a_t|s_t) - \sum_{k\in [K]} \lambda_k\delta^k(c_t),
\]
where $\delta^k(c_t)$, $\forall k\in [K]$,  are defined as follows.
\[
\delta^k(c_t) = 
\begin{cases}
    0 \text{ if },~ c^k_t +d^k(s_t)\leq c^k_{max}\\\
    (c^k_t+d^k(s_t))/\gamma^t\text{ if } c^k_t\leq c^k_{\max},\\
    \qquad\qquad c^k_t+d^k(s_t)\geq c^k_{max}\\
    d^k(s_t)/\gamma^t\text{ if } c^k_t> c^k_{\max}.
\end{cases}
\]
Here, we allow  
penalty parameters $\lambda_k$ to be different over constraints. We formulate the extended unconstrained MDP as:
\begin{equation}
\begin{aligned}
\label{equ:umdp-multi}
    &\max_\pi \left\{\mathbb{E}\left[\sum_{t=0}^T {\gamma^t}\widetilde{r}(a_t|(s_t,\bc^K_t))\Big|(s_0,\bc^K_0),\pi\right]\right\}.
\end{aligned}
\end{equation}
Similar to the single-constrained case, the reward penalties allow us to write the objective function of the extended MDP as
\begin{equation}
 \label{eq:new-obj-multiple}
 \bbE_{\pi}\left[\sum_{t}\gamma^t r(s_t,a_t)\right] - \sum_{k\in [K]} \lambda_k\widetilde{\bbE}_{\pi}\left[D^k(\tau)|~ D^k(\tau)> c^k_{\max}\right]
\end{equation}
where $D^k(\tau)$ is the accumulated cost $d^k(s_t)$ on trajectory $\tau$, i.e., $D^k(\tau) = \sum_{s_t\in\tau}d^k(s_t)$. As a result,  when $\lambda_k$ grows, the extended MDP will discount the second term of  \eqref{eq:new-obj-multiple}, thus yielding policies that satisfy or even solve risk-neural or risk-averse CMDP problems. Specifically, the following results can be proved:
\begin{itemize}
    \item When $\lambda_k = \infty$, $\forall k\in [K]$, then \eqref{equ:umdp-multi} is equivalent to worst-case CMDP (i.e., all the trajectories generated by the policy will satisfy all the cost constraints).
    \item There are lower bounds for $\lambda_k$ from which any solution to \eqref{equ:umdp-multi} will be feasible to risk-neural and VaR CMDP with multiple
    constraints.
    \item For any $\lambda_k>0$, under different reward penalty settings, \eqref{equ:umdp-multi} is equivalent to a multi-constrained CVaR CMDP or equivalent to a multi-constrained VaR CMDP.  
\end{itemize}
%\subsection{Proofs}
We present, in the following, a series of theoretical results for the above claims. 
%\begin{itemize}
%    \item If $\lambda_k = \infty$ for all $k\in[K]$, then \eqref{equ:umdp-multi} is equivalent to a worst-case CMDP.
%    \item There is a lower bound for each $\lambda_k$ such that any optimal policy to \eqref{equ:umdp-multi} will always be feasible to a given risk-neutral or chance-constrained MDP.
%    \item By employing different reward penalty settings, \eqref{equ:umdp-multi} is equivalent to a VaR or CVaR CMDP. 
%\end{itemize}
Since the proofs are quite similar to those in the single-constrained case,  we keep them brief. 
\begin{proposition}\label{prop:multi-WC}
    If we set $\lambda_k = \infty$ for all $k\in [K]$, then the extended MDP is equivalent to the following worst-case CMDP
\begin{equation}
\begin{aligned}
\label{equ:rmdp-multi}
    &\max_\pi \mathbb{E}\left[\sum_{t=0}^T {\gamma^t}r(s_t,a_t)|s_0,\pi\right]\\
    s.t&.\quad \sum_{s_t \in \tau } d^k(s_t)\leq c^k_{max},~\forall \tau\sim \pi,~\forall k\in [K]
\end{aligned}
\end{equation}
\end{proposition}
\begin{proof}
    Similar to the proof of Theorem \ref{th:UMDP-WCMDP}, we write the objective of the extended MDP as  
    \begin{align}
     \mathbb{E}&\left[\sum_{t=0}^T {\gamma^t}r(a_t|s_t,\bc^K_t)|s_0,\pi\right] = \sum_{\tau'  = \{(s_t,\bc^K_t)\} \sim \pi } P_\pi (\tau') \left(\sum_{t} \gamma^t\widetilde{r}(a_t|s_t,\bc^K_t)\right)\nonumber \\
     &=  \sum_{\substack{\tau  = \{s_0,s_1,...\} \sim \pi \\ D(\tau)\leq c_{max}} } P_\pi(\tau) \left(\sum_t \gamma^t r(s_t,a_t)\right) \nonumber \\&\qquad + \sum_{\substack{\tau  = \{s_0,s_1,...\} \sim \pi \\ D(\tau)> c_{max}} } P_\pi(\tau) \left(\sum_t \gamma^t r(s_t,a_t) -\sum_{k\in [K]}\lambda_k \sum_t d^k(s_t)\right) \nonumber \\ 
     &=\bbE_{\pi}\left[\sum_{t}\gamma^t r(s_t,a_t)\right] - \sum_{k\in [K]}\lambda_k\sum_{\substack{\tau \sim \pi \\ D^k(\tau) \geq c^k_{\max}}} P_\pi (\tau) D^k(\tau).
 \end{align}
 So, if $\lambda_k = \infty$, then one needs to seek a policy that assigns zero probabilities to all the trajectories that violate the constraints, implying that the extended MDP would yield the same optimal policies as the worst-case CMDP \eqref{equ:rmdp-multi}. 
\end{proof}

\begin{proposition}
\label{prop:multi-RN}
Let $\pi^*$ and $\overline{\pi}$ be optimal policies to the extended MDP \eqref{equ:umdp-multi} and the worst-case MDP, and $\phi^k = c^k_{max} - \max_{\pi}\widetilde{\bbE}_{\pi}[D^k(\tau)|D^k(\tau)\leq c^k_{max}]$, $\forall k\in [K]$.
    If we choose $\lambda_k$ such that $\lambda_k>(\Psi^*-\overline{\Psi})/\phi^k$, then any optimal policy of \eqref{equ:umdp-multi} is feasible to the risk-neutral CMDP with multiple constraints.
\end{proposition}
\begin{proof}
    % Let $\pi^*$ and $\overline{\pi}$ be optimal policies to the extended MDP \eqref{equ:umdp-multi} and the worst-case MDP. 
    Since $\overline{\pi}$ is also feasible to \eqref{equ:umdp-multi}, we have:
    \begin{align}
    \bbE_{\pi^*}\left[\sum_{t}\gamma^t r(s_t,a_t)\right] - \sum_{k\in [K]}\lambda_k\sum_{\substack{\tau \\ D^k(\tau) > c^k_{\max}}} P_{\pi^*} (\tau) D^k(\tau) \geq \bbE_{\overline{\pi}}\left[\sum_{t}\gamma^t r(s_t,a_t)\right] = \overline{\Psi}. \label{eq:proof-lm2-eq3}
\end{align}
Moreover, since $\Psi^*$ is an optimal value of the original unconstrained MDP, we have $\Psi^* \geq \bbE_{\pi^*}\left[\sum_{t}\gamma^t r(s_t,a_t)\right]$, leading to
\begin{equation}
\label{eq:multi-prop1}    
\sum_{k\in [K]}\lambda_k\sum_{\substack{\tau \\ D^k(\tau) \geq c^k_{\max}}} P_{\pi^*} (\tau) D^k(\tau) \leq \Psi^*-\overline{\Psi}. 
\end{equation}
Moreover, from Lemma \ref{lm:lm1}, we know that if $\widetilde{ \bbE}_{\pi}[D^k(\tau)|~D^k(\tau)>c^k_{max}] \leq \phi^k$, then $\bbE_{\pi}[D^k(\tau)]<c^k_{max}$, where $\phi^k = c^k_{max} - \max_{\pi}\widetilde{\bbE}_{\pi}[D^k(\tau)|D^k(\tau)\leq c_{max}]$. Therefore, if we select $\lambda_k\geq (\Psi^*-\overline{\Psi})/\phi^k$, then from \eqref{eq:multi-prop1} we see that $\widetilde{ \bbE}_{\pi^*}[D^k(\tau)|~D^k(\tau)>c^k_{max}] \leq \phi^k$ for all $k\in [K]$, implying that $\pi^*$ satisfies all the constraints, as desired.
\end{proof}

\begin{proposition} \label{prop:multi-VaR}
Given any $\alpha_k\in (0,1)$, $k\in [K]$, if we choose $\lambda_k \geq {(\Psi^*-\overline{\Psi})/}{(\alpha_k c^k_{max})
}$, $\forall k\in [K]$, then
a solution $\pi^*$ to  \eqref{equ:umdp-multi} is always feasible to the following VaR (or chance-constrained) MDP.   
\begin{equation}
\begin{aligned}
\label{equ:var-mdp-multi0}
      &\max_\pi \mathbb{E}\left[\sum_{t=0}^T {\gamma^t}r(s_t,a_t)|s_0,\pi\right]\\
    s.t&.\quad  P_{\pi}\Big[ (D^k(\tau)>c^k_{\max}\Big] \leq \alpha_k,~\forall k\in [K]
\end{aligned}
\end{equation}
\end{proposition}
\begin{proof}
  From the proof of Proposition \ref{prop:multi-RN} above, we have the following inequalities
\begin{align}
\Psi^*-\overline{\Psi} &\geq \sum_{k\in [K]}\lambda_k\sum_{\substack{\tau \\ D^k(\tau) > c^k_{\max}}} P_{\pi^*} (\tau) D^k(\tau) \nonumber\\
&\geq \sum_{k\in [K]} \lambda_k c^k_{max} P_{\pi^*}(D^k(\tau)> c^k_{max}) \nonumber
\end{align}
So if we choose $\lambda_k \geq {(\Psi^*-\overline{\Psi})/}{(\alpha_k c^k_{max})
}$, $\forall k\in [K]$, we will have
\begin{align}
    \Psi^*-\overline{\Psi} \geq \frac{(\Psi^*-\overline{\Psi})}{(\alpha_k c^k_{max})
} c^k_{max} P_{\pi^*}(D^k(\tau)> c^k_{max}),~\forall k\in [K],\nonumber
\end{align}
implying $P_{\pi^*}(D^k(\tau)> c^k_{max})\leq \alpha_k$, as desired. 
\end{proof}

\begin{proposition}\label{prop:var-multi}
If we define the reward penalties as 
\[
\widetilde{r}(a_t|(s_t,\bc^K_t)) = r(a_t|s_t) - \sum_{k\in [K]} \lambda_k\delta^k(c_t),~\forall s_t,a_t,\bc^K_t,
\]
where $\delta^k(c_t)$, $\forall k\in [K]$,  are defined as follows:
\[
\delta^k(c_t) = 
\begin{cases}
    0 \text{ if }  c^k_t +d^k(s_t)\leq c^k_{max}\\\
    (T + 1)/\gamma^t\text{ if } c^k_t\leq c^k_{\max}, ~ c^k_t+d^k(s_t)> c^k_{max}\\
    1/\gamma^t\text{ if } c^k_t> c^k_{\max},
\end{cases}
\]
then if $\pi^*$ is an optimal solution to \eqref{equ:umdp}, there is $\alpha_k^{\bDelta}\in [0;\frac{\Psi^*-\overline{\Psi}}{T\lambda_k}]$ ($\alpha_k$ is dependent of $\bDelta$)\footnote{$\bDelta$ denotes the vector $(\lambda_1,\ldots,\lambda_K)$} such that $\pi^*$ is also optimal to the following VaR CMDP.
\begin{equation}
\begin{aligned}
\label{equ:var-mdp-multi}
      &\max_\pi \mathbb{E}\left[\sum_{t=0}^T {\gamma^t}r(s_t,a_t)|s_0,\pi\right]\\
    s.t&.\quad  P_{\pi}\Big[ (D(\tau)>c^k_{\max}\Big] \leq \alpha_k^{\bDelta}, ~ \forall k\in [K]. 
\end{aligned}
\end{equation}
Moreover $\lim_{\lambda_k\rightarrow \infty}\alpha_k^{\bDelta} = 0,~ \forall k\in [K]$.  
\end{proposition}
\begin{proof}
    Similar to the  proof of Theorem \ref{th:VAR-CMDP}, we can write the objective of the extended MDP as
\[
\bbE_{\pi}\left[\sum_{t}\gamma^t r(a_t|s_t,\bc^K_t)\right] 
 = \bbE_{\pi}\left[\sum_{t}\gamma^t r(s_t,a_t)\right] - \sum_{k\in [K]}\lambda_k T P_\pi (D^k(\tau)>c^k_{max})
\]  
Then, in a similar way, if we let $\alpha^{\bDelta}_k = T P_{\pi^*} (D^k(\tau)>c^k_{max})$, then $\pi^*$ should be an optimal policy to \eqref{equ:var-mdp-multi}. In addition, we can bound $\alpha_k$ by deriving the following inequalities. 
 \begin{align}
 \Psi^* -  \sum_{k\in [K]}\lambda_k T\alpha_k^{\bDelta}&\geq \bbE_{\pi^*}\left[\sum_{t}\gamma^t r(s_t,a_t)\right] - \sum_{k\in [K]}\lambda_k T P_{ \pi^*}(D^k(\tau)>c^k_{max}) \nonumber\\   
 &\geq \bbE_{\widehat{\pi}}\left[\sum_{t}\gamma^t r(s_t,a_t)\right] - \sum_{k\in [K]}\lambda_k T P_{\widehat{\pi}}(D^k(\tau)>c^k_{max})  \nonumber \\
 &= \bbE_{\widehat{\pi}}\left[\sum_{t}\gamma^t r(s_t,a_t)\right] = \overline{\Psi},
 \end{align}
 where $\widehat{\pi}$ and $\overline{\Psi}$ are optimal policy and optimal value of the worst-case CMDP \eqref{equ:rmdp-multi}. This implies
 \[
  \sum_{k\in [K]}\lambda_k T\alpha_k^{\bDelta} \leq \Psi^* -\overline{\Psi},
 \]
 which tells us that $\alpha_k^{\bDelta} \leq \frac{\Psi^*-\overline{\Psi}}{T\lambda_k}$, implying that $\lim_{\lambda_k \rightarrow \infty} \alpha_k^{\bDelta} = 0$. 
\end{proof}

%%%%%%%%%%%%%%%%%%%%%%%%%
\begin{proposition}
For any $\lambda_k>0$, $k\in [K]$,
if we define the reward penalties as 
\[
\widetilde{r}(a_t|(s_t,\bc^K_t)) = r(a_t|s_t) - \sum_{k\in [K]} \lambda_k\delta^k(c_t),~\forall s_t,a_t,\bc^K_t,
\]
where $\delta^k(c_t)$, $\forall k\in [K]$,  are defined as follows:
\[
\delta^k(c_t) = 
\begin{cases}
    0 \text{ if }  c^k_t +d^k(s_t)\leq c^k_{max}\\\
    (c^k_t + d^k_t- c^k_{max})/\gamma^t\text{ if } c^k_t\leq c^k_{\max}, ~ c^k_t+d^k(s_t)> c^k_{max}\\
    d^k_t/\gamma^t\text{ if } c^k_t> c^k_{\max},
\end{cases}
\]
then there are $\beta_k^{\bDelta}\in [0;\frac{\Psi^*-\overline{\Psi}}{\lambda_k}]$ ($\beta^{\bDelta}_k$ is dependent of $\bDelta$) such that 
any optimal solution $\pi^*$ to the extended CMDP \eqref{equ:umdp-multi} is also optimal to the following multi-constrained CVaR CMDP
\begin{equation}
\begin{aligned}
\label{equ:cvar-mdp-multi}
   &\max_\pi \mathbb{E}\left[\sum_{t=0}^T {\gamma^t}r(s_t,a_t)|s_0,\pi\right]\\
    s.t&.\quad  \bbE_{\tau\sim \pi}\Big[ (D(\tau)-c_{max})^+ \Big] \leq \beta_k^{\bDelta},~\forall k\in [K]
\end{aligned}
\end{equation}
 Moreover, $\lim_{\lambda_k\rightarrow \infty} \beta_k^{\bDelta} = 0$.     
\end{proposition}
\begin{proof}
   Under the reward setting, we first write the objective function of the extended MDP as
    \[    \bbE_{\pi}\left[\sum_{t}\gamma^t r(a_t|s_t,\bc^K_t)\right] 
 = \bbE_{\pi}\left[\sum_{t}\gamma^t r(s_t,a_t)\right] - \sum_{k\in [K]}\lambda_k \bbE_\pi \left[(D^k(\tau)-c^k_{max})^+\right]
    \]
 Following the same derivations as in the proof of Theorem \ref{th:CVAR},    
   we  can further show that, by contradiction,  $\pi^*$ is also optimal for the CVaR CMDP \eqref{equ:cvar-mdp-multi} with $\beta_k^{\bDelta} =  \bbE_{\pi^*}\Big[ (D^k(\tau)- c^k_{\max})^+ \Big]$. To prove  $\lim_{\lambda_k\rightarrow \infty} \beta_k^{\bDelta} = 0$, we derive similar inequalities as in the proof of Proposition \ref{prop:var-multi}, as follows:
   \begin{align}
 \Psi^* -  \sum_{k\in [K]}\lambda_k\beta_k^{\bDelta}&\geq \bbE_{\pi^*}\left[\sum_{t}\gamma^t r(s_t,a_t)\right] - \sum_{k\in [K]}\lambda_k \bbE_{\pi^*}\left[(D^k(\tau)-c^k_{max})^+\right] \nonumber\\   
 &\geq \bbE_{\widehat{\pi}}\left[\sum_{t}\gamma^t r(s_t,a_t)\right] - \sum_{k\in [K]}\lambda_k \bbE_{\widehat{\pi}}\left[(D^k(\tau)-c^k_{max})^+\right]  \nonumber \\
 &= \bbE_{\widehat{\pi}}\left[\sum_{t}\gamma^t r(s_t,a_t)\right] = \overline{\Psi},\nonumber
 \end{align}  
implying that $\beta_k^{\bDelta}\leq \frac{\Psi^*-\overline{\Psi}}{\lambda_k}$, thus $\lim_{\lambda_k\rightarrow \infty} \beta_k^{\bDelta} = 0$ as desired. 
\end{proof}

\section{Experimental Results on Puddle Environment: RN-CMDP}
%\subsection{Continuous Puddle Environment}
 Inspired by \cite{jain2021safe}, we test all the methods on the continuous puddle environment. The environment is shown in Figure \ref{fig:puddleperform}. It is a continuous two-dimensional state-space environment in [0, 1]. The agent starts at the bottom left corner of the map (0, 0) and the objective is to move to the goal at the upper right corner (1, 1). The agent can move in four directions and occasionally agent will execute a random action with probability $p=0.05$ instead of the one selected by the agent. When the agent is within 0.1 L1 distance from the goal state, the agent can be seen as reaching the goal and receive a reward of 100 while agent gets a time penalty as -0.1 at each time step. There is a square puddle region centering at (0.5, 0.5) with 0.4 height. In each time step, if agent is located in the puddle area, it gets a cost of 1. Due to the existence of noise, we cannot set the threshold $c_{max}$ too small as it would be hard for agent to reach the goal, so we set $c_{max}=8$, meaning agent could stay in puddle area for at most 8 time steps.

% \begin{figure}[htb]
% \centering
% \includegraphics[scale=0.4]{Continuepuddle.pdf}
% \caption{Puddle Environment}
% \label{fig:puddlecontinue}
% \end{figure}

 We show the results in Figure \ref{fig:puddleperform}. As can be seen from the figure, Safe SAC could outperform other methods in both reward and cost. Although Safe DQN can always satisfy the constraint, it always fail to reach the goal to get the maximum reward. For BVF, when the backward value function succeeds to estimate the cost, the reward starts to decrease and worse than Safe SAC.

\begin{figure*}[htbp]
    \centering
    \begin{minipage}{.3\columnwidth}
        \centering
        %\caption*{}
        \includegraphics[scale=0.3]{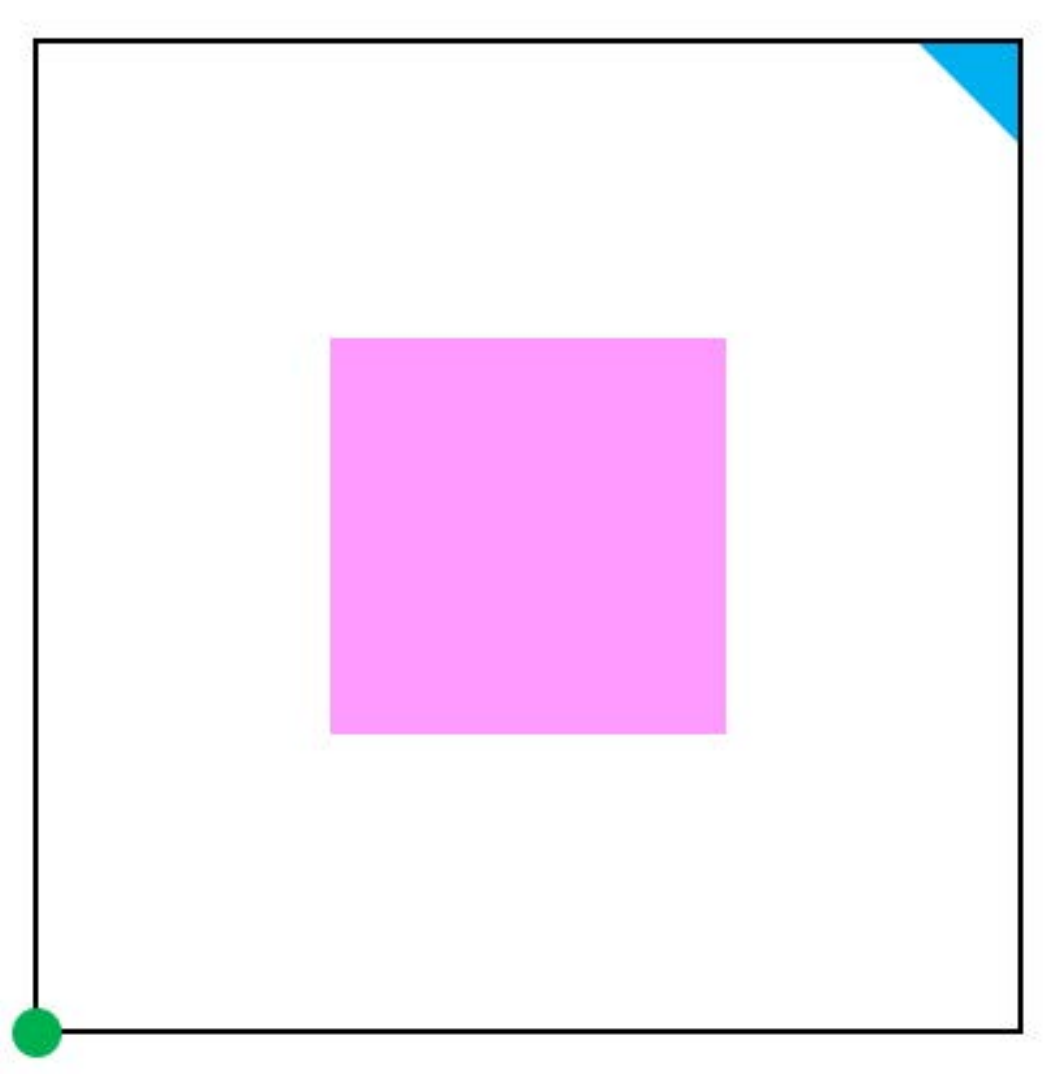}
        \label{puddleenv}
    \end{minipage}\hfill
    \begin{minipage}{.34\columnwidth}
        \centering
        %\caption*{}
        \includegraphics[scale=0.33]{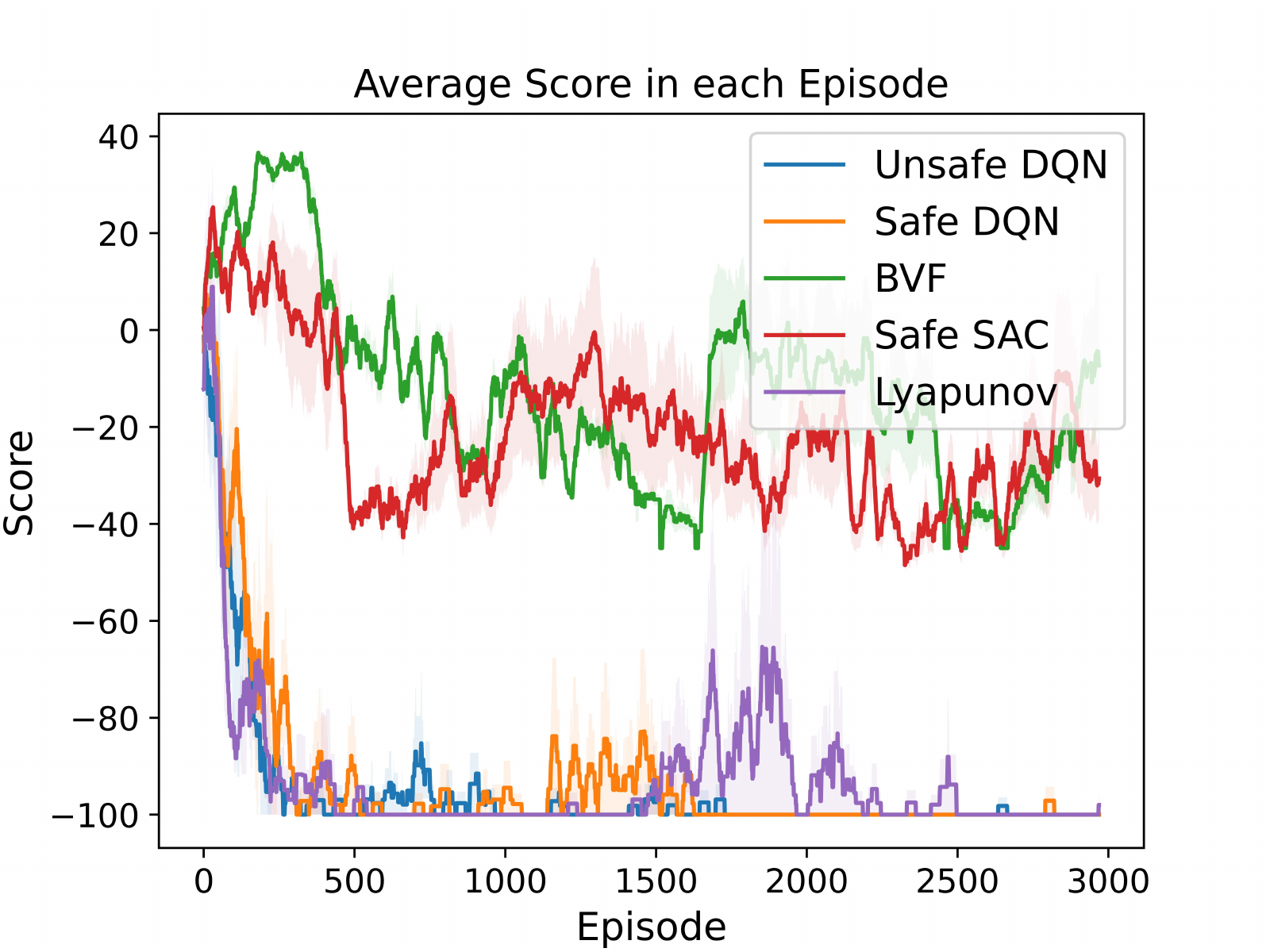}
        \label{puddlescore}
    \end{minipage}\hfill
    \begin{minipage}{.34\columnwidth}
        \centering
        %\caption*{}
        \includegraphics[scale=0.33]{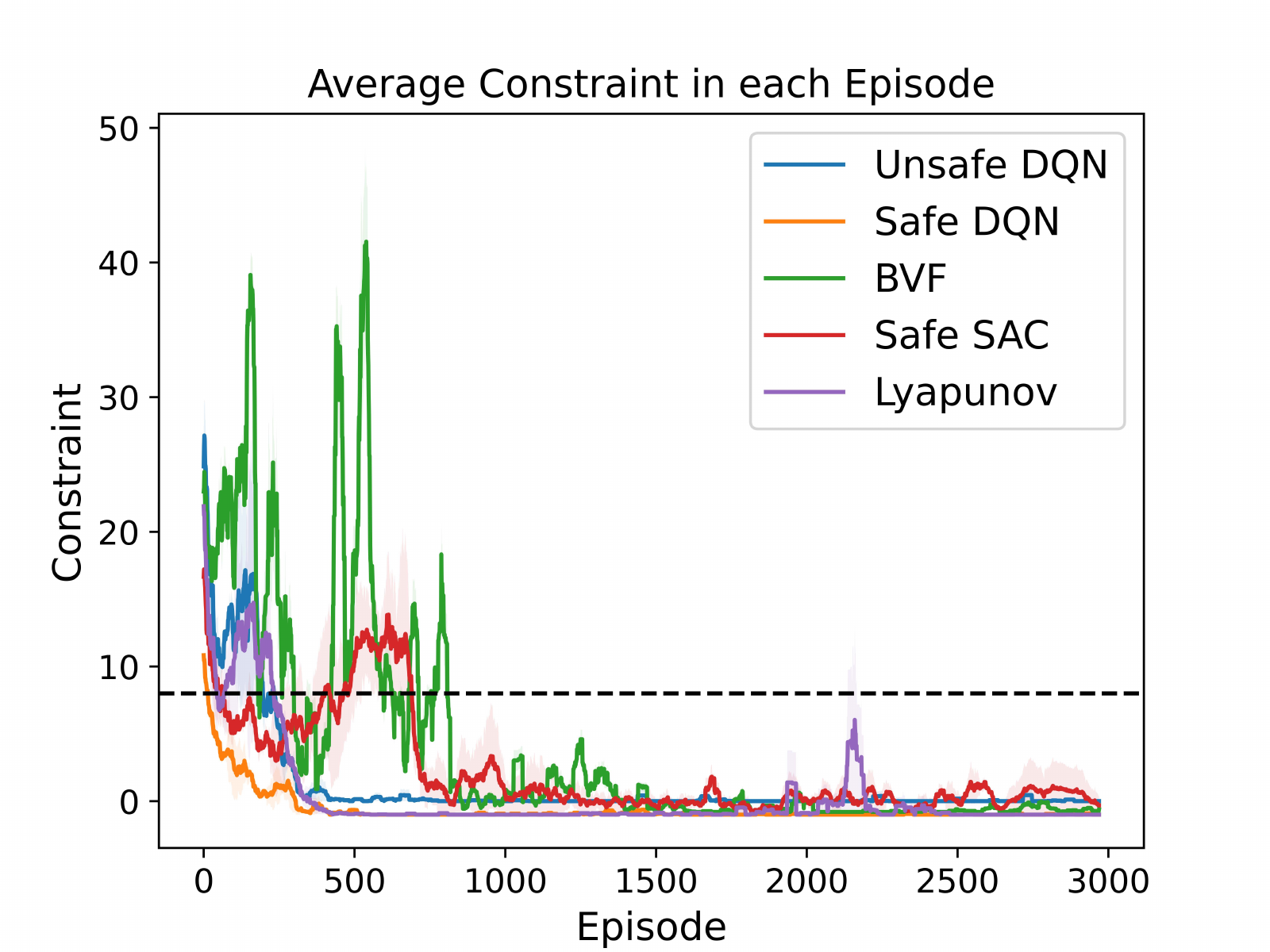}
        \label{puddlecons}
    \end{minipage}\hfill
    \caption{Puddle environment and reward, cost comparison of different approaches}
    \label{fig:puddleperform}
\end{figure*}

\section{Experimental Results on Highway Merge Environment: RN-CMDP}

\begin{figure*}[tb]
    \centering
    \begin{minipage}{0.3\linewidth}
        \centering
        %\caption*{}
        \includegraphics[scale=0.4]{Merge_Environment.pdf}
        \label{merge}
    \end{minipage}\hfill
    \begin{minipage}{0.34\linewidth}
        \centering
        %\caption*{}
        \includegraphics[scale=0.33]{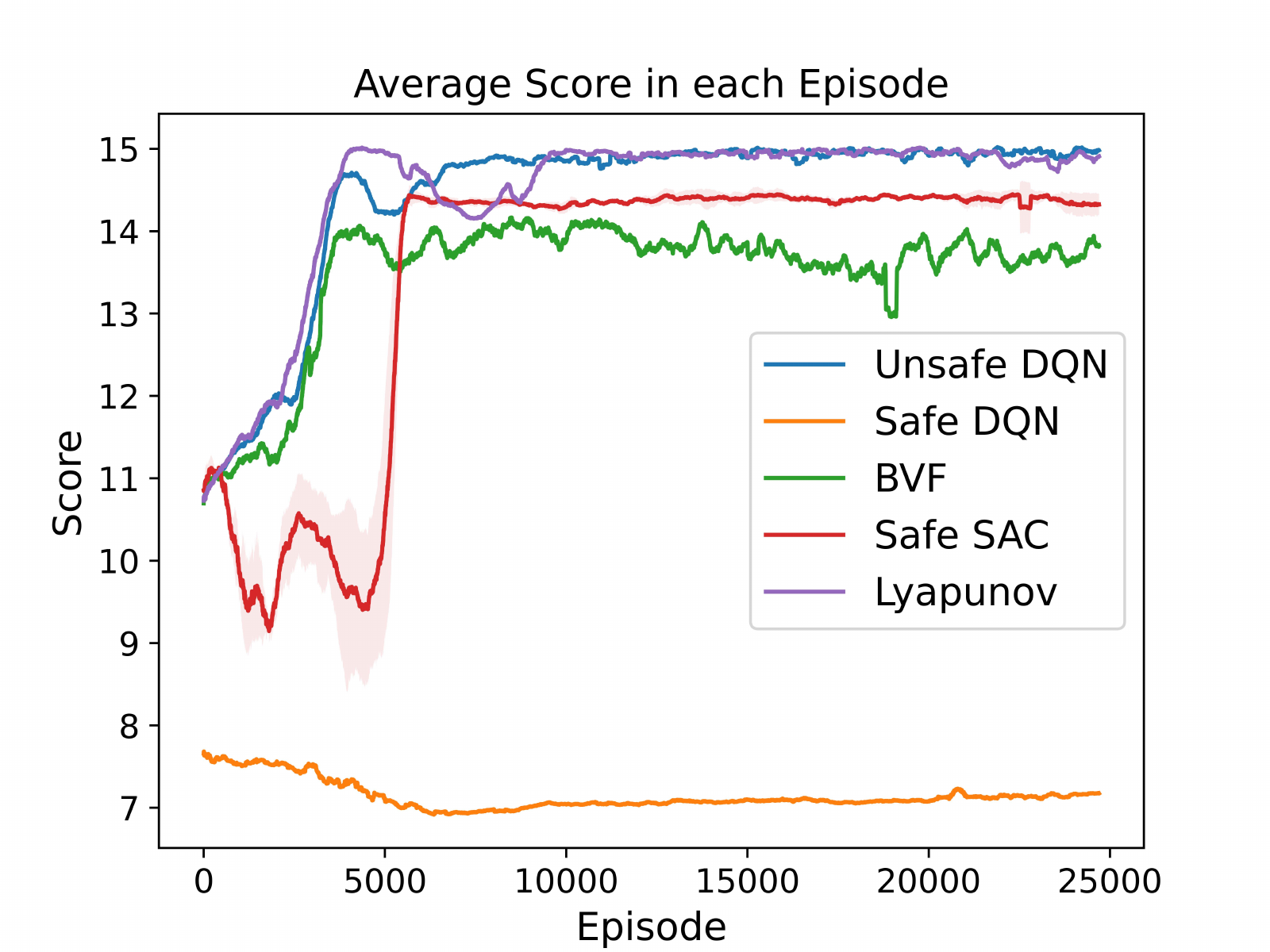}
        \label{mergescore}
    \end{minipage}\hfill
    \begin{minipage}{0.34\linewidth}
        \centering
        %\caption*{}
        \includegraphics[scale=0.33]{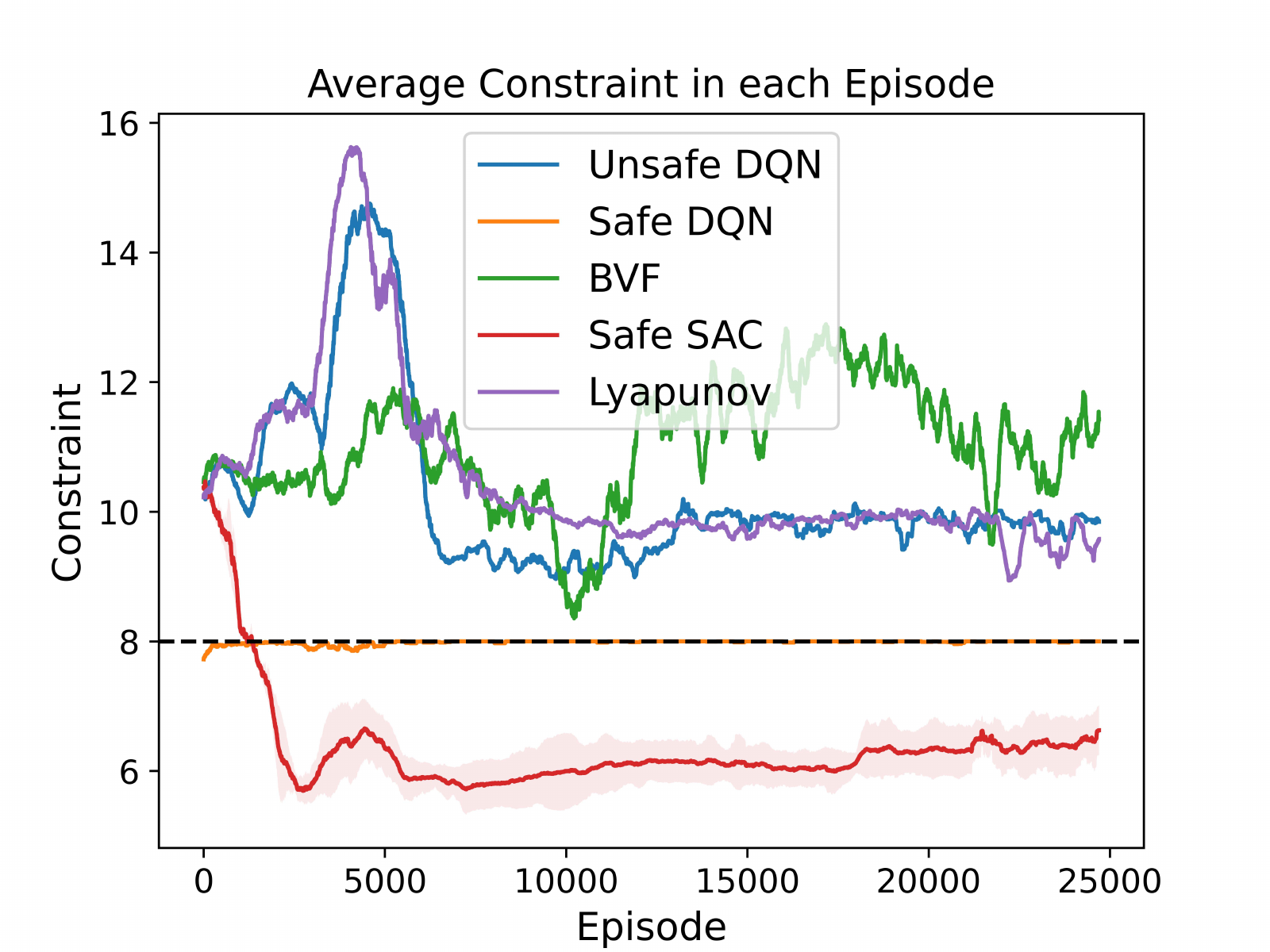}
        \label{mergecons}
    \end{minipage}\hfill
    \caption{Merge environment and reward, cost comparison of different approaches}
    \label{fig:mergeperform}
\end{figure*}

We also evaluate our safe methods on another highway environment - merge. The environment is shown in Figure \ref{fig:mergeperform} where agent needs to take actions to complete merging with other vehicles. The rewards are similar to those in highway environment. Figure \ref{fig:mergeperform} shows a comparison of our safe methods with other benchmarks. Although Safe DQN fails to complete the task in merge environment, Safe SAC still outperforms BVF and Unsafe DQN with better score and lower cost. The reason that Safe DQN fails is that the combinations of extended space is too large in merge environment for Safe DQN to figure it out. That is also why Safe DQN converges quite slowly in highway environment. As Safe DQN is unable to deal with large size of state space, Safe SAC outperforms Safe DQN in continuous environments.

\section{Hyperparameters}
In case of discrete environment - GridWorld, the size of state space is $8\times 8$ with 18 pits. In Highway environment (including merge), related parameters and their values are listed below. There is an additional reward in merge environment named $merging\_speed\_reward$ with value of -0.5. It penalties the agent if it drives with speed less than 30 while merging.
 \begin{itemize}
     \item $lanes\_count$: Number of lanes, setting as 4 in both environments.
     \item $vehicles\_count$: Number of vehicles on lanes, setting as 50 in both environments.
     \item $controlled\_vehicles$: Number of agents, setting as 1 in both environments.
     \item $duration$: Duration of the game, setting as 40 in both environments.
     \item $ego\_spacing$: The space of vehicles, setting as 2 in both environments.
     \item $vehicles\_density$: The density of vehicles on lanes, setting as 1 in both environments.
     \item $reward\_speed\_range$: The range where agent can receive $high\_speed\_reward$, setting as [20, 30] in both environments.
     \item $high\_speed\_reward$: Reward received when driving with speed in $reward\_speed\_range$, setting as 0.4 in highway while 0.2 in merge.
     \item $collision\_reward$: Reward received when colliding with a vehicle, setting as -1 in both environments.
     \item $right\_lane\_reward$: Reward received when driving on the right-most lane, setting as 0.1 in both environments.
     \item $lane\_change\_reward$: Reward received when taking lane change action, setting as 0 in highway while -0.05 in merge.
 \end{itemize}
 In all the methods, we use networks with a hidden layer size of 64,64,64 along with the ReLu activation and use Adam optimizer to optimize the networks. We test our methods on GridWorld, Highway, Safety Gym, Puddle, Highway for 25000, 25000, 1000, 3000, 25000 episodes respectively and update the network every 4 steps.

 \section{Ablation Analysis}
To investigate the effect of state augmentation and reward penalty, we conduct the ablation analysis in GridWorld and Highway environment. We compare the performance of Original SAC/ SAC with state augmentation (Augmented SAC)/ SAC with state augmentation and reward penalty (Safe SAC) and show the results of GridWorld in Figure \ref{fig:ablation1} and Highway in Figure \ref{fig:ablation2}. In Augmented SAC, agent chooses action using Equation \ref{equ:constraintaa}, if no action could satisfy the constraint, it chooses the action with minimum future cost. %We do not consider SAC with reward penalty for the reason: 1. SAC without state augmentation cannot store the local cost; 2. One state shares different constraints by choosing different trajectories, if we set constraint penalty for original state space, agent is unable to learn. 

In concrete environment (GridWorld), agent can reach good performance and satisfy the constraint only with state augmentation and reward penalty together. In continuous environment (Highway), Augmented SAC is able to learn the safe policy but suffers conservative in costs while dynamic reward penalty in Safe SAC manages to solve it.

\begin{equation}
\begin{aligned}
\label{equ:constraintaa}
    &\arg\max_a \min_{i=1,2} Q^i((s,c),a)-\alpha \log\pi(a|(s,c))\\
    s.t.&\max_{i=1,2} Q^i_D((s,c),a)+c-d((s,c))\leq c_{max},\forall (s,c)
\end{aligned}
\end{equation}

%\begin{figure}[htbp]
%    \centering
%    \begin{minipage}{0.5\linewidth}
%        \centering
%        %\caption*{}
%        \includegraphics[scale=0.45]{SAC_Comparison_Score.pdf}
%        \label{comparescore}
%    \end{minipage}\hfill
%    \begin{minipage}{0.5\linewidth}
%        \centering
%        %\caption*{}
%        \includegraphics[scale=0.45]{SAC_Comparison_Constraint.pdf}
%        \label{comparecons}
%    \end{minipage}\hfill
%    \caption{Ablation Analysis in GridWorld}
%    \label{fig:ablation1}
%\end{figure}

\begin{figure}[htbp]
    \centering
    \begin{minipage}{0.5\linewidth}
        \centering
        %\caption*{}
        \includegraphics[scale=0.45]{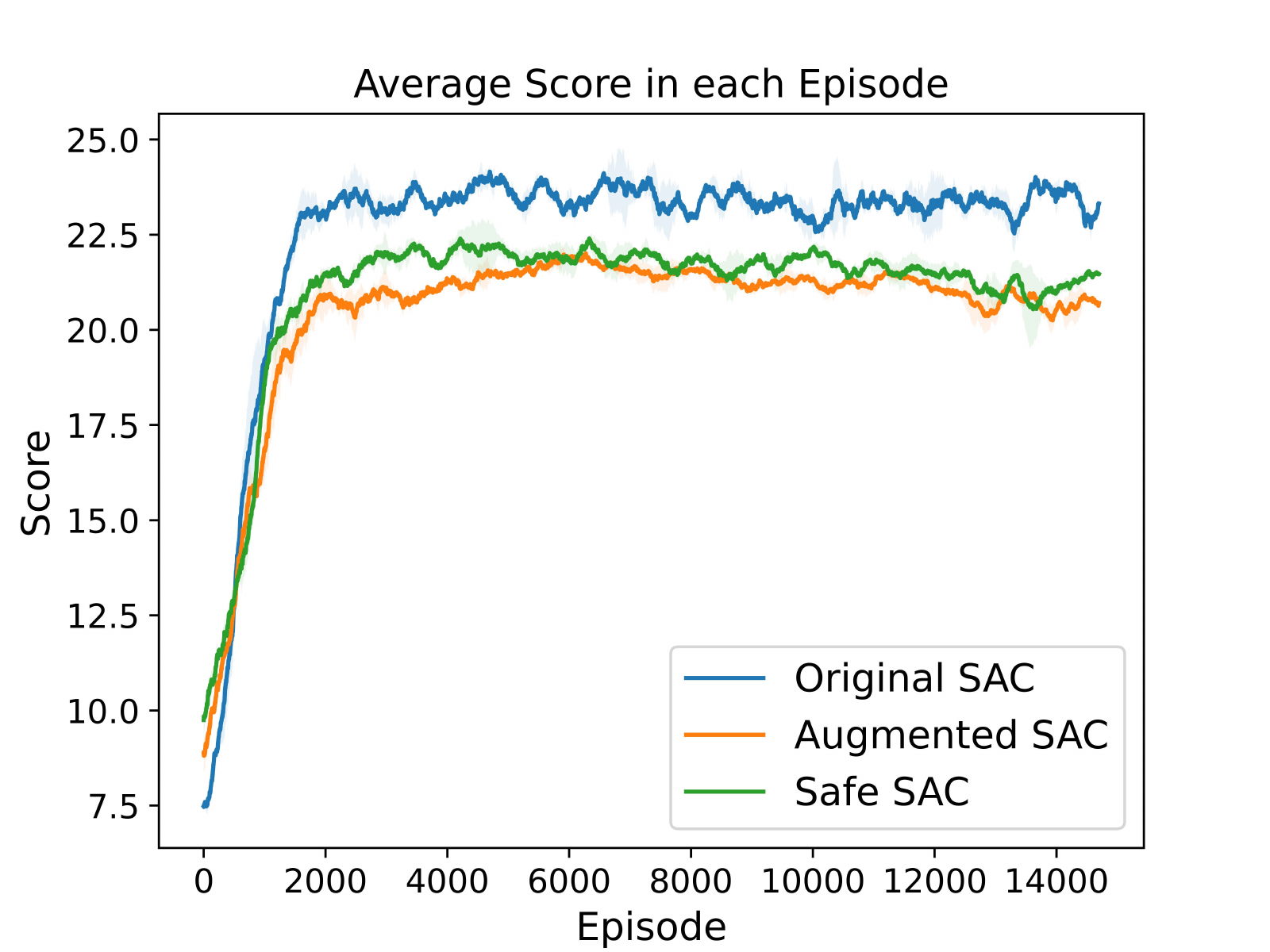}
        \label{comparescore1}
    \end{minipage}\hfill
    \begin{minipage}{0.5\linewidth}
        \centering
        %\caption*{}
        \includegraphics[scale=0.45]{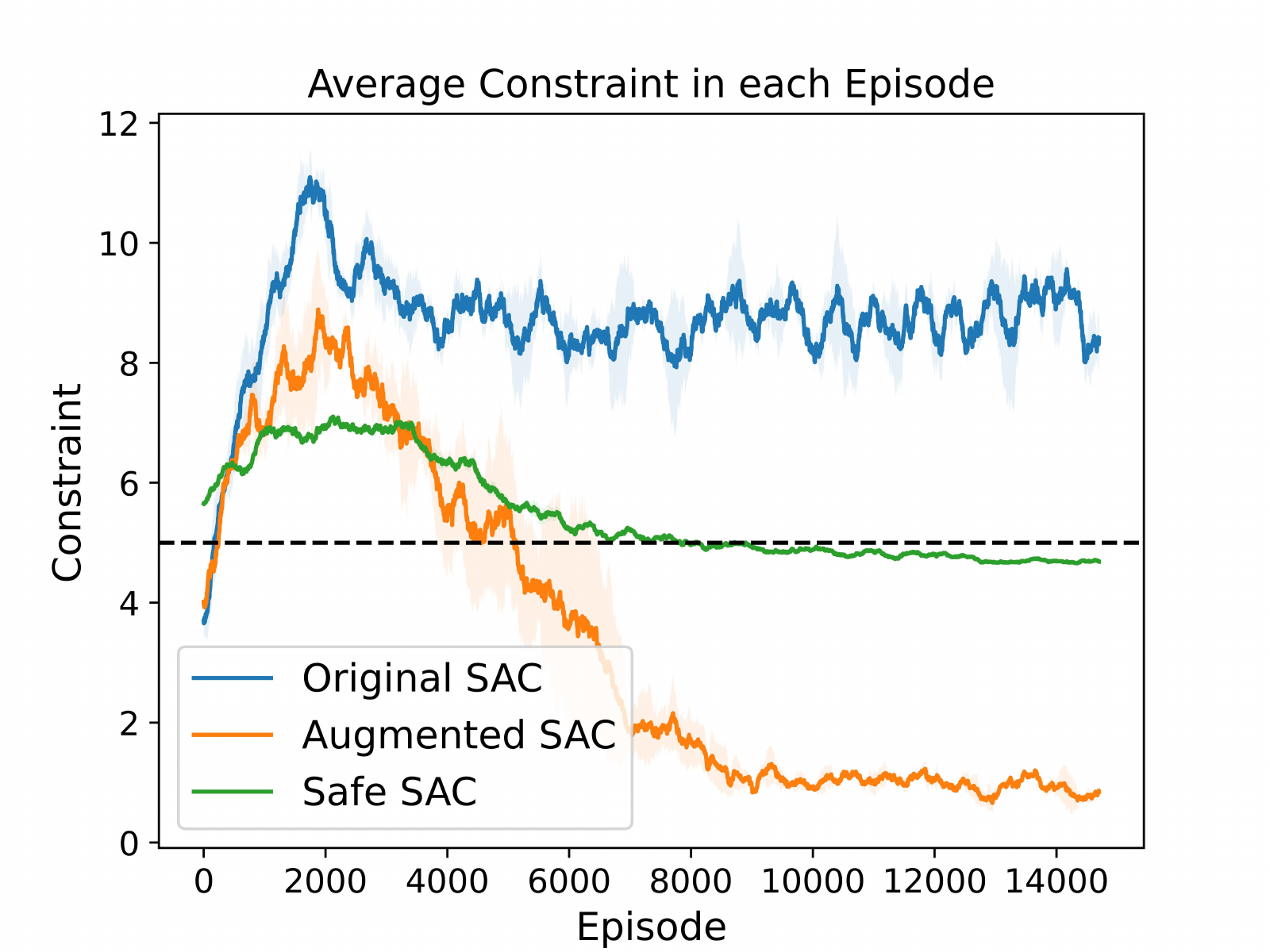}
        \label{comparecons1}
    \end{minipage}\hfill
    \caption{Ablation Analysis in Highway Environment}
    \label{fig:ablation2}
\end{figure}

\section{Impact of Value of Reward Penalty \texorpdfstring{$\lambda$}{}}
 To investigate the effect of different reward penalty values, we conduct experiments in GridWorld and Highway with Safe SAC by , with  $\lambda_1 = 1,\lambda_2=5\lambda_1,\lambda_3=10\lambda_1$, a small $\lambda_4 = 0.001$ and $\lambda_5=0$, we show the Highway results in Figure \ref{fig:constraintpenalty2}. In this experiment, we fix $\lambda$ so that it would be conservative in some cases. Here are some observations: 
% \squishlist
%\item As can be seen from results with $\lambda_1,\lambda_2,\lambda_3$, $\lambda_2$ in GridWorld and $\lambda_1$ in Highway could be the best value to achieve good performance, it means that only with an appropriate $\lambda$ value, agent can receive high reward and satisfy the constraint with fast convergence speed. That is consistent with our theory and is why we set the threshold $\lambda$ while dynamically change $\lambda$ in the safe algorithms.
%\item The experiment results are consistent of the theory: Cases with a small $\lambda$ are close to those without $\lambda$. Cases with large $\lambda$ are close to \eqref{equ:rmdp}, which does not allow the violation of constraint.
%\squishend
 We can see that $\lambda_1$ in Highway could be good choices to achieve good performance, implying that only with an appropriate $\lambda$ value, the agent can receive high rewards while satisfies the constraint with fast convergence speed. That is consistent with our theory and is why we set a threshold for $\lambda$ and dynamically change it during the safe algorithms.
Moreover, cases with a small $\lambda$ are close to those without $\lambda$. Cases with large $\lambda$ are close to \eqref{equ:rmdp}, which does not allow constraint violation.

% \begin{figure}[htbp]
%    \centering
%    \begin{minipage}{0.5\linewidth}
%        \centering
%        %\caption*{}
%        \includegraphics[scale=0.45]{SAC_Comparison_Score1.pdf}
%        \label{consscore}
%    \end{minipage}\hfill
%    \begin{minipage}{0.5\linewidth}
%        \centering
%        %\caption*{}
%        \includegraphics[scale=0.45]{SAC_Comparison_Constraint1.pdf}
%        \label{conscons}
%    \end{minipage}\hfill
%    \caption{Experiment in GridWorld with Different Reward Penalties}
%    \label{fig:constraintpenalty1}
%\end{figure}

 \begin{figure}[htbp]
    \centering
    \begin{minipage}{0.5\linewidth}
        \centering
        %\caption*{}
        \includegraphics[scale=0.45]{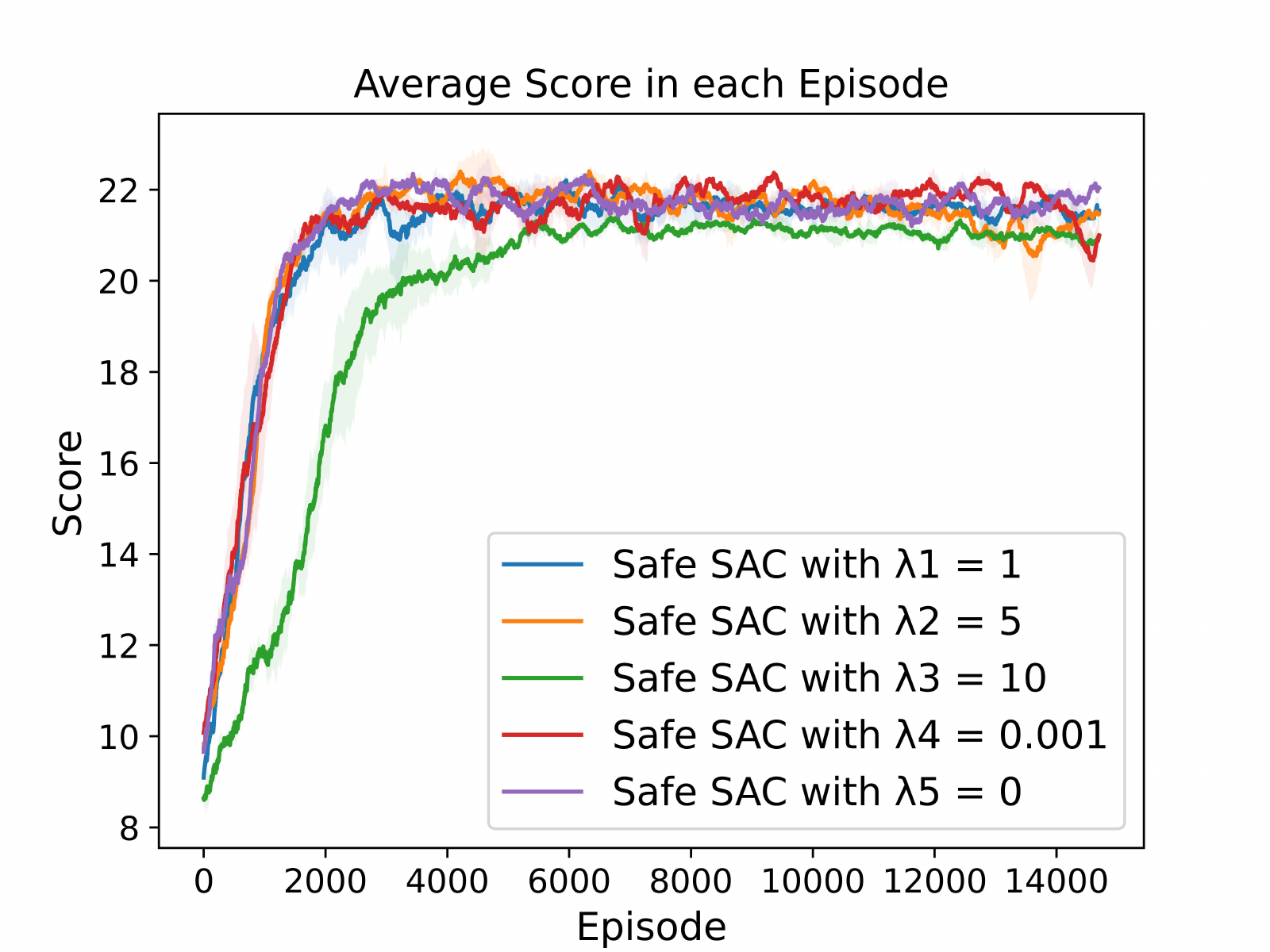}
        \label{consscore2}
    \end{minipage}\hfill
    \begin{minipage}{0.5\linewidth}
        \centering
        %\caption*{}
        \includegraphics[scale=0.45]{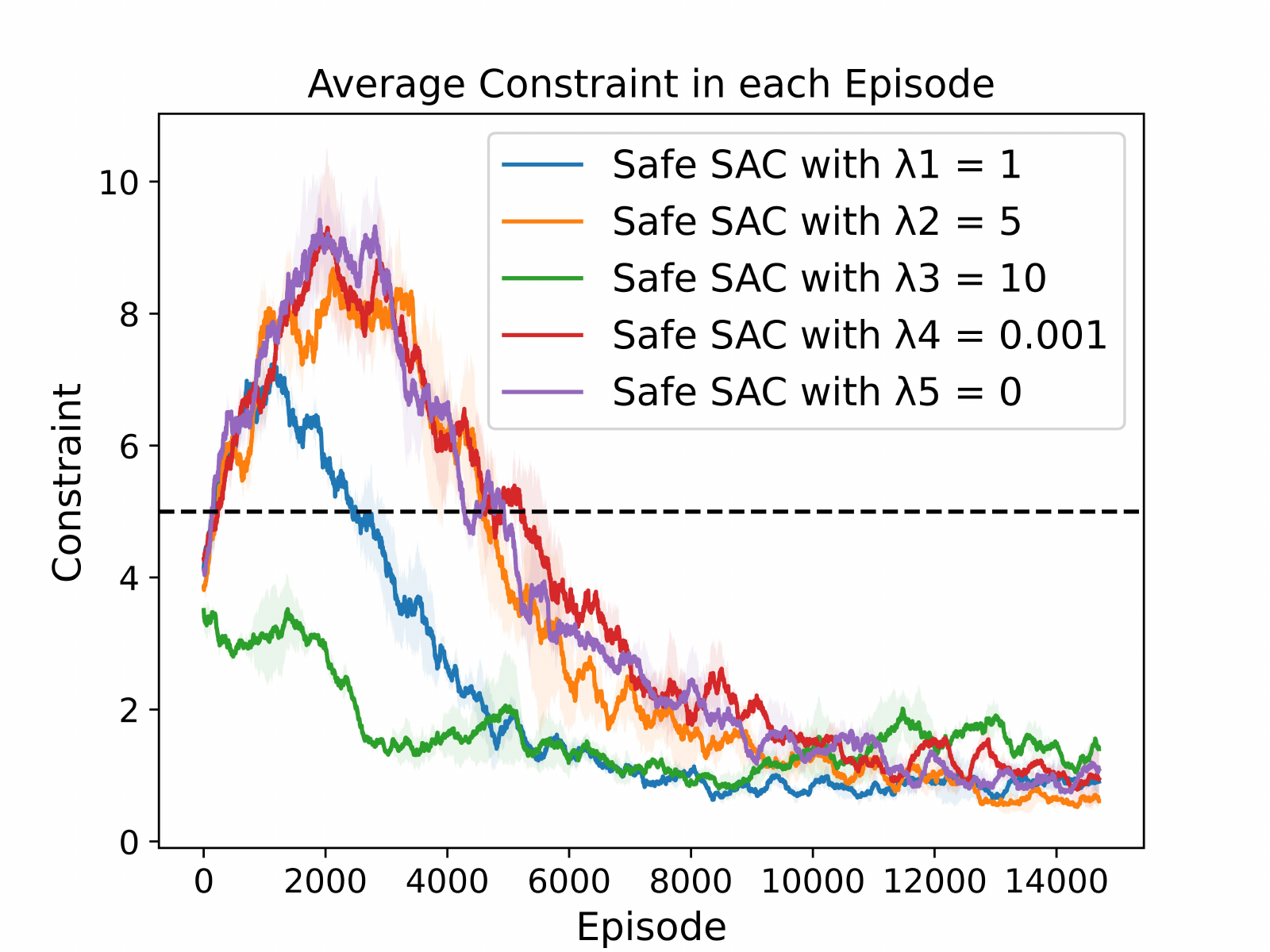}
        \label{conscons2}
    \end{minipage}\hfill
    \caption{Experiment in Highway with Different Reward Penalties}
    \label{fig:constraintpenalty2}
\end{figure}

We also show how $\lambda$ changes in Safe SAC. Figure \ref{fig:lambda} shows the average value of $\lambda$ in recent 300 episodes in GridWorld and Highway environment. We set the initial value of $\lambda$ as 2 and the lower bound for constraint penalty $\lambda$ as 0.1 in both environments.

 \begin{figure}[htbp]
    \centering
    \begin{minipage}{0.5\linewidth}
        \centering
        %\caption*{}
        \includegraphics[scale=0.45]{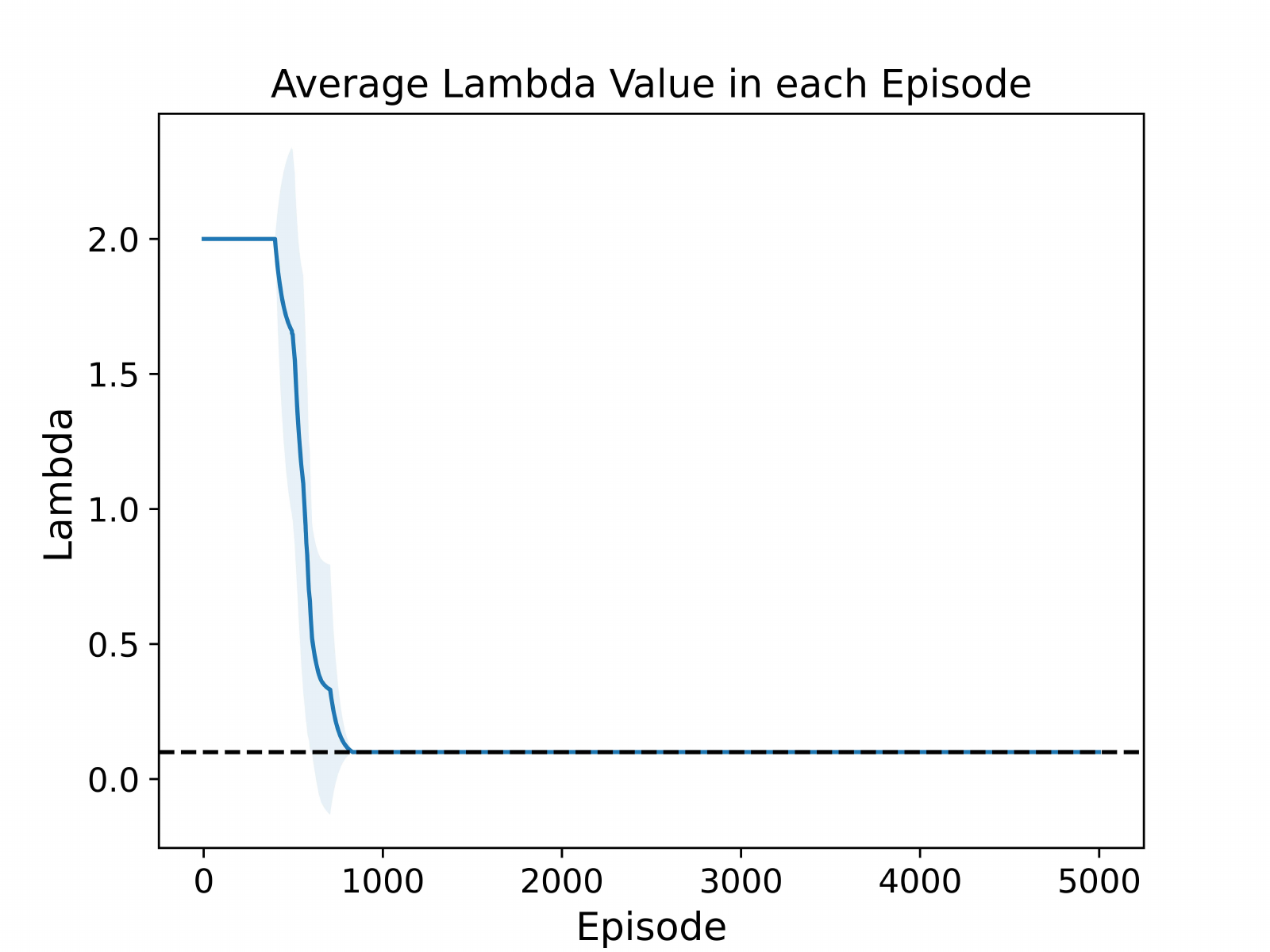}
        \label{gridlambda}
    \end{minipage}\hfill
    \begin{minipage}{0.5\linewidth}
        \centering
        %\caption*{}
        \includegraphics[scale=0.45]{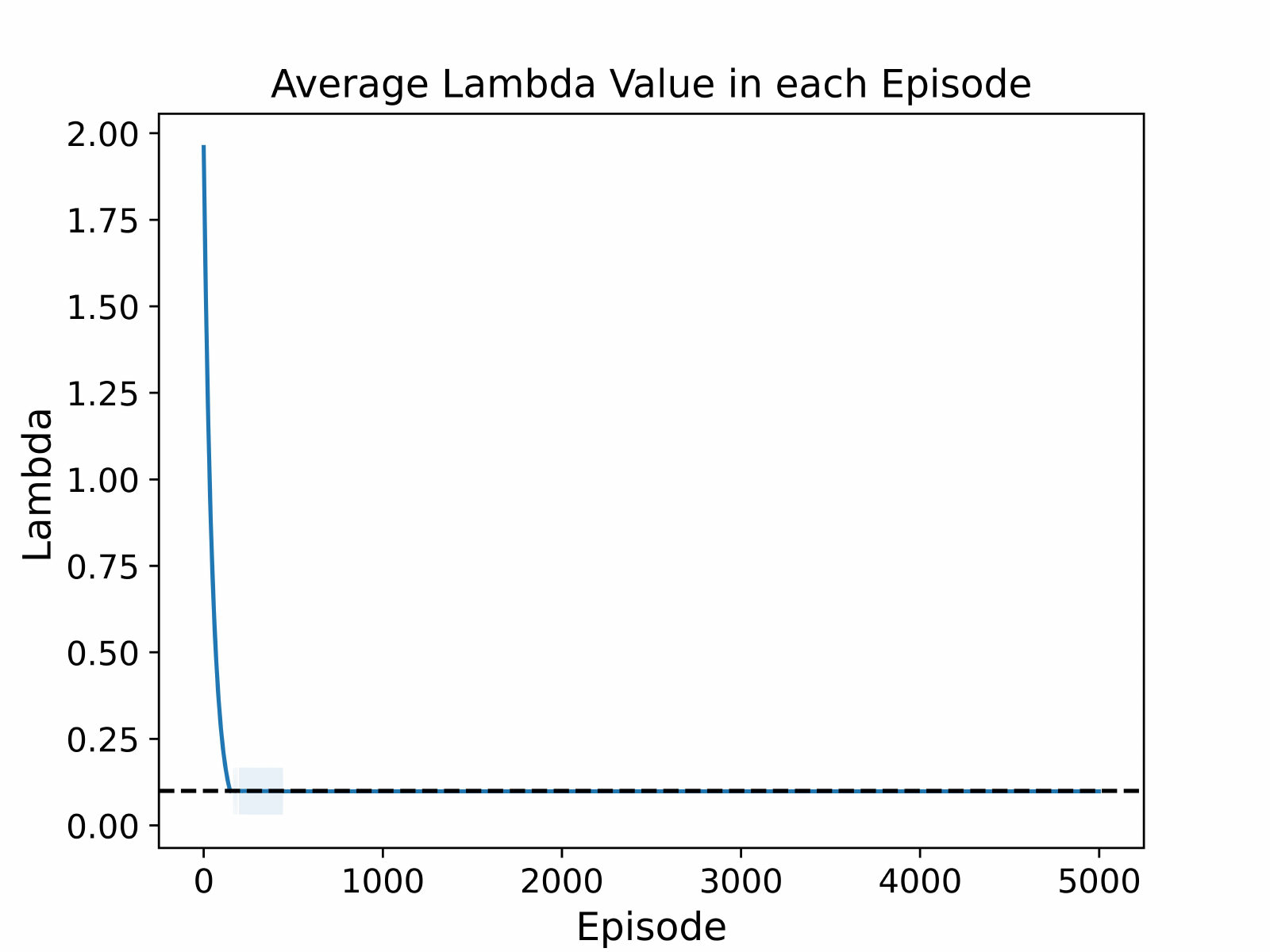}
        \label{highwaylambda}
    \end{minipage}\hfill
    \caption{$\lambda$ Value in GridWorld and Highway Environment}
    \label{fig:lambda}
\end{figure}

\section{Experimental Results with CVaR constraint: CVaR-CMDP}
In the previous experiments, we think a policy is safe when its expected constraint does not exceed $c_{max}$. However, using expected constraint does not consider some potential risk brought by randomness in the system, so in this part we consider the worst case with CVaR constraint. 

We introduce CVaR constraint to our Safe SAC method and compare it with WCSAC \cite{yang2021wcsac}, which is a leading method using CVaR constraint. As WCSAC is limited to continuous space, we only do the comparison in the continuous environment. We firstly compare the performance in Safety Gym environment, which is used in WCSAC research, additionally we provide the results in other continuous environments.

\subsection{Safety Gym Environment}
For safety gym environment, we use the same environment from \cite{yang2021wcsac} - StaticEnv in Safety Gym \cite{ray2019benchmarking}. The environment is shown in Figure \ref{fig:safetygymcvar}. The point agent has two types of actions: one is for turning and another is for moving forward/backward. The objective is to reach the goal position while trying to avoid hazardous areas. The agent gets a reward of $r-0.2$ in each time step, where $r$ is an original reward signal of Safety Gym (distance towards goal plus a constant for being within range of goal) while -0.2 functions as a time penalty. In each step, if the agent is located in the hazardous area, it gets a cost of 1. We set $c_{max}=8$, meaning agent could stay in hazardous area for at most 8 time steps. For risk level $\alpha$ in WCSAC, we set $\alpha=0.9$ and use the almost risk-neutral WCSAC, which is proven to reach the best performance in both reward and cost in experiment.

 We show the results in Figure \ref{fig:safetygymcvar}. As can be seen from the figure, Safe SAC is able to achieve similar performance to that of WCSAC.
 
\begin{figure*}[htbp]
    \centering
    \begin{minipage}{.3\columnwidth}
        \centering
        %\caption*{}
        \includegraphics[scale=0.24]{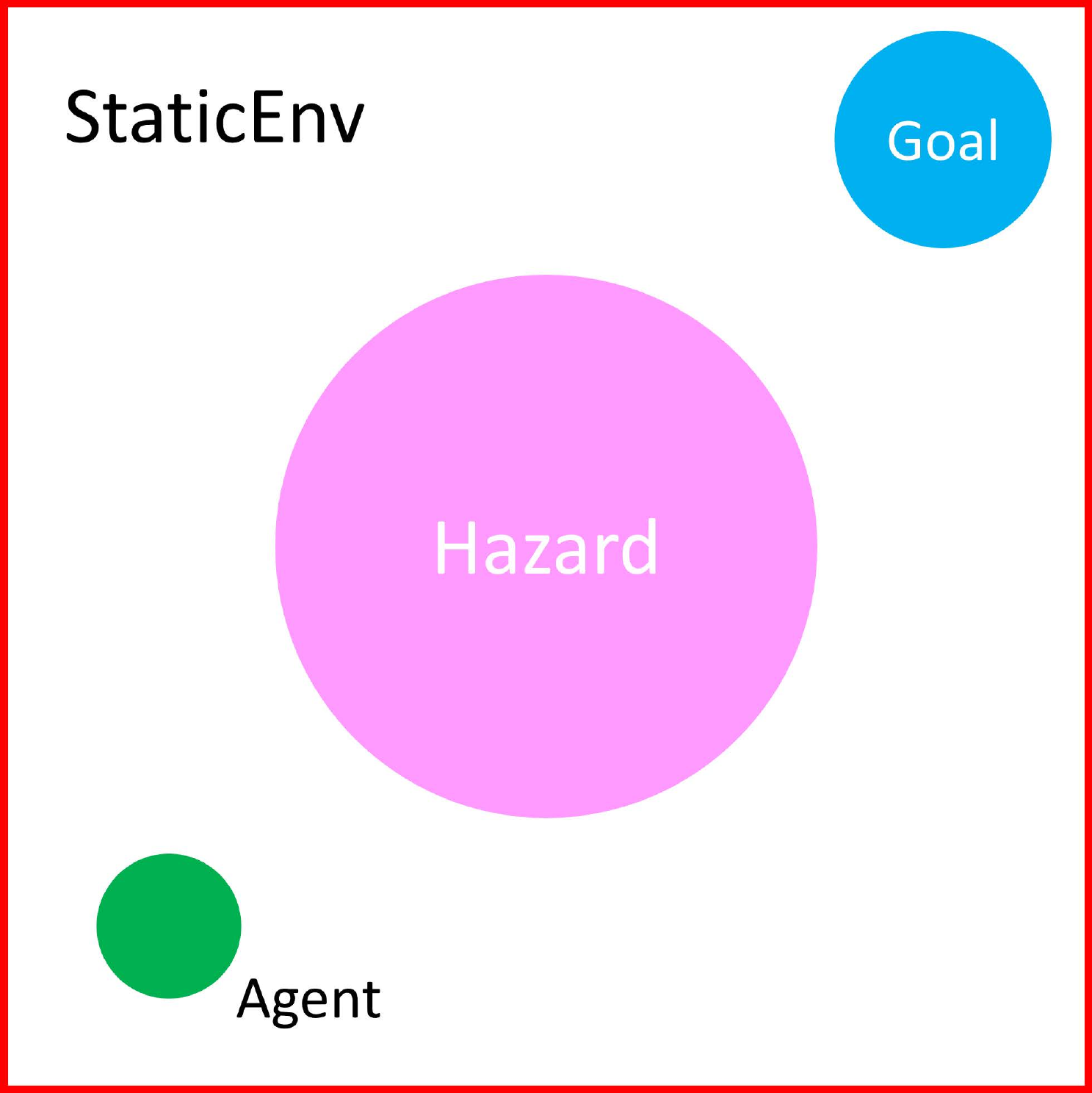}
        \label{safetygymenv}
    \end{minipage}\hfill
    \begin{minipage}{.34\columnwidth}
        \centering
        %\caption*{}
        \includegraphics[scale=0.33]{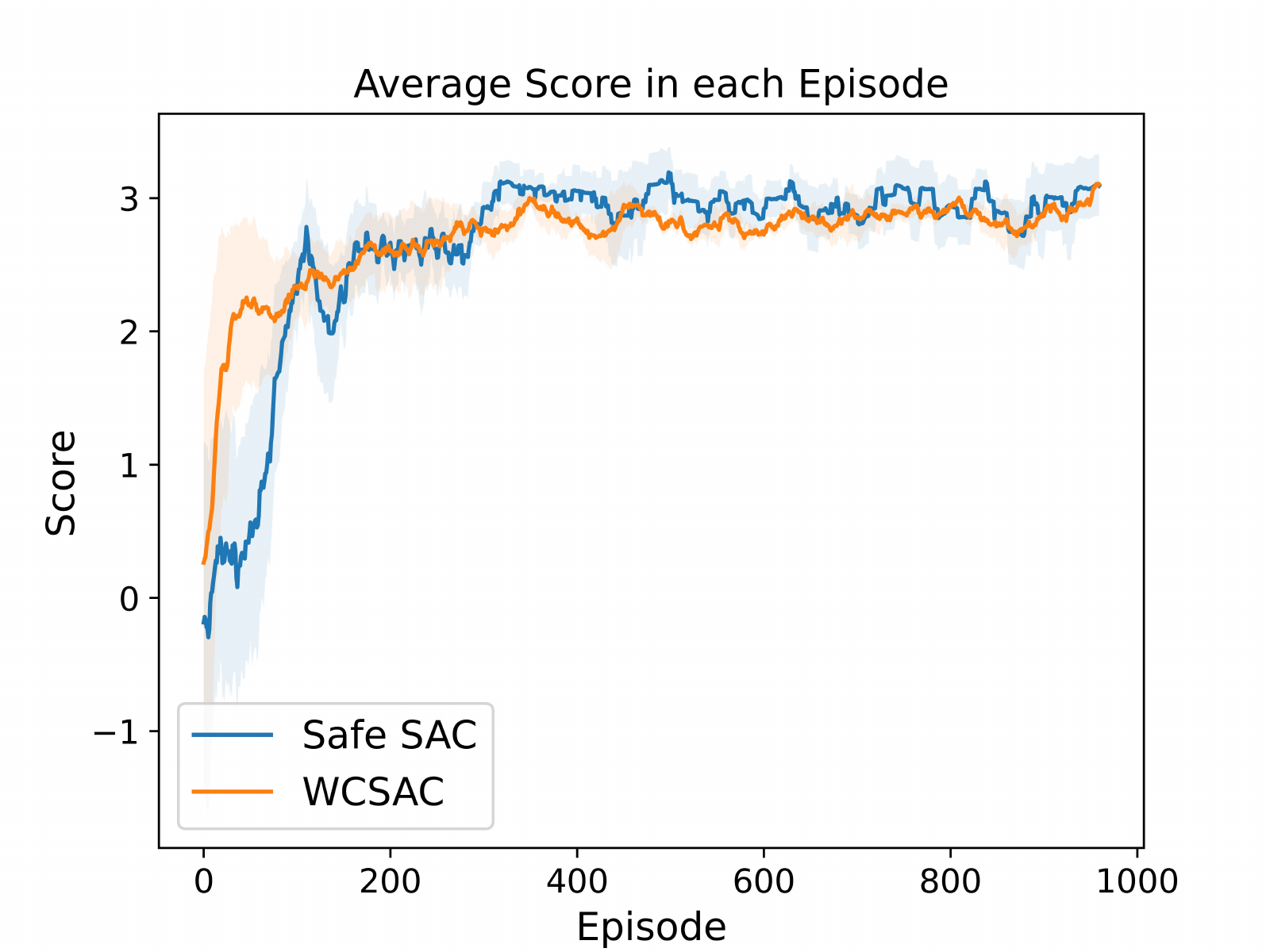}
        \label{safetygymscore}
    \end{minipage}\hfill
    \begin{minipage}{.34\columnwidth}
        \centering
        %\caption*{}
        \includegraphics[scale=0.33]{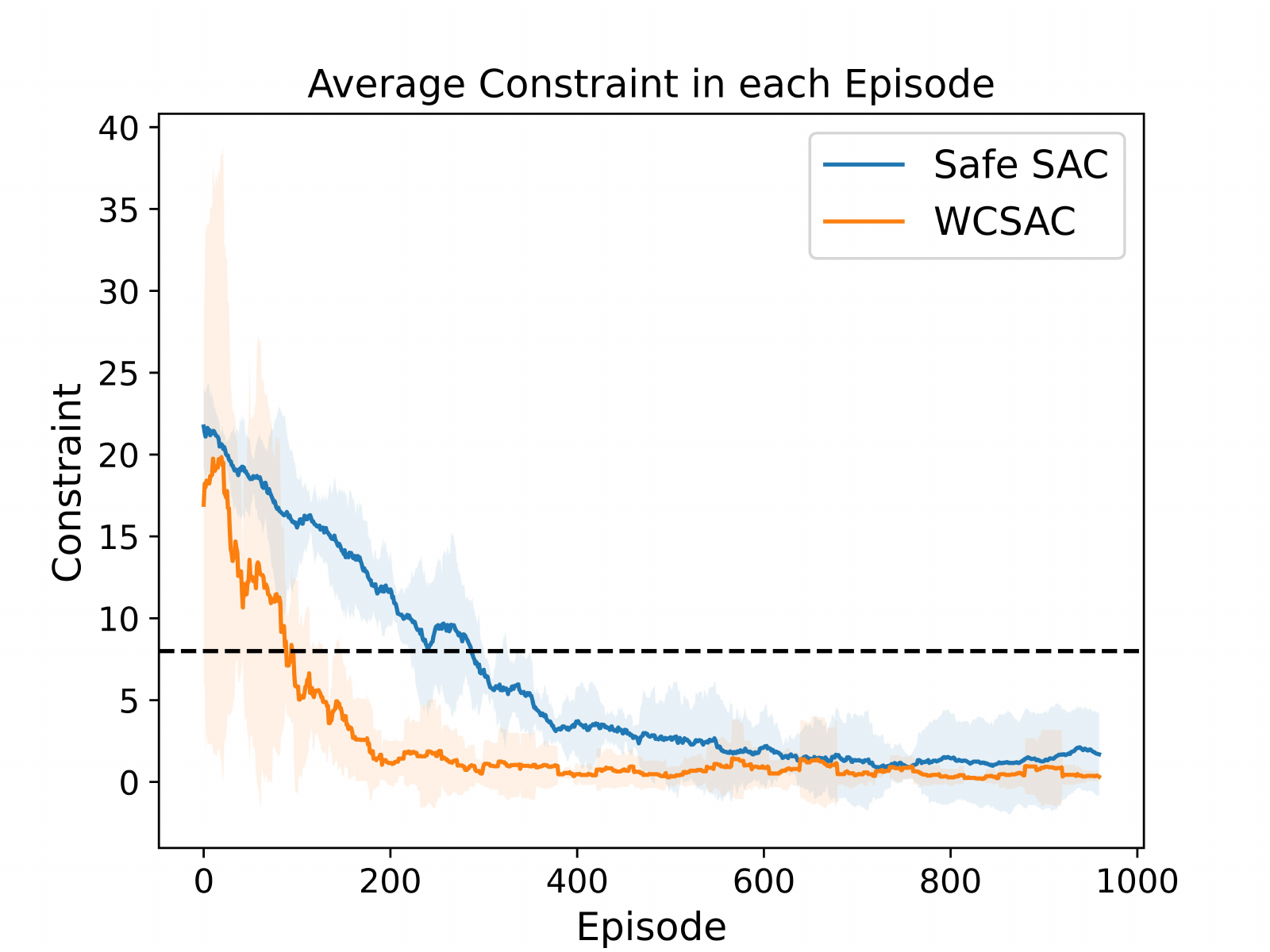}
        \label{safetygymcons}
    \end{minipage}\hfill
    \caption{Experiment with CVaR Constraint in Safety Gym Environment}
    \label{fig:safetygymcvar}
\end{figure*}

\subsection{Experiment Results on Other Environments}
Figure \ref{fig:highwaycvar}, \ref{fig:puddlecvar} show the comparisons in Highway, Puddle with risk level $\alpha = 0.9$. In most environments, Safe SAC shares similar costs with WCSAC while we outperforms WCSAC in scores as we introduce some methods in reward improvement (double Q trick and entropy).

 \begin{figure}[htbp]
    \centering
    \begin{minipage}{0.5\linewidth}
        \centering
        %\caption*{}
        \includegraphics[scale=0.45]{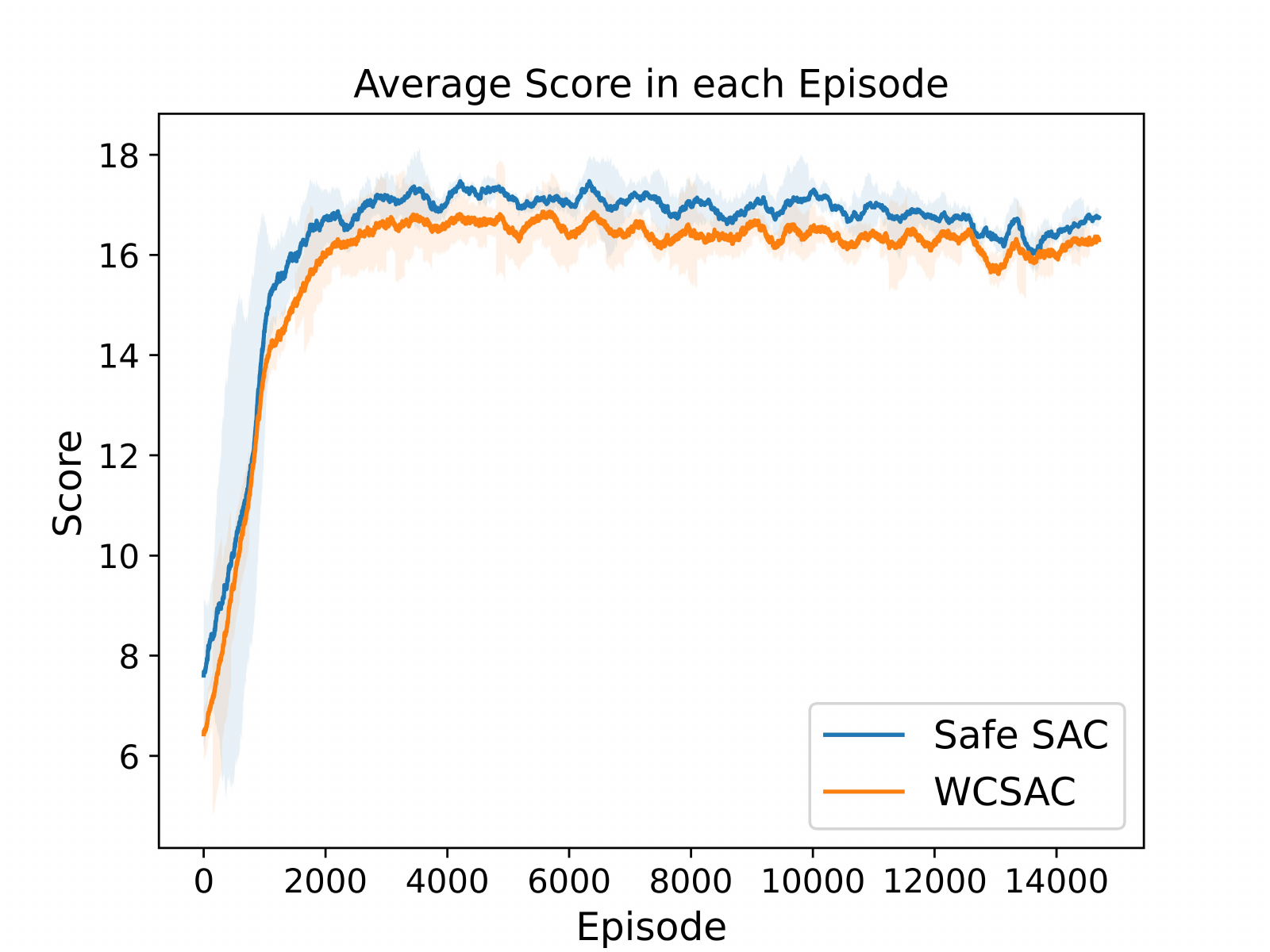}
        \label{highwaycvarscore}
    \end{minipage}\hfill
    \begin{minipage}{0.5\linewidth}
        \centering
        %\caption*{}
        \includegraphics[scale=0.45]{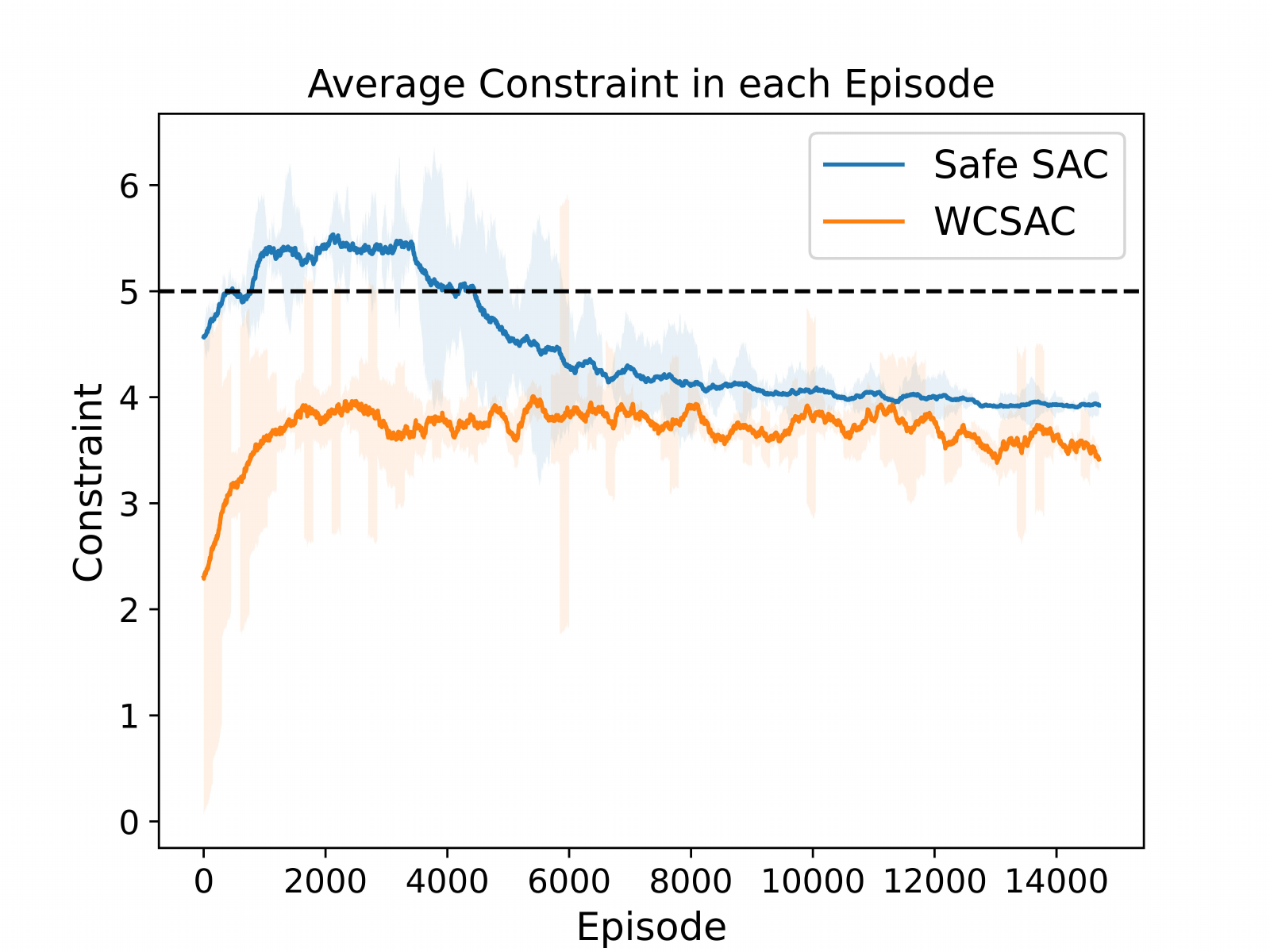}
        \label{highwaycvarcons}
    \end{minipage}\hfill
    \caption{Experiment with CVaR Constraint in Highway Environment}
    \label{fig:highwaycvar}
\end{figure}

% \begin{figure}[htbp]
%    \centering
%    \begin{minipage}{0.5\linewidth}
%        \centering
%        %\caption*{}
%        \includegraphics[scale=0.45]{Merge_Scores_2.pdf}
%        \label{mergecvarscore}
%    \end{minipage}\hfill
%    \begin{minipage}{0.5\linewidth}
%        \centering
%        %\caption*{}
%        \includegraphics[scale=0.45]{Merge_Constraint_2.pdf}
%        \label{mergecvarcons}
%    \end{minipage}\hfill
%    \caption{Experiment with CVaR Constraint in Merge Environment}
%    \label{fig:mergecvar}
%\end{figure}

 \begin{figure}[htbp]
    \centering
    \begin{minipage}{0.5\linewidth}
        \centering
        %\caption*{}
        \includegraphics[scale=0.45]{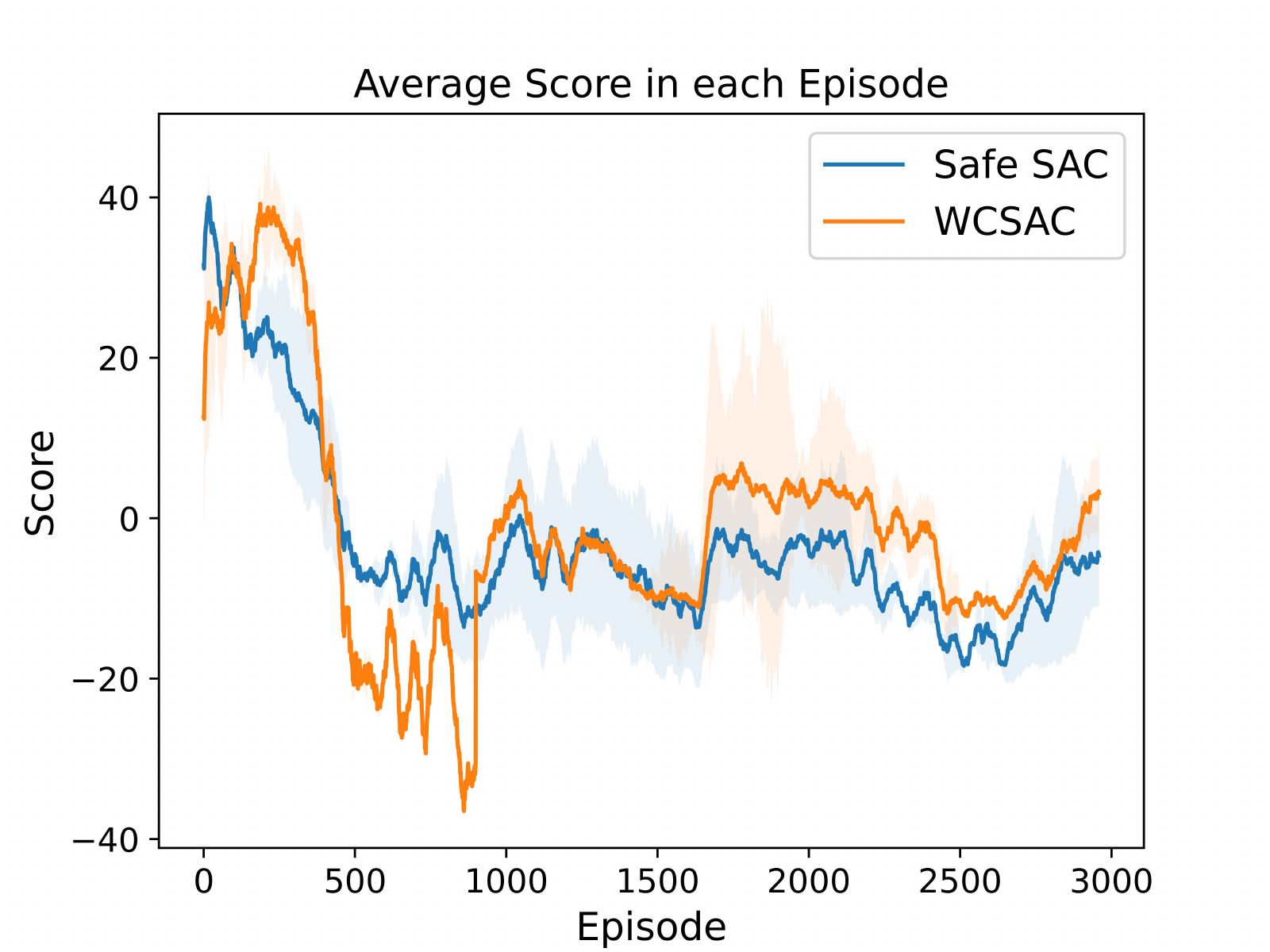}
        \label{puddlecvarscore}
    \end{minipage}\hfill
    \begin{minipage}{0.5\linewidth}
        \centering
        %\caption*{}
        \includegraphics[scale=0.45]{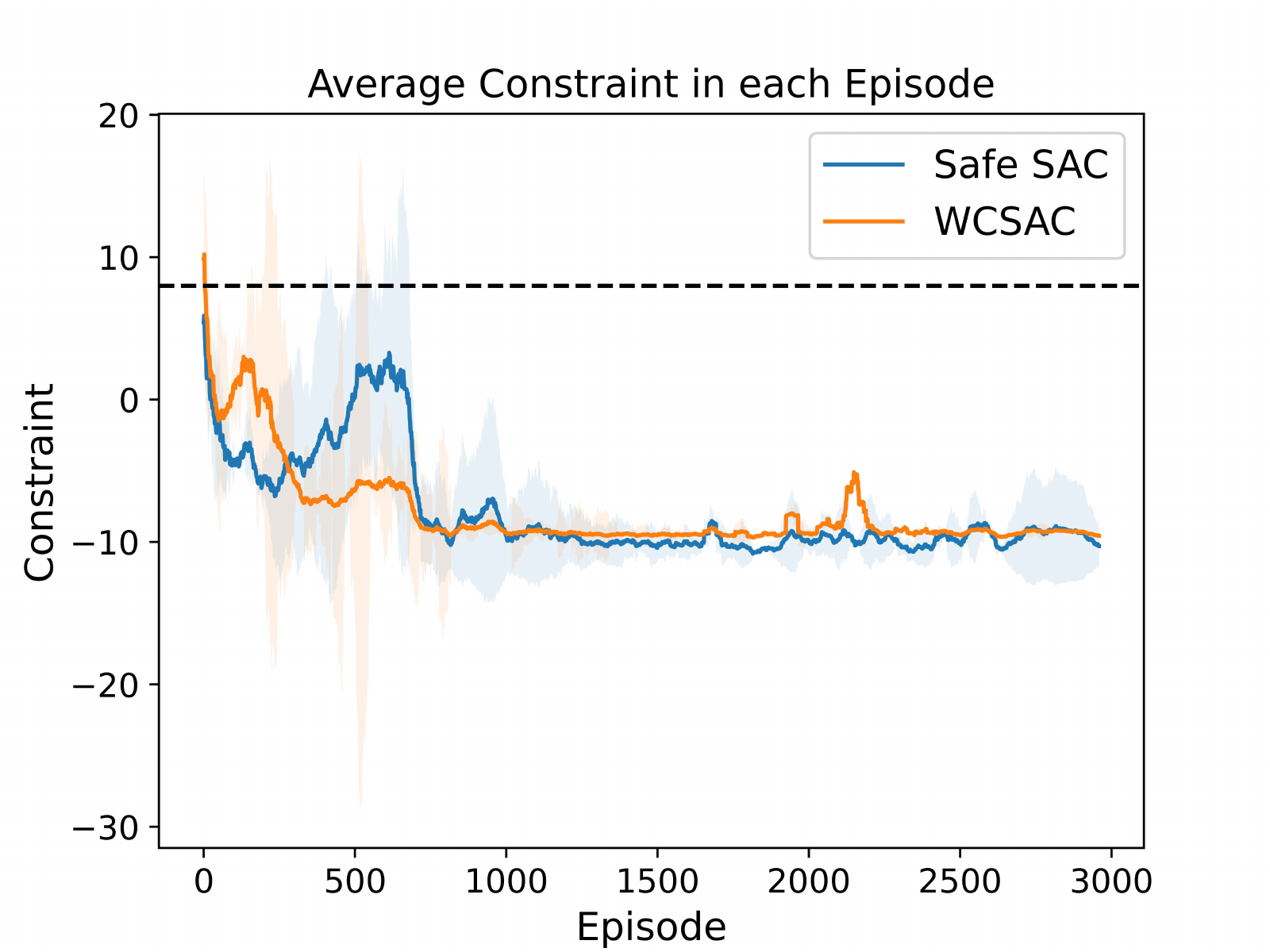}
        \label{puddlecvarcons}
    \end{minipage}\hfill
    \caption{Experiment with CVaR Constraint in Puddle Environment}
    \label{fig:puddlecvar}
\end{figure}

\end{document}